\documentclass{article}
\usepackage{graphicx} 
\usepackage{array} 
\usepackage{amsmath}
\usepackage{custom_tex}

\title{\Large Conformal Inference under High-Dimensional Covariate Shifts via Likelihood-Ratio Regularization}
\author{
  Sunay Joshi\thanks{Equal Contribution. Correspondence to: \texttt{sunayj@sas.upenn.edu}, 
  \texttt{shayank@seas.upenn.edu}.} \quad
  Shayan Kiyani\footnotemark[1] \\
  George Pappas \quad
  Edgar Dobriban  \quad
  Hamed Hassani \\
  University of Pennsylvania
}

\begin{document}
\date{}
\maketitle

\begin{abstract}
We consider the problem of conformal prediction under covariate shift. Given labeled data from a source domain and unlabeled data from a covariate shifted target domain, we seek to construct prediction sets with valid marginal coverage in the target domain. Most existing methods require estimating the unknown likelihood ratio function, which can be prohibitive for high-dimensional data such as images. To address this challenge, we introduce the likelihood ratio regularized quantile regression (LR-QR) algorithm, which combines the pinball loss with a novel choice of regularization in order to construct a threshold function without directly estimating the unknown likelihood ratio. We show that the LR-QR method has coverage at the desired level in the target domain, up to a small error term that we can control. Our proofs draw on a novel analysis of coverage via stability bounds from learning theory. Our experiments demonstrate that the LR-QR algorithm outperforms existing methods on high-dimensional prediction tasks, including a regression task for the Communities and Crime dataset, an image classification task from the WILDS repository, and an LLM question-answering task on the MMLU benchmark.
\end{abstract}

\section{Introduction}

Conformal prediction is a framework 
to construct distribution-free prediction sets 
for black-box predictive models \citep[e.g.,][etc]{saunders1999transduction,vovk1999machine,vovk2022algorithmic}.
Formally,
given a pretrained prediction model $f : \xx \to \yy$ that maps features $x\in \xx$ to labels $y\in \yy$,
as well as $n_1$ calibration datapoints $(X_i, Y_i) : i \in [n_1]$ sampled i.i.d.~from a calibration distribution $\mathbb{P}_1$,
we seek to construct a prediction set $C(X_{\tn{test}}) \subseteq \yy$
for test features $X_{\tn{test}}$ sampled from a test distribution $\mathbb{P}_2$.
We aim to cover the true label $Y_{\tn{test}}$ with probability at least $1-\alpha$
for some $\alpha \in (0,1)$:
that is, $\mathbb{P}( Y_{\tn{test}} \in C(X_{\tn{test}})  ) \ge 1-\alpha$.
The left-hand side of this inequality is the marginal coverage of the prediction set $C$,
averaged over the randomness of both 
the calibration datapoints and the test datapoint.
In the case that the calibration and test distributions coincide ($\mathbb{P}_1 = \mathbb{P}_2$),
there are numerous conformal prediction algorithms that construct
distribution-free prediction sets with valid marginal coverage;
for instance, split and full conformal prediction 
\citep[e.g.,][]{papadopoulos2002inductive,lei2013distribution}.

However, in practice, it is often the case
that test data is sampled from a different distribution than calibration data.
This general phenomenon is known as distribution shift \citep[e.g.,][]{quinonero2009dataset,Sugiyama2012}.
One particularly common type of distribution shift is
covariate shift \citep{shimodaira2000improving}, where 
the conditional distribution of $Y|X$ stays fixed,
but the marginal distribution of features changes from calibration to test time.
For instance, in the setting of image classification for autonomous vehicles,
the calibration and test data might have been collected under different weather conditions \citep{yu2020bdd100k, koh2021wilds}.
Under covariate shift, ordinary conformal prediction algorithms
may lose coverage.

Recently, a number of methods have been proposed to adapt conformal prediction to covariate shift, e.g., in  \cite{tibshirani2019conformal,park2021pac,park2022pac,gibbs2025conformal,qiu2023prediction,yang2024doubly,gui2024distributionally}.
Most existing approaches attempt to estimate the likelihood ratio function $r : \xx \to \R$,
defined as
$r(x) = (d\mathbb{P}_{2,X}/d\mathbb{P}_{1,X})(x)$, for all $x\in \xx$.
One can construct an estimate $\hat r$ of the likelihood ratio
if one has access to additional unlabeled datapoints 
sampled i.i.d.~from the test distribution $\mathbb{P}_2$.
Methods for likelihood ratio estimation include 
using Bayes' rule to express it as a ratio of classifiers \citep{friedman2004multivariate,qiu2023prediction} and
domain adaptation \citep{Ganin2015,park2021pac}.
However, such estimates may be inaccurate for high-dimensional data.
This error propagates to the coverage 
of the resulting conformal predictor,
and the prediction sets may no longer attain the nominal coverage level.
Thus, it is natural to ask the following question:

\begin{center}
    \emph{Can one design a conformal prediction algorithm 
that attains valid coverage in the target domain,
without estimating the entire function $r$?}
\end{center}

In this paper, we present a method
that answers this question in the affirmative.
We construct our prediction sets by 
introducing and 
solving
a regularized quantile regression problem,
which combines the pinball loss 
with a novel data-dependent regularization term
that can be computed from
one-dimensional projections of the likelihood ratio $r$.
Crucially, the objective function can be estimated at the parametric rate, with only a mild dependence on the dimension of the feature space.
This regularization is specifically chosen
to ensure that the first order conditions of the pinball loss lead to coverage at test-time.
Geometrically, it turns out that
the regularization aligns the selected threshold function with the true likelihood ratio $r$. 
The resulting method, 
which we call likelihood ratio regularized quantile regression (LR-QR),
outperforms existing methods on high-dimensional datasets with covariate shift.

Our contributions include the following: 
\begin{itemize}
    \item We propose the LR-QR algorithm, which constructs a conformal predictor that adapts to covariate shift without directly estimating the likelihood ratio.
    \item We show that the minimizers of the population LR-QR objective have coverage in the test distribution. We also show that the minimizers of the empirical LR-QR objective lead to coverage up to a small error term that we can control, by drawing on a novel analysis of coverage via stability bounds from learning theory.
    \item We demonstrate the effectiveness of the LR-QR algorithm on high-dimensional datasets under covariate shift, including the Communities and Crime dataset, the RxRx1 dataset from the WILDS repository, and the MMLU benchmark. Here, we crucially leverage our theory by choosing the regularization parameter proportional to the theoretically optimal value. An implementation of LR-QR can be accessed at the following link: \url{https://github.com/shayankiyani98/LR-QR}. 
\end{itemize}
The structure of this paper is as follows.
In \Cref{pf}, we rigorously state the problem.
In \Cref{approach}, we present our method, 
as well as intuitions behind it.
In \Cref{algorithm}, we present the algorithm.
In \Cref{theory} we present our theoretical results, in both the infinite sample and finite sample settings.
In \Cref{experiments}, we present our experimental results on high-dimensional datasets with covariate shift.
All proofs are deferred to the appendix.

\subsection{Related work}

Here we only list prior work most closely related to our method; we provide more references in \Cref{app:related-work}.
The early ideas of conformal prediction were developed in \cite{saunders1999transduction,vovk1999machine}.
With the rise of machine learning, conformal prediction has emerged as a widely used framework for constructing prediction sets \citep[e.g.,][]{papadopoulos2002inductive,vovk2005algorithmic,Vovk2013}.
Classical conformal prediction guarantees validity when the calibration and test data are drawn from the same distribution. In contrast, when there is distribution shift between the calibration and test data \citep[e.g.,][]{quinonero2009dataset, shimodaira2000improving, Sugiyama2012, ben2010theory, taori2020measuring}, coverage may not hold.
Covariate shift is a type of dataset shift that arises in many settings, e.g., when predicting disease risk for individuals whose features may evolve over time, while the outcome distribution conditioned on the features remains stable \citep{quinonero2009dataset}. 

Numerous works have addressed conformal prediction under various types of distribution shift 
\citep{tibshirani2019conformal,park2021pac,park2022pac,qiu2023prediction,si2023pac}. For example, \cite{tibshirani2019conformal} investigated conformal prediction under covariate shift, assuming the likelihood ratio between source and target covariates is known. 
\cite{Lei2021} allowed the likelihood ratio to be estimated, rather than assuming it is known.
\cite{park2021pac} developed prediction sets with a calibration-set conditional (PAC) property under covariate shift. 
\cite{qiu2023prediction,yang2024doubly} developed prediction sets with asymptotic coverage that are doubly robust in the sense that their coverage error is bounded by the product of the estimation errors of the quantile function of the score and the likelihood ratio.
\cite{cauchois2024robust} construct prediction sets based on a distributionally robust optimization approach.

In contrast, our algorithm entirely avoids estimating the likelihood ratio function. Rather, it works by constructing a novel regularized regression objective, whose stationary conditions ensure coverage in the test domain.
We can minimize the objective by estimating certain expectations of the data distribution---which implicitly involve estimating only certain functionals of the likelihood ratio.
We further show that the coverage is retained in finite samples via a novel analysis of coverage leveraging stability bounds 
\citep{shalev2010learnability, shalev2014understanding}. We illustrate that our algorithms behave better in high-dimensional datasets than existing methods.

Aiming to achieve coverage under a predefined set of covariate shifts,
\cite{gibbs2025conformal} develop an approach based on minimizing the quantile loss over a linear function class.
We build on their approach, but develop a novel regularization scheme that 
allows us to 
effectively optimize over a data-driven class,
adaptive to the unknown shift $r$.
\section{Problem formulation}
\label{pf}
In this section we fix notation and state our problem.
\subsection{Preliminaries and notations}
For $\alpha \in (0,1)$,
recall that the quantile (pinball) loss $\ell_\alpha$ is defined for all $c,s\in \R$ as
$$
\ell_\alpha(c, s):=\left\{\begin{array}{l}
(1-\alpha)(s-c) \text { if } s \geq c, \\
\alpha(c-s) \text { if } s<c.
\end{array}\right.
$$
For any distribution $P$, 
the minimizers of
$c\mapsto  
\E_{S\sim P}[\ell_\alpha(c, S)]$
are the $(1-\alpha)$th quantiles of $P$.

Let the \emph{source} or \emph{calibration} distribution be denoted
$\mathbb{P}_1 = \mathbb{P}_{1,X} \times \mathbb{P}_{Y|X}$, 
and let the \emph{target} or \emph{test} distribution be denoted 
$\mathbb{P}_2 = \mathbb{P}_{2,X} \times \mathbb{P}_{Y|X}$, 
a covariate shifted version of the calibration distribution.
Let $\E_i$ denote the expectation over $\mathbb{P}_i$, $i=1,2$.
Let $x\mapsto r(x) = \frac{d\mathbb{P}_{2,X}}{d\mathbb{P}_{1,X}}(x)$ denote the unknown likelihood ratio function.

Recall that a prediction set $C : \xx \to 2^{\yy}$
has marginal $(1-\alpha)$-coverage in the test domain
if $\PP[2]{ Y \in C(X) } \ge 1-\alpha$.
Observe that $\PP[2]{ Y\in C(X) }$ can be rewritten as
$\E_2[ \mathbf{1}[Y\in C(X)] ]$,
where $\mathbf{1}[\cdot]$ denotes an indicator function.
Let $S:(x,y)\mapsto S(x,y)$ denote the nonconformity score associated to a pair $(x,y) \in \xx\times \yy$.
Given a \emph{threshold function} $q : \xx \to \R$, 
we consider the corresponding conformal predictor $C : \xx \to 2^{\yy}$ given by 
\begin{equation} \label{C_d_def}
C(x) = \{ y \in \yy: S(x,y) \le q(x) \} 
\end{equation}
for all $x\in \xx$.
Thus a threshold function $q$ yields a conformal predictor with marginal $(1-\alpha)$-coverage in the \emph{test} domain if
$\mathbb{P}_2[S(X,Y) \le q(X)] \ge 1-\alpha$.
We assume that $\alpha \le 0.5$. For our theory, we consider $[0,1]$-valued  scores.

In this paper, a linear function class refers to a linear subspace of functions from $\xx \to \R$ that are square-integrable with respect to $\mathbb{P}_{1,X}$.
An example is the space of functions representable by a pretrained model with a scalar read-out layer. 
If $\Phi : \xx \to \R^d$ denotes the last hidden-layer feature map of the pretrained model, where $\Phi = (\phi_1,\ldots,\phi_d)$ for $\phi_i : \xx \to \R$ for all $i\in [d]$, then the linear class of functions representable by the network is given by $\{ \langle \gamma, \Phi\rangle : \gamma \in \R^d\}$, where $\langle \cdot,\cdot\rangle$ is the $\ell^2$ inner product on $\R^d$.

\subsection{Problem statement}

We observe $n_1$ labeled calibration  (or, source) datapoints $\{(X_i, Y_i) : i \in [n_1]\}$ drawn i.i.d.~from the source distribution $\mathbb{P}_1$, and an additional $n_3$ unlabeled calibration datapoints $\mathcal{S}_3$.
We  also have $n_2$ \textit{unlabeled} (target) datapoints $\mathcal{S}_2$ drawn i.i.d.~from the target distribution $\mathbb{P}_2$.
Given $\alpha \in (0,1)$, our goal is to construct
a threshold function $q : \xx \to \R$
that achieves marginal $(1-\alpha)$-coverage in the test domain:
$
    \mathbb{P}_2[S(X,Y) \le q(X)] \ge 1-\alpha.
$

\section{Algorithmic principles}
\label{approach}

Here we present the intuition behind our approach.
Our goal is to construct a prediction set of the form
$C(x) = \{ y \in \yy: S(x,y) \le \tilde q(x) \} $, where $\tilde q$ should be close to a conditional quantile of $S$ 
given $X=x$.
The quantile loss $\ell_\alpha$ is designed such that for any random variable $S$,
the minimizers of the objective 
$\kappa\mapsto  
\E\ell_\alpha(\kappa, S)$
are the $(1-\alpha)$th quantiles of $S$.
This has motivated prior work
\citep{jung2023batch,gibbs2025conformal}, 
where the authors
minimize the objective
$h\mapsto  
\E\ell_\alpha(h(X), S(X,Y))$ for $h$ in 
some linear hypothesis class $\mathcal{H}$.
At a minimizer $h^*$, the derivatives in all directions $g \in \mathcal{H}$ should be zero, so that
\begin{align}\label{deriv}
     \frac{\partial}{\partial \ep}\bigg|_{\ep=0} \E_1[\ell_{\alpha}(h^*(X) + \ep g(X), S )]
     &=\E_1[ g(X) (\mathbf{1}[S(X,Y) \le h^*(X)] - (1-\alpha)) ] = 0.
\end{align}
If $g$ takes the form $g(x) = d\mathbb{Q}_{X}/d\mathbb{P}_{1,X}(x)$ for some distribution $\mathbb{Q}_{X}$, then\footnote{This holds due to the change of measure identity
 $\E_P[dQ/dP(X) h(X)] = \E_Q[h(X)]$ for all integrable functions $h$.}
this equality can be viewed as exact coverage 
under the 
covariate shift induced by $g$ for the prediction set $x\mapsto \{y\in \yy: S(x,y)\le h^*(x)\}$.
In other words, if the test distribution is $\mathbb{Q} = \mathbb{Q}_{X} \times \mathbb{P}_{Y|X}$, then we have the exact coverage result
\begin{align*}
    \E_{\mathbb{Q}}[\mathbf{1}[S(X,Y)\le h^*(X)] ] =
    \mathbb{Q}[S(X,Y)\le h^*(X)] = 1-\alpha.
\end{align*}
Therefore, if 
the hypothesis class $\mathcal{H}$ is large enough to
include 
the true likelihood ratio $r=d\mathbb{P}_{2,X}/d\mathbb{P}_{1,X}$,
then 
the threshold function $h^*$ attains valid coverage in the test domain $\mathbb{P}_2$, as desired.

\subsection{Our approach}

{\bf An adaptive choice of the hypothesis class.}
The above approach requires special assumptions on 
 the 
hypothesis class $\cH$.
The choice of the hypothesis class poses a challenge in practice: if $\cH$ is too small, then coverage may fail, while if $\cH$ is too large, then finite-sample performance may suffer due to large estimation errors.

To address this challenge, our idea is to choose $\cH$ adaptively. 
We start by considering 
the class 
of 
hypotheses 
$h$ that are close to
the true likelihood ratio
$r$, as measured by 
$\E_1[(h(X)-r(X))^2]$ being small.
By our remarks above,
if we minimize $\E_1[\ell_{\alpha}(h(X), S(X,Y))]$
for $h$ restricted to this set,
we obtain 
a threshold function
with valid coverage under the covariate shift $r$.

{\bf Removing the explicit dependence on the likelihood ratio.} The quantity $\E_1[(h(X)-r(X))^2]$ depends on the unknown $r$.
However, we can expand this to obtain
\begin{align*}
    \E_1[(h(X)-r(X))^2] &=
   \E_1[h(X)^2] + \E_1[-2 r(X)h(X)] 
   + \E_1[r(X)^2].
\end{align*}
The term $\E_1[r(X)^2]$ does not depend on the optimization variable $h$, so it is enough to consider the first two terms.
Due to the change-of-measure identity
$\E_1[r(X) h(X)] = \E_2[h(X)]$, the sum of these terms equals
\begin{align*} 
   \E_1[h(X)^2] + \E_1[-2 r(X) h(X)] 
   = 
   \E_1[h(X)^2] + \E_2[-2 h(X)].
\end{align*}
A key observation is that
neither of the terms $\E_1[h(X)^2]$ or $\E_2[-2h(X)]$ explicitly involve $r$, and thus they can be estimated by sample averages over the source and target data, respectively.
Thus,
we can minimize $\E_1[\ell_{\alpha}(h(X), S(X,Y))]$ over
$h \in \mathcal{H}$
while keeping 
$\E_1[h(X)^2] + \E_2[-2h(X)]$
bounded.
The threshold $h^*$ will have valid coverage under the covariate shift $r$.

\textbf{Introducing a normalizing scalar.}
We also need to make sure that $h$ is a valid likelihood ratio under $d\mathbb{P}_{1,X}$, 
of the form
$g(x) = d\mathbb{Q}_{X}/d\mathbb{P}_{1,X}(x)$ for some distribution $\mathbb{Q}_{X}$.
This imposes the constraint $\int h(x) d\mathbb{P}_{1,X}(x) = 1$, 
which can be equivalently achieved for any non-negative $h$ by scaling it with an appropriate scalar $\beta$.
In our analysis, it turns out to be convenient to use the optimization variable $\beta h$ 
and consider the class of functions $h$ such that
$\E_1[(\beta h(X) - r(X))^2]$ is bounded for some scalar $\beta\in\R$.
By the above discussion, 
the term $\E_1[r(X)^2]$ is immaterial and 
it is sufficient to impose the constraint that 
$\min_{\beta\in \R} (\E_1[\beta^2 h(X)^2] + \E_2[-2\beta h(X)])$ is bounded.

\textbf{Replacing the constraint with a regularization.}
Instead of imposing a constraint on 
$\min_{\beta\in \R} (\E_1[\beta^2 h(X)^2] + \E_2[-2\beta h(X)])$, 
we can use this term as a regularizer.
Given a regularization strength $\lambda \ge 0$,
we can solve
\begin{align*}
    \min_{h\in\cH} 
    \left\{\E_1[\ell_{\alpha}(h(X), S(X,Y))] + \lambda \min_{\beta\in \R} (\E_1[\beta^2 h(X)^2] + \E_2[-2\beta h(X)])\right\}.
\end{align*}
Since the first term does not depend on $\beta$, this is equivalent to the joint optimization problem
\begin{align}\label{opt:unconstr-CP-obj}\tag{LR-QR}
    \min_{h\in\cH, \beta\in \R} 
     \left\{L_{\lambda}(h,\beta):= \E_1[\ell_{\alpha}(h(X), S(X,Y))] + \lambda (\E_1[\beta^2 h(X)^2] - \E_2[2\beta h(X)])\right\}.
\end{align}

\subsection{Algorithm: likelihood ratio regularized quantile regression}
\label{algorithm}

We solve an empirical version of this objective.
We use 
our labeled source data $\{(X_i, Y_i) : i \in [n_1]\}$ to estimate $\E_1[\ell_{\alpha}(h(X), S(X,Y))]$, 
our additional unlabeled source data $\mathcal{S}_3$ to estimate $\E_1[ \beta^2 h(X)^2 ]$, and our unlabeled target data $\mathcal{S}_2$ to estimate $\lambda \E_2[ -2\beta h(X) ]$.
Letting $\hat \E_1$, $\hat \E_2$, and $\hat \E_3$  denote empirical expectations over $\{(X_i, Y_i) : i \in [n_1]\}$, $\mathcal{S}_2$, and $\mathcal{S}_3$, respectively, we
then solve the following empirical likelihood ratio regularized quantile regression problem, for $\lambda \ge 0$:
\begin{align}\label{opt:emp-unconstr-CP-obj}\tag{Empirical-LR-QR}
     (\hat h, \hat \beta) \in \arg\min_{h\in \cH, \beta\in \R} 
     \left\{\hat L_{\lambda}(h,\beta) :=     \hat \E_1[\ell_{\alpha}(h(X), S(X,Y))] + \lambda \hat \E_3[ \beta^2 h(X)^2 ] - \lambda \hat \E_2[ 2\beta h(X) ]\right\}.
\end{align}
Our proposed threshold is $q = \hat h$. See \Cref{alg:lr-qr}.
In the following section, we justify this algorithm through a novel theoretical analysis of the test-time coverage.

\begin{algorithm}[tb]
   \caption{Likelihood-ratio regularized quantile regression}
   \label{alg:lr-qr}
\begin{algorithmic}
   \item {\bfseries Input:}
   $n_1$ labeled source datapoints, 
   $n_2$ unlabeled target  datapoints, 
   $n_3$ unlabeled source datapoints
   \vspace{0.5em}
   \item 1: Compute scores $S_i = S(x_i, y_i)$ for all $i\in [n_1]$ 
   \vspace{0.5em}
   \item 2: Solve
   $(\hat h, \hat \beta) \in \arg\min_{h\in \cH, \beta\in \R}     \hat \E_1[\ell_{\alpha}(h(X), S(X,Y))] + \lambda \hat \E_3[ \beta^2 h(X)^2 ] + \lambda \hat \E_2[ -2\beta h(X) ]$, where $\hat \E_1, \hat \E_2, \hat \E_3$ denote expectations over the source, unlabeled target, and
   unlabeled source data; 
   \vspace{0.5em}
   \item {\bfseries Return:} Prediction set $\hat C(x) \leftarrow \{ y \in \yy : S(x,y) \le \hat h(x) \}$ with asymptotic $1-\alpha$ coverage in the target distribution
\end{algorithmic}
\end{algorithm}

\section{Theoretical results}
\label{theory}

\subsection{Infinite sample setting}
We first consider the infinite sample or ``population" setting,
characterizing the solutions of the LR-QR problem from \eqref{opt:unconstr-CP-obj} in an idealized scenario 
 where the exact values of the expectations
 $\mathbb{E}_1, \mathbb{E}_2$ can be calculated. 
In this case, we will show that
if the hypothesis class $\cH$ is linear and contains the true likelihood ratio $r$,
then the optimizer achieves valid coverage in the test domain.
Let $r_{\cH}$ be
the projection of $r$ onto $\cH$
in the Hilbert space induced by the inner product $\langle f,g\rangle = \E_1[fg]$.
The key step is the result below, which characterizes coverage weighted by $r_{\cH}$.

\begin{proposition}\label{prop:infinite-sample}
Let $\cH$ be a linear hypothesis class consisting of square-integrable functions with respect to $\mathbb{P}_{1,X}$. 
Then under regularity conditions specified in \Cref{sec:conds} (the conditions of \Cref{lem:unconstr-existence}),
if $(h^*, \beta^*)=(h^*_\lambda, \beta^*_\lambda)$ is a minimizer of the objective in \Cref{opt:unconstr-CP-obj} with regularization strength $\lambda > 0$,
then we have 
$\E_1[r_{\cH}(X) \mathbf{1}[S(X,Y) \le h^*(X)]] \ge 1-\alpha.$
\end{proposition}
The proof is given in \Cref{pf:infinite-sample}.
As a consequence of \Cref{prop:infinite-sample},
if $\cH$ contains the true likelihood ratio $r$,
so that $r_{\cH}=r$,
then in the infinite sample setting,
the LR-QR threshold function $h^*$ attains valid coverage at test-time:
$$\E_1[r(X) \mathbf{1}[S(X,Y) \le h^*(X)]] =
\mathbb{P}_2[S(X,Y) \le h^*(X)]
\ge 1-\alpha.$$
However, in practice, we can only optimize over finite-dimensional hypothesis classes, and as a result we must control the effect of mis-specifying $\cH$.
If $r$ is not in $\cH$, we can derive a lower bound on the coverage as follows.
First, write
\begin{align*}
    &\E_1[ r(X) \mathbf{1}[S(X,Y)\le h^*(X)] ] \\
    &= \E_1[ r_{\cH}(X) \mathbf{1}[S(X,Y)\le h^*(X)] ] 
    + \E_1[ (r(X) - r_{\cH}(X)) \mathbf{1}[S(X,Y)\le h^*(X)] ].
\end{align*}
By \Cref{prop:infinite-sample}, the first term on the right-hand side is at least $1-\alpha$.
Since the random variable $\mathbf{1}[S(X,Y)\le h^*(X)]$ is $\{0,1\}$-valued, the Cauchy-Schwarz inequality implies that the second term on the right-hand side is at least $-\E_1[(r(X)-r_{\cH}(X))^2]^{1/2}$.
We set our threshold function $q$ to equal $h^*$, so that our conformal prediction sets equal $C^*(x) = \{ y\in \yy : S(x,y) \le h^*(x) \}$ for all $x\in \xx$.
Thus, we have the lower bound
\begin{align*}
\mathbb{P}_2[Y \in C^*(X)] &= 
    \E_1[ r(X) \mathbf{1}[S(X,Y)\le h^*(X)] ] 
    \ge (1-\alpha) - \E_1[(r(X)-r_{\cH}(X))^2]^{1/2}.
\end{align*}
Geometrically, this coverage gap is the result of restricting to $\cH$; in fact, $\E_1[(r(X)-r_{\cH}(X))^2]^{1/2}$ is the distance from $r$ to $\cH$. This error decreases if $\cH$ is made larger, but in the finite sample setting, this comes at the risk of overfitting.

\subsection{Finite sample setting}
From the analysis of the infinite sample regime,
it is clear that if the hypothesis class $\cH$ is made larger,
the test-time coverage of the population level
LR-QR threshold function $h^*$
moves closer to the nominal value.
However, in the finite sample setting,
optimizing over a larger hypothesis class
also presents the risk of overfitting.
By tuning the regularization parameter $\lambda$,
we are 
trading off 
the estimation error incurred for the first term of \Cref{opt:unconstr-CP-obj},
namely $(\hat \E_1-\E_1)[\ell_{\alpha}(h(X),S(X,Y))]$,
and the error incurred for the second and third terms of \Cref{opt:unconstr-CP-obj},
namely $\lambda (\hat \E_3-\E_3)[\beta^2 h(X)^2] + \lambda (\hat \E_2 - \E_2)[-2\beta h(X)]$.
Heuristically, for a fixed $h$,
the former should be proportional to $1/\sqrt{n_1}$,
and the latter should be proportional to $\lambda(1/\sqrt{n_3} + 1/\sqrt{n_2})$.
Thus, if we pick $\lambda$ to make these two errors of equal order, it will be proportional to $\sqrt{(n_2+n_3)/n_1}$.

Put differently, in order to ensure that the Empirical LR-QR threshold $\hat h$ from \Cref{opt:emp-unconstr-CP-obj} has valid test coverage,
one must choose the regularization $\lambda$ based on the relative amount of labeled and unlabeled data.
The unlabeled datapoints carry information about the covariate shift $r$, because $r$ depends only on the  distribution of the features.
The labeled datapoints provide information about 
the 
\emph{conditional $(1-\alpha)$-quantile function} $q_{1-\alpha}$, which depends only on the conditional distribution of $S|X$.
When $\lambda$ is large, our optimization problem places more weight on approximating $r$ (the minimizer of $\E_1[(\beta h(X)-r(X))^2]$ in $\beta h$), and if $\lambda$ is small, we instead aim to approximate $q_{1-\alpha}$ (the minimizer of $\E_1[ \ell_{\alpha}(h(X),S(X,Y)) ]$ in $h$). 
Therefore, if the number of unlabeled datapoints ($n_2+n_3$) is large compared to the number of labeled datapoints ($n_1$), our data contains much more information about the covariate shift $r$, and we should set $\lambda$ to be large. If instead $n_1$ is very large, the quantile function $q_{1-\alpha}$ can be well-approximated from the labeled calibration datapoints, and we set $\lambda$ to be close to zero.
In the theoretical results, we make this intuition precise.

In order to facilitate our theoretical analysis in the finite sample setting, we consider constrained versions of \Cref{opt:unconstr-CP-obj} and \Cref{opt:emp-unconstr-CP-obj}.
Fix a collection $\Phi = (\phi_1, \ldots, \phi_d)^\top $ of $d$ basis functions, where $\phi_i : \xx \to \R$ for $i\in [d]$.
Let $\mathcal{I} = [\beta_{\tn{min}}, \beta_{\tn{max}}] \subset \R$ be an interval with $\beta_{\tn{min}} > 0$.
Let $\cH_B= \{ \langle \gamma, \Phi \rangle : \|\gamma\|_2 \le B < \infty \}$ be the $B$-ball centered at the origin in the linear hypothesis class spanned by $\{ \phi_1, \ldots, \phi_d \}$.
We equip $\cH_B$ with the norm $\|h\| = \|\gamma\|_2$ for $h = \langle \gamma, \Phi \rangle$.

At the population level, consider the following constrained LR-QR problem:
$   (h^*, \beta^*) \in \arg \min_{h\in \cH_B, \beta\in \mathcal{I}} L_{\lambda}(h, \beta)$.
Also consider the following empirical constrained LR-QR problem\footnote{For brevity, this notation overloads the definition of $(\hat h, \hat \beta)$ from \eqref{opt:emp-unconstr-CP-obj}. From now on, $(\hat h, \hat \beta)$ will refer to the definition from \eqref{opt:emp-constr-CP-obj}, and the one from \eqref{opt:emp-unconstr-CP-obj} will not be used again.}:
\begin{align}\label{opt:emp-constr-CP-obj}
     (\hat h, \hat \beta) \in \arg\min_{h\in \cH_B, \beta\in \mathcal{I}} \hat L_{\lambda}(h,\beta).
\vspace{-0.5em}
\end{align}
We begin by bounding the generalization error of an ERM $(\hat h, \hat \beta)$ computed via \Cref{opt:emp-constr-CP-obj}.


\begin{theorem}[Suboptimality gap of ERM for likelihood ratio regularized quantile regression]\label{thm:generalize}
Under the regularity conditions specified in \Cref{sec:conds}, and for appropriate choices of the optimization hyperparameters\footnote{Specifically, suppose that $\beta_{\tn{min}} \le \beta_{\tn{lower}}$, $\beta_{\tn{max}} \ge \beta_{\tn{upper}}$, and $B \ge B_{\tn{upper}}$, 
where the positive scalars $\beta_{\tn{lower}}$, $\beta_{\tn{upper}}$, and $B_{\tn{upper}}$ are defined in \Cref{lem:apriori-bdd} in the Appendix, and depend on the data distribution and the choice of basis functions, but not on the data, the sample sizes, or the regularization parameter $\lambda$.}, for sufficiently large $n_1, n_2, n_3$, with probability at least $1-\delta$, 
any optimizer $(\hat h, \hat \beta)$ of the  empirical constrained LR-QR objective
from \eqref{opt:emp-constr-CP-obj} with regularization strength $\lambda > 0$ has
suboptimality gap $L_{\lambda}(\hat h, \hat \beta) - L_{\lambda}(h^*, \beta^*)$
with respect to the population risk \eqref{opt:unconstr-CP-obj} bounded by
\begin{align*}
    &\ee_{\tn{gen}} := c \lambda \sqrt{{1}/{n_2} +{1}/{n_3}} + {c'}/{\sqrt{n_1}}
    + {c''}/{\sqrt{\lambda n_1}},
\end{align*}
and $c,c',c''$ are positive scalars that do not depend on $\lambda$.
\end{theorem}

The proof is in \Cref{pf:generalize}.
The generalization error $\ee_{\tn{gen}}$ is minimized for an optimal regularization on the order of
\begin{align}\label{optimal-lambda}
    \lambda^* \propto {n_1^{-1/3}} \left( {1}/{n_2} +{1}/{n_3} \right)^{-1/3},
\end{align}
which yields an optimized upper bound of order
$\ee_{\tn{gen}}^* = O\left( {n_1^{-1/3}} \left( {1}/{n_2} +{1}/{n_3} \right)^{1/6} + 1/\sqrt{n_1} \right)$.

As a corollary of \Cref{thm:generalize}, we have the following lower bound on the excess marginal coverage of our ERM threshold $\hat h$ in the covariate shifted domain.
    Let $r_{B}$ denote the projection of $r$ onto the closed convex set $\cH_B$ in the Hilbert space induced by the inner product $\langle f,g\rangle = \E_1[fg]$. 
    
\begin{theorem}[Main result: Coverage under covariate shift]\label{thm:cov-lower-bd}
    Under the same conditions as \Cref{thm:generalize}, 
    consider the LR-QR optimizers $\hat h$ and $\hat \beta$ from \eqref{opt:emp-constr-CP-obj} with regularization strength $\lambda > 0$.
    Given any $\delta > 0$, for sufficiently large $n_1,n_2,n_3$, we have with probability at least $1-\delta$ that\footnote{The probability $\PP[2]{ Y\in \hat C(X) }$ is over $(X,Y)\sim\mathbb{P}_2$, conditional on $\hat C$.} 
    \begin{align*}
        \PP[2]{ Y \in \hat C(X) } &\ge (1-\alpha) + 2 \hat \beta \lambda \E_1[ (r_{B}(X) - \hat \beta \hat h(X))^2 ]- \ee_{\tn{cov}} - (1-\alpha) \E_1[ |r(X) - r_{B}(X)| ],
    \end{align*}
    where
$
\ee_{\tn{cov}} := 
A \left({1}/{n_2} +{1}/{n_3}\right)^{1/4}
\lambda
+ A' (\lambda n_1)^{-1/4}
+ \lambda^{1/2} /n_1^{1/4},
$
    and $A, A'$ are positive scalars that do not depend on $\lambda$.
\end{theorem}
The proof is in \Cref{pf:cov-lower-bd}.
This result states that our LR-QR method has nearly valid coverage at level $1-\alpha$ under covariate shift, up to small error terms that we can control.
The quantity $\ee_{\tn{cov}}$ vanishes as we collect more data.
The term $\E_1[ |r(X) - r_{B}(X)| ]$ captures the level of mis-specification by not including the true likelihood ratio function $r$ in our hypothesis class $\cH_B$. This can be decreased by making the hypothesis class $\cH_B$ larger. Of course, this will also increase the size of the terms $A,A'$ in our coverage error, but in our theory we show that the dependence is mild. Indeed, the terms depend only on a few geometric properties of $\cH_B$ such as the eigenvalues of the sample covariance matrix of the basis $\Phi(X)$ under the source distribution, and a quantitative measure of linear dependence of the features; but not explicitly on the dimension of the basis. 

We highlight the term $2\hat \beta \lambda\E_1[ (r_{B}(X) - \hat \beta \hat h(X))^2 ]$, which is an error term relating the projected likelihood ratio $r_{B}$ to the LR-QR solution $\hat \beta \hat h$.
Crucially, this term is a non-negative quantity multiplied by $\lambda$, and so for appropriate $\lambda$ it may counteract in part the coverage error loss.
Consistent with the above observations, we find empirically that choosing small nonzero regularization parameters improves coverage.
Moreover, we find that choosing the regularization parameter to be on the order of the optimal value for $\ee_{\tn{cov}}$ is suitable choice across a range of experiments.

Our proofs are quite involved and require a number of delicate arguments.
Crucially, they draw on a 
\emph{novel analysis of coverage via stability bounds} from learning theory.
Existing stability results cannot directly be applied,
due to our use of a data-dependent regularizer.
For instance, in classical settings, the optimal regularization tends to zero as the sample size goes to infinity, but this is not the case here.
To overcome this challenge, we combine stability bounds \citep{shalev2010learnability, shalev2014understanding}
with a novel conditioning argument, and we show that the values of $L$ at the minimizers of $\hat L$ and $L$ are close by introducing intermediate losses that sequentially swap out empirical expectations $\hat \E_1, \hat \E_2, \hat \E_3$ with their population counterparts.
We then leverage the smoothness of $L$, to derive that the gradient of $L$ at $(\hat \beta, \hat h)$ is small. 
Finally, we show that a small gradient implies the desired small coverage gap.

As an immediate corollary of \Cref{thm:cov-lower-bd},
we have the following result,
which states that the LR-QR algorithm can be used
to construct prediction sets with group-conditional coverage
for a finite set of potentially overlapping groups.

\begin{corollary}[Group-conditional coverage]\label{cor:gp-cond-cov}
    Let $G_1, \ldots, G_d \subseteq \xx$ be a finite collection of distinct subsets of $\xx$
    such that $\PP[1]{G_i} > 0$ for all $i \in [d]$
    and $\PP[1]{G_i \triangle G_j} > 0$ for all $i, j \in [d]$ with $i\neq j$,
    where $\triangle$ denotes symmetric difference.
    For $i\in [d]$,
    let $\phi_i : \xx \to \R$ be given by
    $\phi_i(x) = \mathbf{1}[x\in G_i]$,
    and consider the basis $\Phi = (\phi_1, \ldots, \phi_d)^\top$.
    Under the same conditions as \Cref{thm:generalize}, 
    consider the LR-QR optimizers $\hat h$ and $\hat \beta$ from \eqref{opt:emp-constr-CP-obj} with basis given by $\Phi$
    and regularization strength $\lambda > 0$.
    Given any $\delta > 0$, for sufficiently large $n_1,n_2,n_3$, we have with probability at least $1-\delta$ that
    \begin{align*}
        \PP[1]{ Y \in \hat C(X) | X \in G_i } &\ge (1-\alpha) 
        + 2 \hat \beta \lambda \E_1[ (c_i \phi_i(X) - \hat \beta \hat h(X))^2 ]- \ee_{\tn{cov}} - (1-\alpha) (1 - c_i \PP[1]{G_i})
    \end{align*}
    for each $i\in [d]$,
    where $c_i = \min\{ 1 / \PP[1]{G_i}, B \}$
    and $\ee_{\tn{cov}}$ is defined as in \Cref{thm:cov-lower-bd}.
\end{corollary}

\section{Experiments}
\label{experiments}

We compare our method
with the following baselines:
(1) Split/inductive conformal prediction  \citep{papadopoulos2002inductive, lei2018distribution};
(2) Weighted-CP: Weighted conformal prediction \citep{tibshirani2019conformal};
(3) 2R-CP: The \emph{doubly robust} method from \cite{yang2024doubly};
 (4) DRO-CP: Distributionally robust optimization \citep{cauchois2024robust};
 (5) DR-iso: Isotonic distributionally robust optimization \citep{gui2024distributionally};
 (6) Robust-CP: Robust weighted conformal prediction
 \citep{ai2024not}.

 \subsection{Choosing the regularization parameter} \label{subsec:choose_lambda}

\Cref{optimal-lambda} suggests an optimal choice of the regularization parameter $\lambda$ in the LR-QR algorithm. Guided by this, we form a uniform grid of size ten from $\lambda^*/{10}$  to $\lambda^*$. 
We then perform three-fold cross-validation over the combined calibration and unlabeled target datasets (without using any labeled test data) as follows: 
we train the LR-QR threshold for each $\lambda$, 
and compute as a validation measure the $\ell^2$-norm of the gradient of the LR-QR objective on the held-out fold. We pick $\lambda$ with the smallest average validation measure across all folds.

This validation measure is motivated by our algorithmic development: the first-order conditions of the LR-QR objective play a fundamental role in ensuring valid coverage in the test domain. While the model is trained to satisfy these conditions on the observed data, we seek to ensure this property generalizes well to unseen data. Thus, our selection criterion is based on two key observations: (1) a small gradient of the LR-QR objective implies reliable coverage, and (2) the regularization parameter $\lambda$ balances the generalization error of the two terms in LR-QR. By minimizing this measure, we select a $\lambda$ that optimally trades off these competing factors.

Finally, we re-train the LR-QR threshold on the entire calibration and unlabeled target datasets using this best $\lambda$, and report coverage and interval size on the held-out labeled test set. This ensures that no test labels are used during hyperparameter tuning.
Additionally, in \Cref{ablation}, we provide deeper insights on different regimes of regularization in practice through an ablation study.

\subsection{Communities and Crime} \label{reg_exp}

We evaluate our methods on the \emph{Communities and Crime} dataset \citep{UCIrepo}, which contains 1994 datapoints corresponding to communities in the United States,
with socio-economic and demographic statistics. 
The task is to predict the (real-valued) per-capita violent crime rate from a 127-dimensional input.

\begin{figure*}
\centering
\includegraphics[width=\textwidth]{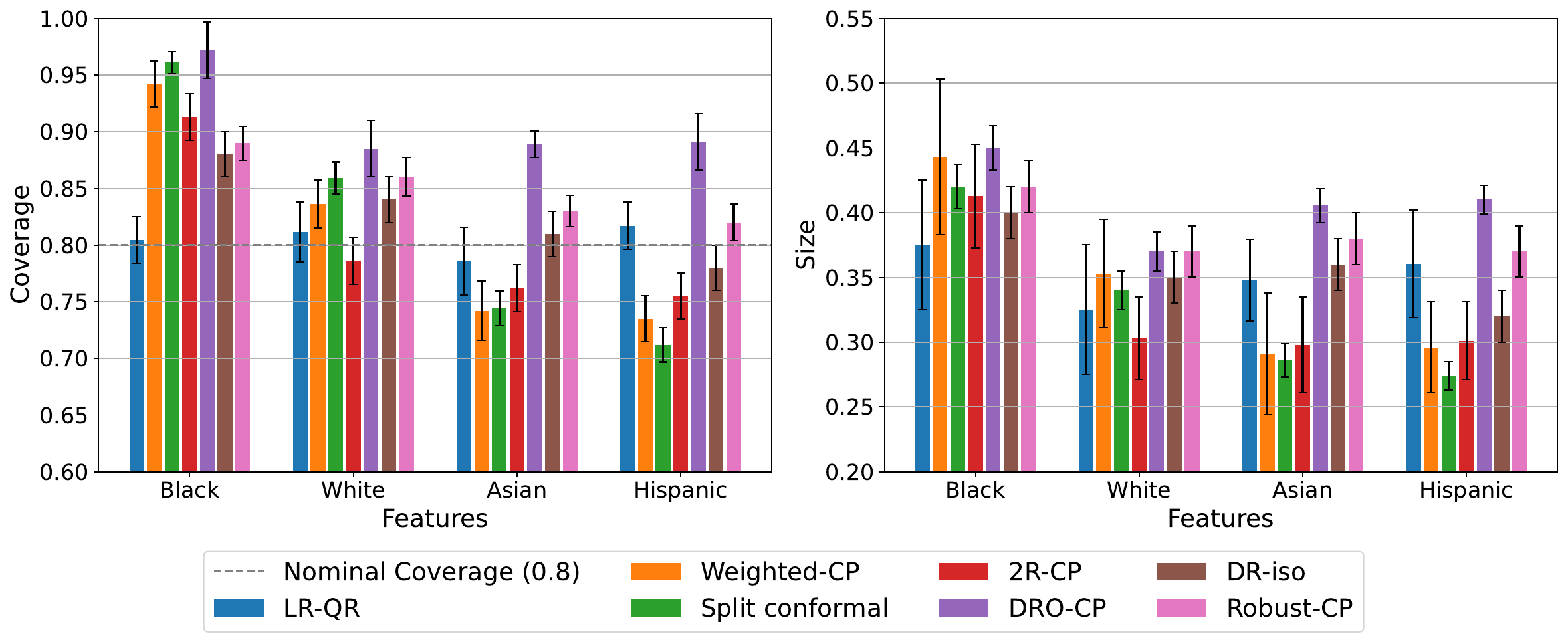}
\caption{(Left)~Coverage. (Right)~Average prediction set size on the Communities and Crime dataset.}
\label{Fig_crime}
\end{figure*}

We first randomly select half of the data as a training set, and use it to fit a ridge regression model $\hat f$ as our predictor. We tune the ridge regularization with five-fold cross-validation.
We use the remaining half to  design four covariate shift scenarios, determined by the frequency of a specific racial subgroup (Black, White, Hispanic, and Asian). 
For each of these features, we find the median value $m$ over the remaining dataset. 
Datapoints with feature value at most $m$ 
form our \emph{source} set, 
and the rest form our \emph{target} set. 
This creates a covariate shift between calibration and test, as the split procedure only observes the covariates and is independent of labels.
We then further split the target set into 
roughly equal \emph{unlabeled} and \emph{labeled} subsets. 
The unlabeled subset and the calibration data (without the labels) is used  to estimate $r$,
while the labeled test subset is held out \emph{only} for final evaluation. 
The same procedure is applied to each of the four racial subgroups, creating four distinct partitions.

{\bf Experimental details.}
The nonconformity score is 
$
s(x,y) \;=\; |y - \hat f(x)\,|.
$
Several baselines require an estimate of the likelihood ratio $r$,
which we obtain by training a logistic regression model $\hat p$ to distinguish unlabeled source and target data. 
We then set $\hat r= \frac{\hat{p}}{1-\hat{p}}$, where $\hat{p}(x)$ is the predicted probability that $x$ came from the target distribution.
The hypothesis class $\mathcal{H}$ consists of all linear maps from the feature space to $\R$. 
All experimental results are averaged over 1000 random splits.

{\bf Results.} \Cref{Fig_crime} displays the results. Notably, split conformal undercovers in two setups and overcovers in the other two.
Methods that estimate $r$ and DRO fail to track the nominal coverage, particularly in the first setup on the left. However, the LR-QR method is closer to the nominal level of coverage, 
showing a stronger adaptivity to the covariate shift.

\subsection{Multiple choice questions - MMLU}

We evaluate all methods using the MMLU benchmark, which covers 57 subjects spanning a wide range of difficulties. To induce a covariate shift, we partition the dataset by subject difficulty: prompts from subjects labeled as \textit{elementary} or \textit{high school} are used for calibration, while those from \textit{college} and \textit{professional} subjects form the test set. 

Motivated by the design from \cite{kumar2023conformal}, we follow a prompt-based scoring scheme adapted for LLMs: we append the string ``The answer is the option:'' to the end of each MMLU question and feed the resulting prompt into the Llama 13B model without generating any output. We then extract the next-token logits corresponding to the first decoding position (i.e., immediately after the prompt) and consider the logits associated with the characters \texttt{A}, \texttt{B}, \texttt{C}, and \texttt{D}. These four logits are normalized using the softmax function to produce a probability vector over the answer options. 

{\bf Experimental details.}
The nonconformity score is \(s(x,y) = 1 - f(x)_y\), where \(f(x)_y\) is the probability assigned to the correct answer. 
For $\hat r$ and $\cH$, we compute prompt embeddings as follows. 
We extract the final hidden layer outputs from GPT-2 Small to obtain 768-dimensional embeddings.
We then apply average pooling across all token embeddings in a prompt to obtain a single fixed-length vector representation for each input.
We fit a probabilistic classifier \(\hat{p}\) using logistic regression on the unlabeled pooled embeddings from the source and target data, 
and we set $\hat{r} = \frac{\hat p}{1-\hat p}$.
We set 
$\mathcal{H}$ to be a linear head on top of the representation layer 
of the pretrained model. 

\textbf{Results.} As shown in Table~\ref{table}, our LR-QR method achieves near-nominal coverage and has the smallest average prediction set size among methods that achieve approximately 90\% or higher coverage, demonstrating both validity and efficiency under covariate shift.

\begin{table}[ht] 
\centering
\caption{Comparison of Methods by Coverage and Set Size}
\begin{tabular}{l@{\hskip 6pt}*{8}{>
{\centering\arraybackslash}p{1.4cm}@{\hskip 6pt}}}
\toprule
Metric & Nominal & LR-QR & DRO & WCP & SCP & DR-iso & Robust-CP & 2R-CP \\
\midrule
Coverage (\%) & 90.0 & 89.6 & 99.7 & 86.5 & 78.1 & 96.3 & 95.8 & 96.9 \\
Set Size      & -- & 3.38 & 3.92 & 3.31 & 2.60 & 3.64 & 3.56 & 3.80 \\
\bottomrule
\end{tabular}\label{table}
\end{table}

\section{Discussion and future work}\sloppy
Distribution shifts are inevitable in machine learning applications. Consequently, precise uncertainty quantification under distribution shifts is essential to ensuring the safety and reliability of predictive models in practice. This challenge becomes even more pronounced when dealing with high-dimensional data, where classical statistical procedures often fail to generalize effectively. In this work, we develop a new conformal prediction method, which we call LR-QR, designed to provide valid test-time coverage under covariate shifts between calibration and test data. In contrast to existing approaches in the literature, LR-QR avoids directly estimating the likelihood ratio function between calibration and test time. Instead, it leverages certain one-dimensional projections of the likelihood ratio function, which effectively enhance LR-QR’s performance in high-dimensional tasks compared to other baselines.

While this paper primarily focuses on marginal test-time coverage guarantees, we acknowledge that in many practical scenarios, marginal guarantees alone may not suffice. An interesting direction for future work is to explore whether the techniques and intuitions developed here can be extended to provide stronger conditional guarantees at test time in the presence of covariate shifts. In particular, is it possible to achieve group-conditional coverage at test time (e.g., see \cite{bastani2022practical,jung2023batch,gibbs2025conformal}) without directly estimating the likelihood ratio function?

Additionally, several open questions remain regarding the regularization technique in LR-QR. Specifically, what alternative forms of regularization, beyond the mean squared error used in this work, could be employed to further improve test-time coverage? Which type of regularization is optimal in the sense that it yields the most precise test-time coverage? Furthermore, what is the most effective strategy for tuning the regularization strength? In particular, can these ideas be extended to design a hyperparameter-free algorithm? Finally, the data-adaptive regularization introduced in this work may have applications beyond conformal prediction, serving as a general technique to improve robustness to covariate shifts in other machine learning problems.

\section{Acknowledgments}
ED and SJ were supported  
by NSF, ARO, ONR, AFOSR, and the Sloan Foundation. The work of HH, SK, and GP was supported by the NSF Institute for CORE Emerging Methods in Data Science (EnCORE). 








{\small
\setlength{\bibsep}{0.2pt plus 0.3ex}
\bibliographystyle{plainnat-abbrev}
\bibliography{ref}

\begin{thebibliography}{55}
\providecommand{\natexlab}[1]{#1}
\providecommand{\url}[1]{\texttt{#1}}
\expandafter\ifx\csname urlstyle\endcsname\relax
  \providecommand{\doi}[1]{doi: #1}\else
  \providecommand{\doi}{doi: \begingroup \urlstyle{rm}\Url}\fi

\bibitem[Ai and Ren(2024)]{ai2024not}
J.~Ai and Z.~Ren.
\newblock Not all distributional shifts are equal: Fine-grained robust conformal inference.
\newblock \emph{arXiv preprint arXiv:2402.13042}, 2024.

\bibitem[Angelopoulos and Bates(2021)]{angelopoulos2021gentle}
A.~N. Angelopoulos and S.~Bates.
\newblock A gentle introduction to conformal prediction and distribution-free uncertainty quantification.
\newblock \emph{arXiv preprint arXiv:2107.07511}, 2021.

\bibitem[Bai et~al.(2022)Bai, Mei, Wang, Zhou, and Xiong]{bai2022efficient}
Y.~Bai, S.~Mei, H.~Wang, Y.~Zhou, and C.~Xiong.
\newblock Efficient and differentiable conformal prediction with general function classes.
\newblock \emph{arXiv preprint arXiv:2202.11091}, 2022.

\bibitem[Bastani et~al.(2022)Bastani, Gupta, Jung, Noarov, Ramalingam, and Roth]{bastani2022practical}
O.~Bastani, V.~Gupta, C.~Jung, G.~Noarov, R.~Ramalingam, and A.~Roth.
\newblock Practical adversarial multivalid conformal prediction.
\newblock \emph{Advances in Neural Information Processing Systems}, 35:\penalty0 29362--29373, 2022.

\bibitem[Ben-David et~al.(2010)Ben-David, Blitzer, Crammer, Kulesza, Pereira, and Vaughan]{ben2010theory}
S.~Ben-David, J.~Blitzer, K.~Crammer, A.~Kulesza, F.~Pereira, and J.~W. Vaughan.
\newblock A theory of learning from different domains.
\newblock \emph{Machine learning}, 79:\penalty0 151--175, 2010.

\bibitem[Bhattacharyya and Barber(2024)]{bhattacharyya2024group}
A.~Bhattacharyya and R.~F. Barber.
\newblock Group-weighted conformal prediction.
\newblock \emph{arXiv preprint arXiv:2401.17452}, 2024.

\bibitem[Cauchois et~al.(2024)Cauchois, Gupta, Ali, and Duchi]{cauchois2024robust}
M.~Cauchois, S.~Gupta, A.~Ali, and J.~C. Duchi.
\newblock Robust validation: Confident predictions even when distributions shift.
\newblock \emph{Journal of the American Statistical Association}, pages 1--66, 2024.

\bibitem[Foygel~Barber et~al.(2021)Foygel~Barber, Candes, Ramdas, and Tibshirani]{foygel2021limits}
R.~Foygel~Barber, E.~J. Candes, A.~Ramdas, and R.~J. Tibshirani.
\newblock The limits of distribution-free conditional predictive inference.
\newblock \emph{Information and Inference: A Journal of the IMA}, 10\penalty0 (2):\penalty0 455--482, 2021.

\bibitem[Friedman(2003)]{friedman2004multivariate}
J.~H. Friedman.
\newblock On multivariate goodness-of-fit and two-sample testing.
\newblock \emph{Statistical Problems in Particle Physics, Astrophysics, and Cosmology}, 1:\penalty0 311--313, 2003.

\bibitem[Ganin and Lempitsky(2015)]{Ganin2015}
Y.~Ganin and V.~Lempitsky.
\newblock {Unsupervised domain adaptation by backpropagation}.
\newblock In \emph{32nd International Conference on Machine Learning, ICML 2015}, volume~2, pages 1180--1189. PMLR, 2015.
\newblock ISBN 9781510810587.

\bibitem[Gibbs et~al.(2025)Gibbs, Cherian, and Cand{\`e}s]{gibbs2025conformal}
I.~Gibbs, J.~J. Cherian, and E.~J. Cand{\`e}s.
\newblock Conformal prediction with conditional guarantees.
\newblock \emph{Journal of the Royal Statistical Society Series B: Statistical Methodology}, page qkaf008, 2025.

\bibitem[Gui et~al.(2024)Gui, Barber, and Ma]{gui2024distributionally}
Y.~Gui, R.~F. Barber, and C.~Ma.
\newblock Distributionally robust risk evaluation with an isotonic constraint.
\newblock \emph{arXiv preprint arXiv:2407.06867}, 2024.

\bibitem[He et~al.(2016)He, Zhang, Ren, and Sun]{he2016deep}
K.~He, X.~Zhang, S.~Ren, and J.~Sun.
\newblock Deep residual learning for image recognition.
\newblock In \emph{Proceedings of the IEEE conference on computer vision and pattern recognition}, pages 770--778, 2016.

\bibitem[Hoeffding(1963)]{hoeffding1963probability}
W.~Hoeffding.
\newblock Probability inequalities for sums of bounded random variables.
\newblock \emph{Journal of the American Statistical Association}, 58\penalty0 (301):\penalty0 13--30, 1963.

\bibitem[Jung et~al.(2023)Jung, Noarov, Ramalingam, and Roth]{jung2023batch}
C.~Jung, G.~Noarov, R.~Ramalingam, and A.~Roth.
\newblock Batch multivalid conformal prediction.
\newblock In \emph{International Conference on Learning Representations (ICLR)}, 2023.

\bibitem[Kasa et~al.(2024)Kasa, Zhang, Yang, and Taylor]{kasa2024adapting}
K.~Kasa, Z.~Zhang, H.~Yang, and G.~W. Taylor.
\newblock Adapting conformal prediction to distribution shifts without labels.
\newblock \emph{arXiv preprint arXiv:2406.01416}, 2024.

\bibitem[Kiyani et~al.(2024{\natexlab{a}})Kiyani, Pappas, and Hassani]{kiyani2024conformal}
S.~Kiyani, G.~Pappas, and H.~Hassani.
\newblock Conformal prediction with learned features.
\newblock \emph{arXiv preprint arXiv:2404.17487}, 2024{\natexlab{a}}.

\bibitem[Kiyani et~al.(2024{\natexlab{b}})Kiyani, Pappas, and Hassani]{kiyani2024length}
S.~Kiyani, G.~Pappas, and H.~Hassani.
\newblock Length optimization in conformal prediction.
\newblock \emph{arXiv preprint arXiv:2406.18814}, 2024{\natexlab{b}}.

\bibitem[Koh et~al.(2021)Koh, Sagawa, Marklund, Xie, Zhang, Balsubramani, Hu, Yasunaga, Phillips, Gao, et~al.]{koh2021wilds}
P.~W. Koh, S.~Sagawa, H.~Marklund, S.~M. Xie, M.~Zhang, A.~Balsubramani, W.~Hu, M.~Yasunaga, R.~L. Phillips, I.~Gao, et~al.
\newblock Wilds: A benchmark of in-the-wild distribution shifts.
\newblock In \emph{International conference on machine learning}, pages 5637--5664. PMLR, 2021.

\bibitem[Kumar et~al.(2023)Kumar, Lu, Gupta, Palepu, Bellamy, Raskar, and Beam]{kumar2023conformal}
B.~Kumar, C.~Lu, G.~Gupta, A.~Palepu, D.~Bellamy, R.~Raskar, and A.~Beam.
\newblock Conformal prediction with large language models for multi-choice question answering.
\newblock \emph{arXiv preprint arXiv:2305.18404}, 2023.

\bibitem[Lei et~al.(2013)Lei, Robins, and Wasserman]{lei2013distribution}
J.~Lei, J.~Robins, and L.~Wasserman.
\newblock Distribution-free prediction sets.
\newblock \emph{Journal of the American Statistical Association}, 108\penalty0 (501):\penalty0 278--287, 2013.

\bibitem[Lei et~al.(2018)Lei, G'Sell, Rinaldo, Tibshirani, and Wasserman]{lei2018distribution}
J.~Lei, M.~G'Sell, A.~Rinaldo, R.~Tibshirani, and L.~Wasserman.
\newblock Distribution-free predictive inference for regression.
\newblock \emph{Journal of the American Statistical Association}, 113\penalty0 (523):\penalty0 1094--1111, 2018.

\bibitem[Lei and Cand{\`{e}}s(2021)]{Lei2021}
L.~Lei and E.~J. Cand{\`{e}}s.
\newblock {Conformal inference of counterfactuals and individual treatment effects}.
\newblock \emph{Journal of the Royal Statistical Society. Series B: Statistical Methodology}, 83\penalty0 (5):\penalty0 911--938, 2021.
\newblock ISSN 14679868.
\newblock \doi{10.1111/rssb.12445}.
\newblock URL \url{http://arxiv.org/abs/2006.06138}.

\bibitem[Noorani et~al.(2024)Noorani, Romero, Fabbro, Hassani, and Pappas]{noorani2024conformal}
S.~Noorani, O.~Romero, N.~D. Fabbro, H.~Hassani, and G.~J. Pappas.
\newblock Conformal risk minimization with variance reduction.
\newblock \emph{arXiv preprint arXiv:2411.01696}, 2024.

\bibitem[Papadopoulos et~al.(2002)Papadopoulos, Proedrou, Vovk, and Gammerman]{papadopoulos2002inductive}
H.~Papadopoulos, K.~Proedrou, V.~Vovk, and A.~Gammerman.
\newblock Inductive confidence machines for regression.
\newblock In \emph{European Conference on Machine Learning}, pages 345--356. Springer, 2002.

\bibitem[Park et~al.(2022{\natexlab{a}})Park, Dobriban, Lee, and Bastani]{park2021pac}
S.~Park, E.~Dobriban, I.~Lee, and O.~Bastani.
\newblock {PAC} prediction sets under covariate shift.
\newblock In \emph{International Conference on Learning Representations}, 2022{\natexlab{a}}.

\bibitem[Park et~al.(2022{\natexlab{b}})Park, Dobriban, Lee, and Bastani]{park2022pac}
S.~Park, E.~Dobriban, I.~Lee, and O.~Bastani.
\newblock {PAC} prediction sets for meta-learning.
\newblock In \emph{{A}dvances in {N}eural {I}nformation {P}rocessing {S}ystems}, 2022{\natexlab{b}}.

\bibitem[Qin et~al.(2024)Qin, Liu, Li, and Huang]{qin2024distribution}
J.~Qin, Y.~Liu, M.~Li, and C.-Y. Huang.
\newblock Distribution-free prediction intervals under covariate shift, with an application to causal inference.
\newblock \emph{Journal of the American Statistical Association}, 0\penalty0 (0):\penalty0 1--26, 2024.
\newblock \doi{10.1080/01621459.2024.2356886}.
\newblock URL \url{https://doi.org/10.1080/01621459.2024.2356886}.

\bibitem[Qiu et~al.(2023)Qiu, Dobriban, and Tchetgen~Tchetgen]{qiu2023prediction}
H.~Qiu, E.~Dobriban, and E.~Tchetgen~Tchetgen.
\newblock Prediction sets adaptive to unknown covariate shift.
\newblock \emph{Journal of the Royal Statistical Society Series B: Statistical Methodology}, 85\penalty0 (5):\penalty0 1680--1705, 2023.

\bibitem[Qui{\~n}onero-Candela et~al.(2009)Qui{\~n}onero-Candela, Sugiyama, Lawrence, and Schwaighofer]{quinonero2009dataset}
J.~Qui{\~n}onero-Candela, M.~Sugiyama, N.~D. Lawrence, and A.~Schwaighofer.
\newblock \emph{Dataset shift in machine learning}.
\newblock Mit Press, 2009.

\bibitem[Redmond(2002)]{UCIrepo}
M.~Redmond.
\newblock {Communities and Crime}.
\newblock UCI Machine Learning Repository, 2002.
\newblock {DOI}: https://doi.org/10.24432/C53W3X.

\bibitem[Romano et~al.(2019)Romano, Patterson, and Candes]{romano2019conformalized}
Y.~Romano, E.~Patterson, and E.~Candes.
\newblock Conformalized quantile regression.
\newblock \emph{Advances in neural information processing systems}, 32, 2019.

\bibitem[Sadinle et~al.(2019)Sadinle, Lei, and Wasserman]{Sadinle2019}
M.~Sadinle, J.~Lei, and L.~Wasserman.
\newblock {Least Ambiguous Set-Valued Classifiers With Bounded Error Levels}.
\newblock \emph{Journal of the American Statistical Association}, 114\penalty0 (525):\penalty0 223--234, 2019.
\newblock ISSN 1537274X.
\newblock \doi{10.1080/01621459.2017.1395341}.

\bibitem[Saunders et~al.(1999)Saunders, Gammerman, and Vovk]{saunders1999transduction}
C.~Saunders, A.~Gammerman, and V.~Vovk.
\newblock Transduction with confidence and credibility.
\newblock In \emph{IJCAI}, 1999.

\bibitem[Scheffe and Tukey(1945)]{scheffe1945non}
H.~Scheffe and J.~W. Tukey.
\newblock Non-parametric estimation. i. validation of order statistics.
\newblock \emph{The Annals of Mathematical Statistics}, 16\penalty0 (2):\penalty0 187--192, 1945.

\bibitem[Sesia and Romano(2021)]{sesia2021conformal}
M.~Sesia and Y.~Romano.
\newblock Conformal prediction using conditional histograms.
\newblock \emph{Advances in Neural Information Processing Systems}, 34:\penalty0 6304--6315, 2021.

\bibitem[Shalev-Shwartz and Ben-David(2014)]{shalev2014understanding}
S.~Shalev-Shwartz and S.~Ben-David.
\newblock \emph{Understanding machine learning: From theory to algorithms}.
\newblock Cambridge university press, 2014.

\bibitem[Shalev-Shwartz et~al.(2010)Shalev-Shwartz, Shamir, Srebro, and Sridharan]{shalev2010learnability}
S.~Shalev-Shwartz, O.~Shamir, N.~Srebro, and K.~Sridharan.
\newblock Learnability, stability and uniform convergence.
\newblock \emph{The Journal of Machine Learning Research}, 11:\penalty0 2635--2670, 2010.

\bibitem[Shimodaira(2000)]{shimodaira2000improving}
H.~Shimodaira.
\newblock Improving predictive inference under covariate shift by weighting the log-likelihood function.
\newblock \emph{Journal of statistical planning and inference}, 90\penalty0 (2):\penalty0 227--244, 2000.

\bibitem[Si et~al.(2024)Si, Park, Lee, Dobriban, and Bastani]{si2023pac}
W.~Si, S.~Park, I.~Lee, E.~Dobriban, and O.~Bastani.
\newblock {PAC} prediction sets under label shift.
\newblock \emph{International Conference on Learning Representations}, 2024.

\bibitem[Stutz et~al.(2022)Stutz, Dvijotham, Cemgil, and Doucet]{StutzICLR2022}
D.~Stutz, K.~D. Dvijotham, A.~T. Cemgil, and A.~Doucet.
\newblock Learning optimal conformal classifiers.
\newblock In \emph{International Conference on Learning Representations}, 2022.
\newblock URL \url{https://openreview.net/forum?id=t8O-4LKFVx}.

\bibitem[Sugiyama and Kawanabe(2012)]{Sugiyama2012}
M.~Sugiyama and M.~Kawanabe.
\newblock \emph{{Machine learning in non-stationary environments : introduction to covariate shift adaptation}}.
\newblock MIT Press, 2012.
\newblock ISBN 9780262017091.

\bibitem[Sypetkowski et~al.(2023)Sypetkowski, Rezanejad, Saberian, Kraus, Urbanik, Taylor, Mabey, Victors, Yosinski, Sereshkeh, et~al.]{sypetkowski2023rxrx1}
M.~Sypetkowski, M.~Rezanejad, S.~Saberian, O.~Kraus, J.~Urbanik, J.~Taylor, B.~Mabey, M.~Victors, J.~Yosinski, A.~R. Sereshkeh, et~al.
\newblock Rxrx1: A dataset for evaluating experimental batch correction methods.
\newblock In \emph{Proceedings of the IEEE/CVF Conference on Computer Vision and Pattern Recognition}, pages 4285--4294, 2023.

\bibitem[Taori et~al.(2020)Taori, Dave, Shankar, Carlini, Recht, and Schmidt]{taori2020measuring}
R.~Taori, A.~Dave, V.~Shankar, N.~Carlini, B.~Recht, and L.~Schmidt.
\newblock Measuring robustness to natural distribution shifts in image classification.
\newblock \emph{Advances in Neural Information Processing Systems}, 33:\penalty0 18583--18599, 2020.

\bibitem[Tibshirani et~al.(2019)Tibshirani, Foygel~Barber, Cand{\`e}s, and Ramdas]{tibshirani2019conformal}
R.~J. Tibshirani, R.~Foygel~Barber, E.~J. Cand{\`e}s, and A.~Ramdas.
\newblock Conformal prediction under covariate shift.
\newblock \emph{Advances in neural information processing systems}, 32, 2019.

\bibitem[Tukey(1947)]{tukey1947non}
J.~W. Tukey.
\newblock Non-parametric estimation ii. statistically equivalent blocks and tolerance regions--the continuous case.
\newblock \emph{The Annals of Mathematical Statistics}, pages 529--539, 1947.

\bibitem[Tukey(1948)]{tukey1948nonparametric}
J.~W. Tukey.
\newblock Nonparametric estimation, iii. statistically equivalent blocks and multivariate tolerance regions--the discontinuous case.
\newblock \emph{The Annals of Mathematical Statistics}, pages 30--39, 1948.

\bibitem[Vovk(2013)]{Vovk2013}
V.~Vovk.
\newblock {Conditional validity of inductive conformal predictors}.
\newblock In \emph{Asian conference on machine learning}, volume~25, pages 475--490. PMLR, 2013.
\newblock \doi{10.1007/s10994-013-5355-6}.

\bibitem[Vovk et~al.(2005)Vovk, Gammerman, and Shafer]{vovk2005algorithmic}
V.~Vovk, A.~Gammerman, and G.~Shafer.
\newblock \emph{Algorithmic learning in a random world}.
\newblock Springer Science \& Business Media, 2005.

\bibitem[Vovk et~al.(2022)Vovk, Gammerman, and Shafer]{vovk2022algorithmic}
V.~Vovk, A.~Gammerman, and G.~Shafer.
\newblock \emph{Algorithmic Learning in a Random World}.
\newblock Springer Nature, 2022.

\bibitem[Vovk et~al.(1999)Vovk, Gammerman, and Saunders]{vovk1999machine}
V.~Vovk, A.~Gammerman, and C.~Saunders.
\newblock Machine-learning applications of algorithmic randomness.
\newblock In \emph{International Conference on Machine Learning}, 1999.

\bibitem[Wald(1943)]{Wald1943}
A.~Wald.
\newblock {An Extension of Wilks' Method for Setting Tolerance Limits}.
\newblock \emph{The Annals of Mathematical Statistics}, 14\penalty0 (1):\penalty0 45--55, 1943.
\newblock ISSN 0003-4851.
\newblock \doi{10.1214/aoms/1177731491}.

\bibitem[Wilks(1941)]{Wilks1941}
S.~S. Wilks.
\newblock {Determination of Sample Sizes for Setting Tolerance Limits}.
\newblock \emph{The Annals of Mathematical Statistics}, 12\penalty0 (1):\penalty0 91--96, 1941.
\newblock ISSN 0003-4851.
\newblock \doi{10.1214/aoms/1177731788}.

\bibitem[Yang et~al.(2024)Yang, Kuchibhotla, and Tchetgen~Tchetgen]{yang2024doubly}
Y.~Yang, A.~K. Kuchibhotla, and E.~Tchetgen~Tchetgen.
\newblock Doubly robust calibration of prediction sets under covariate shift.
\newblock \emph{Journal of the Royal Statistical Society Series B: Statistical Methodology}, page qkae009, 2024.

\bibitem[Yu et~al.(2020)Yu, Chen, Wang, Xian, Chen, Liu, Madhavan, and Darrell]{yu2020bdd100k}
F.~Yu, H.~Chen, X.~Wang, W.~Xian, Y.~Chen, F.~Liu, V.~Madhavan, and T.~Darrell.
\newblock Bdd100k: A diverse driving dataset for heterogeneous multitask learning.
\newblock In \emph{Proceedings of the IEEE/CVF conference on computer vision and pattern recognition}, pages 2636--2645, 2020.

\end{thebibliography}
}

\newpage
\appendix


\section{Additional figures}
\begin{figure*}[ht]
\centering
\includegraphics[width=\textwidth]{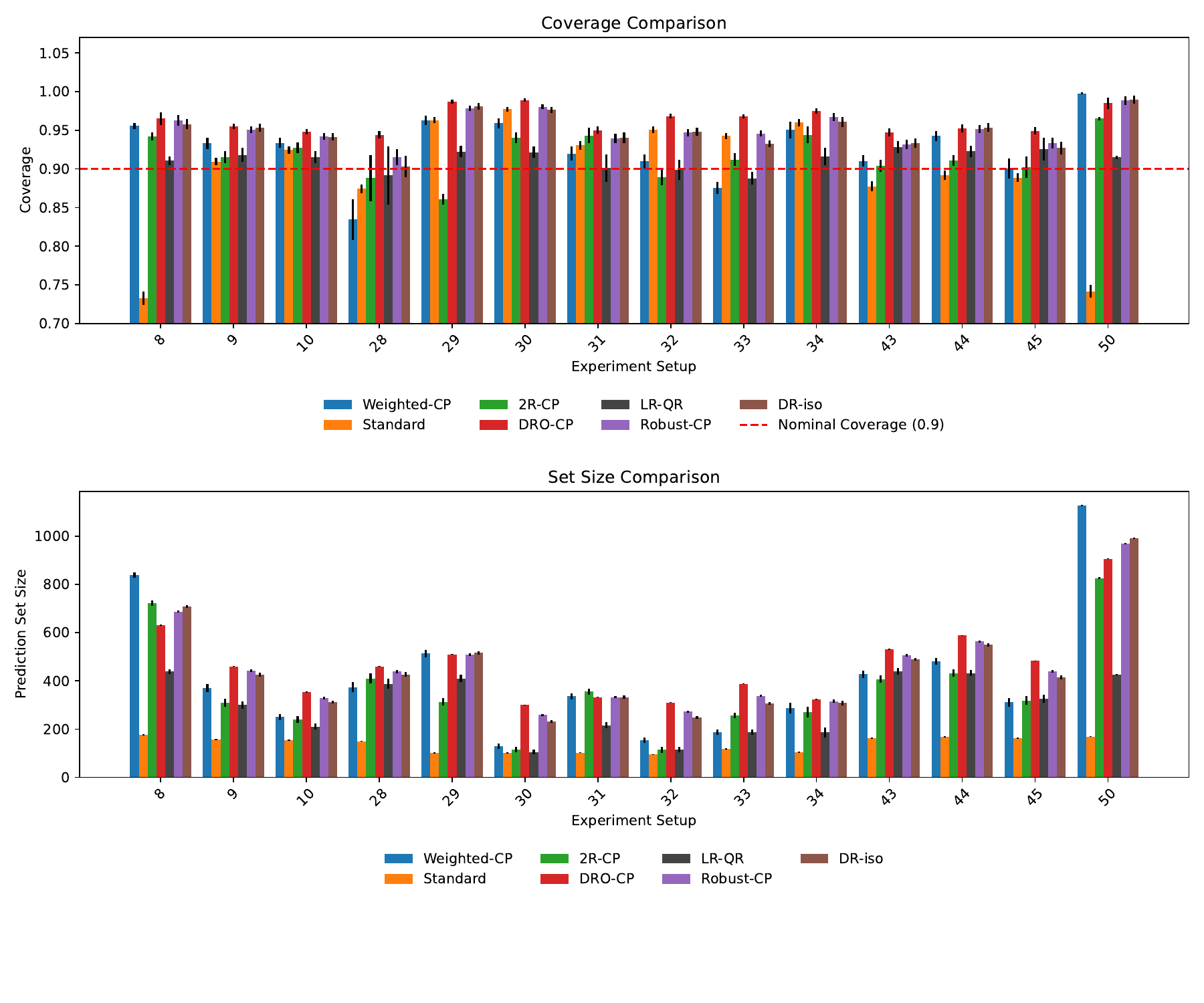}
\caption{(Above)~Coverage, (Below)~Average prediction set size.}
\label{Fig_rx_tot}
\end{figure*}

\begin{figure*}[ht]
\centering
\begin{subfigure}{0.5\textwidth}
    \centering
    \includegraphics[width=\textwidth]{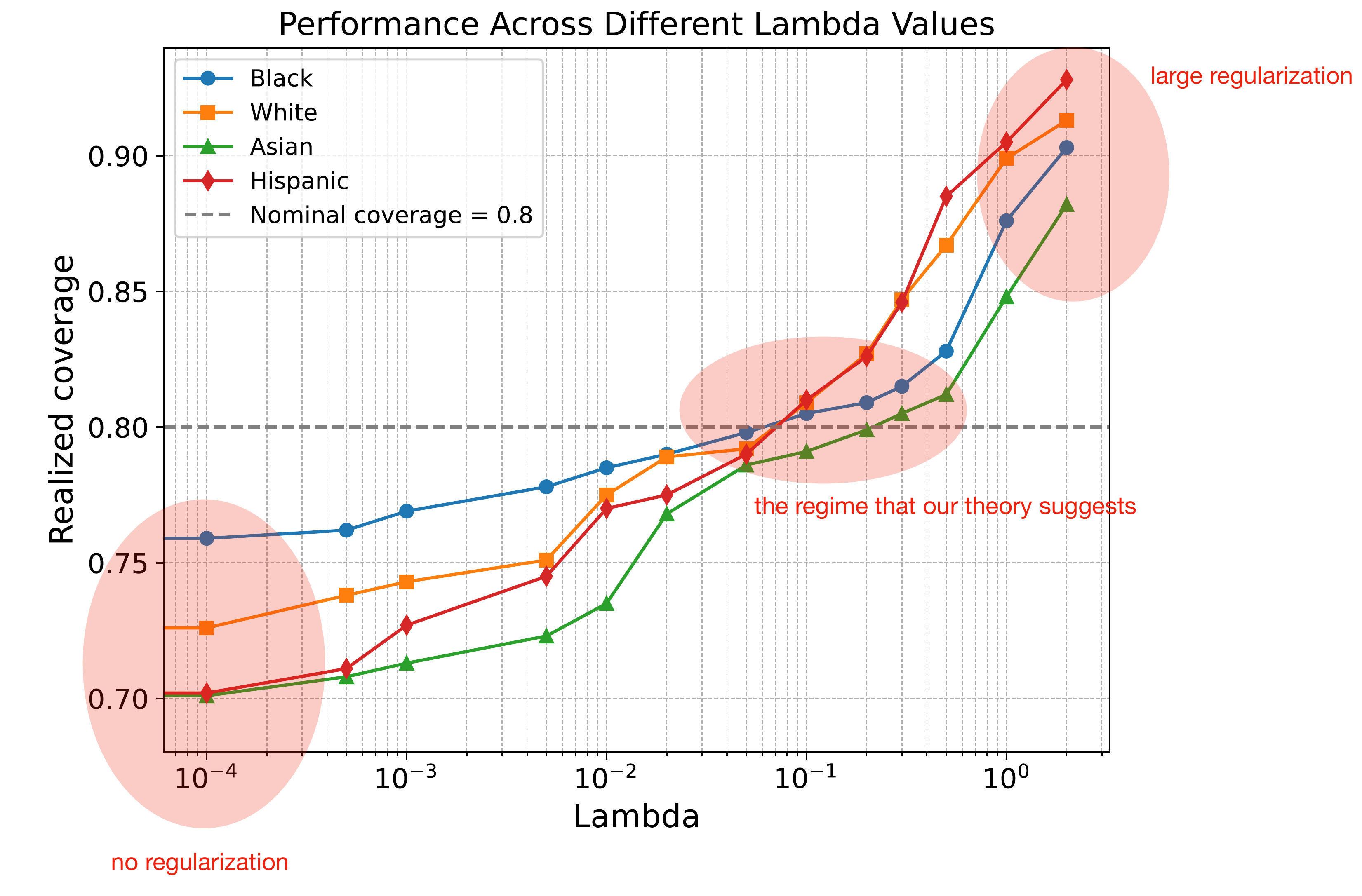}
    \caption{}
    \label{ablation-img}
\end{subfigure}
\hfill
\begin{subfigure}{0.5\textwidth}
    \centering
    \includegraphics[width=0.98\textwidth]{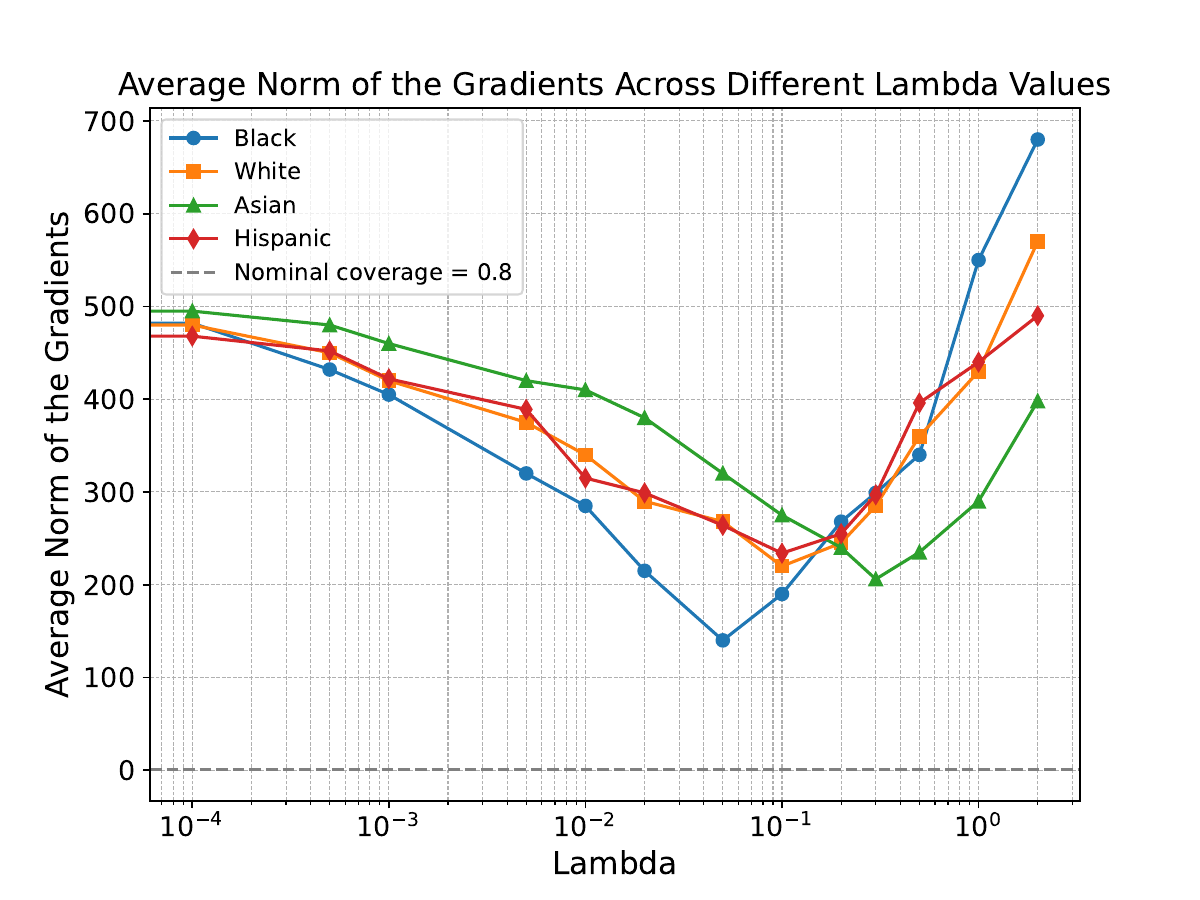} 
    \caption{}
    \label{grad-img}
\end{subfigure}
\caption{\textbf{Ablation study on the effect of $\lambda$ on LR-QR performance in the experimental setup of \Cref{reg_exp}.} In (a), the theoretically suggested regime for $\lambda$ effectively ensures valid test-time coverage. Additionally, in (b), the average norm of the gradients reaches its lowest value in the regime predicted by theory, highlighting the effectiveness of the cross-validation procedure described in \Cref{subsec:choose_lambda}.}
\label{fig:combined-plots}
\end{figure*}

\section{Additional Experiment}

\subsection{RxRx1 data - WILDS}

\begin{figure*}
\centering
\includegraphics[width=\textwidth]{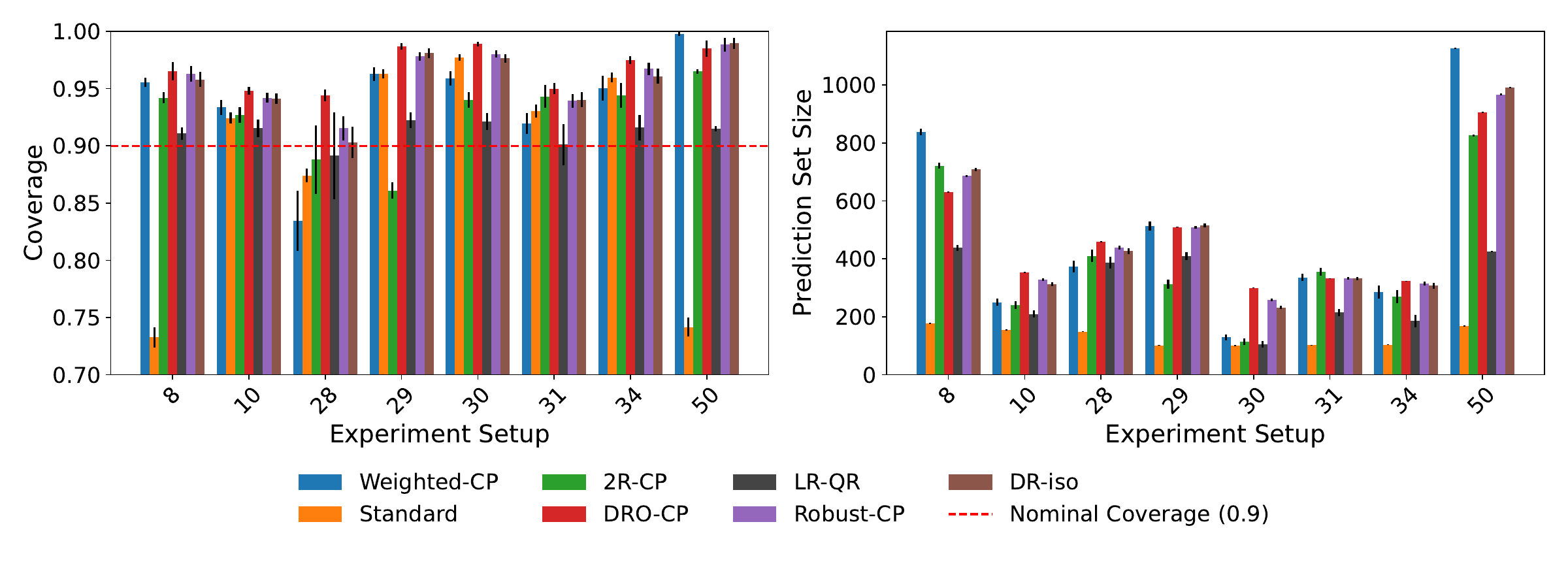}
\caption{(Left)~Coverage, (Right)~Average prediction set size on the RxRx1 dataset from the WILDS repository.}
\label{Fig_rx}
\end{figure*}

We consider the RxRx1 dataset \citep{sypetkowski2023rxrx1} from the WILDS repository \citep{koh2021wilds}, which is designed to evaluate model robustness under distribution shifts. The RxRx1 task involves classifying cell images based on 1339 laboratory genetic treatments. 
These images, captured using fluorescent microscopy, originate from 51 independent experiments. 
Variations in execution and environmental conditions lead to systematic differences across experiments, affecting the distribution of input features (e.g., lighting, cell morphology) while the relationship between inputs and labels remains unchanged.
This situation creates covariate shift where the marginal distribution of inputs shifts across domains, but the conditional distribution \( \mathbb{P}_{Y | X} \) remains the same.

We use a ResNet50 model \citep{he2016deep} 
trained by the WILDS authors on 37 of the 51 experiments. 
Using the other
experiments, we construct 14 distinct evaluations,
where each experiment is selected as the target dataset, and its data is evenly split into an unlabeled target set and a labeled test set. 
The labeled data from the other 13 experiments serves as the source dataset. 

{\bf Experimental details.}
The nonconformity score is 
$
s(x,y) = -\log f_x(y),
$
where $f_x(y)$ is the probability assigned the image-label pair $(x, y)$.
To estimate $r$, we train a logistic regression model $\hat p$ on top of the representation layer of the pretrained model to distinguish unlabeled source and target data, 
and we set $\hat r= \frac{\hat{p}}{1-\hat{p}}$. 
We set the hypothesis class $\mathcal{H}$ to be a linear head on top of the representation layer 
of the pretrained model. 
Experimental results are averaged over 50 random splits.

{\bf Results.}
\Cref{Fig_rx} presents the coverage and average prediction set size for all methods. 
To enhance visual interpretability, we display results for eight randomly selected settings out of the 14, with the full plot provided in \Cref{Fig_rx_tot}. 
The x-axis shows the indices of the test condition.
LR-QR adheres more closely to the nominal coverage value of $0.9$ compared to other methods.  


Notably, split conformal prediction, 
which assumes exchangeability between calibration and test data, shows under- and overcoverage due to the covariate shift.
The coverage of weighted CP and 2R-CP 
is also far from the nominal level, showing that
directly estimating the likelihood ratio and conditional quantile is insufficient to correct the coverage violations
in the case of high-dimensional image data. 
Further, the superior coverage of LR-QR is not due to inflated prediction sets.

\section{Ablation studies}\label{ablation}
Here we provide an ablation study for $\lambda$, the regularization strength that appears in the LR-QR objective.
In the same regression setup as \Cref{reg_exp}, instead of selecting $\lambda$ via cross-validation, here we sweep the value of $\lambda$ from $0$ to $2$, and we plot the coverage of the LR-QR algorithm on the test data. Here, note that the split ratios between train, calibration, and test (both labeled and unlabeled data) are fixed and similar to the setup in \Cref{reg_exp}. We report the averaged plots over 100 independent splits.

\Cref{ablation-img} displays the effect of different regimes of $\lambda$.
At one extreme, when $\lambda$ is close to zero, the LR-QR algorithm reduces to ordinary quantile regression. In this regime, the LR-QR algorithm behaves similarly to the algorithm from \cite{gibbs2025conformal}, without the test covariate imputation. In other words, when we set $\lambda=0$, we try to provide coverage with respect to all the covariate shifts in the linear function class that we optimize over. As we can see in \Cref{ablation-img}, this can lead to overfitting and undercoverage of the test labels.
As we increase $\lambda$, as a direct effect of the regularization, the coverage gap decreases. This is primarily due to the fact that larger $\lambda$ restricts the space of quantile regression optimization
in such a way that it does not hurt the test time coverage,
since the regularization is designed to shrink the optimization space towards the true likelihood-ratio.
Thus, the regularization improves the generalization of the selected threshold, as the effective complexity of the function class is getting smaller.
That being said, this phenomenon is only applicable if $\lambda$ lies within a certain range;
once $\lambda$ grows too large,
due to the data-dependent nature of our regularization,
the generalization error of the regularization term itself becomes non-negligible and hinders the precise test-time coverage of the LR-QR threshold.
As is highlighted in \Cref{ablation-img},
our theoretical results suggest an optimal regime for $\lambda$ which can best exploit the geometric properties of the LR-QR threshold.

Additionally, \Cref{grad-img} demonstrates the effectiveness of the cross-validation technique described in \Cref{subsec:choose_lambda}. We sweep the value of $\lambda$ from $0$ to $2$ and plot the average norm of the gradient on the holdout sets for the cross-validation procedure explained in \Cref{subsec:choose_lambda}. As our theory suggests, it is now evident that the stationary conditions of LR-QR are closely tied to the valid test-time coverage of our method. For all values of $\lambda$, during training, we fit the LR-QR objective to the data, ensuring that the average norm of the gradients is zero. However, when evaluating the LR-QR objective on the holdout set, the average norm of the gradients is no longer zero due to generalization errors. Selecting $\lambda$ correctly minimizes this generalization error, thereby providing more precise test-time coverage.

\section{Further related work}
\label{app:related-work}

The basic concept of prediction sets dates back to foundational works such as \cite{Wilks1941}, \cite{Wald1943}, \cite{scheffe1945non}, and \cite{tukey1947non,tukey1948nonparametric}.
The early ideas of conformal prediction were developed in \cite{saunders1999transduction,vovk1999machine}.
With the rise of machine learning, conformal prediction has emerged as a widely used framework for constructing prediction sets \citep[e.g.,][]{papadopoulos2002inductive,vovk2005algorithmic,lei2018distribution,angelopoulos2021gentle}. Since then, efforts have been emerged to improve prediction set size efficiency \citep[e.g.,][]{Sadinle2019, StutzICLR2022, bai2022efficient, kiyani2024length, noorani2024conformal} and conditional coverage guarantees \citep[e.g.,][]{foygel2021limits, sesia2021conformal, gibbs2025conformal, romano2019conformalized, kiyani2024conformal, jung2023batch}.


Numerous works have addressed conformal prediction under various types of distribution shift 
\citep{tibshirani2019conformal,park2021pac,park2022pac,qiu2023prediction,si2023pac}. For example, \cite{tibshirani2019conformal} and \cite{Lei2021} investigated conformal prediction under covariate shift, assuming the likelihood ratio between source and target covariates is known or can be precisely estimated from data. 
\cite{park2021pac} developed prediction sets with a calibration-set conditional (PAC) property under covariate shift. 
\cite{qiu2023prediction,yang2024doubly} developed prediction sets with asymptotic coverage that are doubly robust in the sense that their coverage error is bounded by the product of the estimation errors of the quantile function of the score and the likelihood ratio.
\cite{cauchois2024robust} construct prediction sets based on a distributionally robust optimization approach.
\cite{gui2024distributionally} develop methods based on an isotonic regression estimate of the likelihood ratio.
\cite{qin2024distribution} combine a parametric working model with a resampling approach to construct prediction sets under covariate shift.
\cite{bhattacharyya2024group} analyze weighted conformal prediction in the special case of covariate shifts defined by a finite number of groups.
\cite{ai2024not} reweight samples to adapt to covariate shift, while simultaneously using distributionally robust optimization to protect against worst-case joint distribution shifts.
\cite{kasa2024adapting} construct prediction sets by using unlabeled test data to modify the score function used for conformal prediction.

\section{Notation and conventions}

Constants are allowed to depend on dimension only through properties of the population and sample covariance matrices of the features, and the amount of linear independence of the features; see the quantities $\lambda_{\tn{min}}(\Sigma)$, $\lambda_{\tn{max}}$, $c_{\tn{min}}$, $c_{\tn{max}}$, and $c_{\tn{indep}}$ defined in \Cref{sec:conds}.
In the Landau notation ($o$, $O$, $\Theta$), we hide constants.
We say that a sequence of events holds with high probability if the probability of the events tends to unity.
We define $\mathcal{S}_1$ as the features of the labeled calibration dataset.
All functions that we minimize can readily be verified to be continuous, and thus attain a minimum over the compact domains over which we minimize them; thus all our minimizers will be well-defined.  We may not mention this further.
We denote by $\mathbf{1}[A]$ the indicator of an event $A$.
Recall that $\cH$ denotes the linear hypothesis class 
$\cH = \{ \langle \gamma, \Phi \rangle : \gamma \in \R^d \}$.
This defines a one-to-one correspondence between $\R^d $ and $\cH$.
This enables us to view functions defined on $\R^d$ equivalently as defined on $\cH$. In our analysis, we will use such steps without further discussion.
Unless stated otherwise, $\cH$ is 
equipped with the norm $\|h\| := \|\gamma\|_2$ for $h = \langle \gamma, \Phi\rangle$.
Given a differentiable function $\varphi : \cH \to \R$,
its directional derivative at $f = \langle \gamma, \Phi\rangle \in \cH$
in the direction defined by the function $g \in \cH$ 
is defined as
$\frac{d}{d\ep}\big|_{\ep=0} \varphi(f + \ep g)$.
Note that if we write
$g = \langle \tilde \gamma, \Phi\rangle$
for some $\tilde \gamma \in \R^d$,
then the directional derivative of $\varphi$
at $f$ equals
$\langle \tilde \gamma, \nabla_{\gamma} \varphi(\gamma) \rangle$,
where $\nabla_{\gamma} \varphi(\gamma)$ denotes
the gradient of $\varphi : \R^d \to \R$
evaluated at $\gamma \in \R^d$.
When it is clear from context, we drop the subscript $\lambda$ from the risks $L_{\lambda}$ and $\hat L_{\lambda}$.

\section{Conditions}\label{sec:conds}




\begin{condition}\label{cond:c-phi}
    Suppose $C_{\Phi} = \sup_{x\in \xx} \|\Phi(x)\|_2$ is finite.
\end{condition}

\begin{condition}\label{cond:pop-cov-matr}
    For the population covariance matrix $\Sigma = \E_1[ \Phi \Phi^\top  ]$, we have $\lambda_{\tn{min}}(\Sigma) > 0$ and $\lambda_{\tn{max}}(\Sigma)$ is of constant order, not depending on the sample size, or any other problem parameter.
\end{condition}

\begin{condition}\label{cond:sample-cov-matr}
    For the sample covariance matrix
$    \hat \Sigma = \frac{1}{n_3} \sum_{k=1}^{n_3} \Phi(x_k) \Phi(x_k)^\top $,
we have both $\lambda_{\tn{min}}(\hat \Sigma) \ge c_{\tn{min}} > 0$ and $\lambda_{\tn{max}}(\hat \Sigma) \le c_{\tn{max}}$ of constant order with probability $1 - o(n_3^{-1})$. 
\end{condition}

\begin{condition}\label{cond:c1-c2-bds}
    Defining $C_1$ as in \eqref{C12} in \Cref{sec:lip_proc}, assume there exists an upper bound $C_{1,\tn{upper}}$ on $\E[C_1]$ of constant order.
\end{condition}

\begin{condition}\label{cond:condl-dens-bdd}
    The conditional density $f_{S|X=x}$ exists for all $x\in \xx$, and $C_f = \sup_{x\in \xx} \| f_{S|X=x}(s) \|_{\infty}$ is a finite constant. 
\end{condition}

The following can be interpreted as an independence assumption on the basis functions.
\begin{condition}\label{cond:basis-indep}
    Suppose $
    \inf_{v\in S^{d-1}} \E_1[ |\langle v, \Phi \rangle| ] \ge c_{\tn{indep}} > 0$ for some constant $c_{\tn{indep}}$.
\end{condition}

\begin{condition}\label{cond:c-align}
Suppose
$   \frac{\E_1[ rh_0^* ]}{\E_1[ |h_0^*|^2 ]^{1/2}} \ge c_{\tn{align}} > 0$
for some minimizer $h_0^*$ of the objective in \Cref{opt:elim-beta-unconstr-CP-obj} with regularization $\lambda = 0$.
\end{condition}

\begin{condition}\label{cond:r-second-mom}
    Suppose $\E_1[r^2]$ is finite.
\end{condition}

\begin{condition}
    \label{cond:1-in-H}
    The constant function $h : \xx \to \R$ given by $h(x) = 1$ for all $x\in \xx$ is in $\cH$.
\end{condition}

The following ensures that the zero function $0\in \cH$ is not a minimizer of the objective in \Cref{opt:unconstr-CP-obj}.
\begin{condition}\label{cond:zero-subopt}
    For each $\lambda \ge 0$, there exists $h \in \cH$ and $\beta \in \R$ such that 
\begin{align*}
    \E_1[ \ell_{\alpha}(h,S) ] + \lambda \E_1[ (\beta h - r)^2 ] < \E_1[ \ell_{\alpha}(0,S) ] + \lambda \E_1[ r^2 ].
\end{align*}
\end{condition}

\section{Constants}\label{sec:consts}

The following are the constants that appear in \Cref{thm:generalize}:
\begin{align*}
    \rho_1 := 2 \beta_{\tn{max}}^2 BC_{\Phi}^2 +  2 \beta_{\tn{max}} C_{\Phi}, \quad
    \mu_1 := 2\beta_{\tn{min}}^2 c_{\tn{min}}, \quad
    \rho_2 := (1-\alpha) C_{\Phi}, \\
    \widetilde C_1 := \frac{4\rho_1^2}{\mu_1}, \quad 
    \widehat C_2 := \frac{4\rho_2^2}{2 \beta_{\tn{min}}^2 c_{\tn{min}}}, \quad
    A_1 := \sqrt{\frac{64 \widetilde C_1 a_1}{\delta}}, \quad 
    A_2 := \sqrt{\frac{128 \widehat C_2 a_2}{\delta}}.
\end{align*}
Further,
\begin{align*}
    A_3 &:= (1-\alpha) (BC_{\Phi} + 1) \sqrt{\frac{1}{2} \log\frac{8}{\delta}}, \qquad
    A_4 :=     \sqrt{2}
    (\beta_{\tn{max}} BC_{\Phi}) \sqrt{\frac{1}{2} \log\frac{16}{\delta}}
    \max\left\{ \beta_{\tn{max}} BC_{\Phi}, 4  \right\}, \\
    A_5 &:= A_1 + A_4, \quad a_1 := 2C_{\Phi} (C_{2,\tn{upper}} + C_{2,\tn{max}}) (1 + \beta_{\tn{max}} BC_{\Phi}), \quad
    a_2 := (1-\alpha) C_{\Phi} (C_{1,\tn{upper}} + C_{1,\tn{max}}).
\end{align*}
The following are the constants that appear in \Cref{thm:cov-lower-bd}:
\begin{align*}
    A_6 := 2 \beta_{\tn{max}}^2 \sqrt{4B^2 \lambda_{\tn{max}}(\Sigma)} , \quad A_7 := \sqrt{4B^2 \beta_{\tn{max}}^2 \lambda_{\tn{max}}(\Sigma)}, \quad A_8 :=   \sqrt{2B^2 C_f \lambda_{\tn{max}}(\Sigma)} , \quad A_9 := A_6 + A_7,
\end{align*}
and 
\begin{align*}
    A_{10} &:= A_9 A_5^{1/2}, \quad
    A_{11} := \max\{ A_9 A_3^{1/2},  A_8 A_5^{1/2} \}, \\ 
    A_{12} &:= A_9 A_2^{1/2}, \quad
    A_{13} := A_8 A_3^{1/2}, \quad
    A_{14} := A_8 A_2^{1/2}.
\end{align*}


\section{Generalization bound for regularized loss}

The following is a generalization of \cite[Corollary 13.6]{shalev2014understanding}.

\begin{lemma}[Generalization bound for regularized loss; extension of \cite{shalev2014understanding}]\label{lem:generalize}
    Fix a compact and convex hypothesis class $\tilde \cH$ equipped with a norm $\|\cdot\|_{\tilde \cH}$, 
    a compact interval $\mathcal{I} \subseteq \R$, and a sample space $\mz$. 
    Consider the objective function $f : \tilde \cH \times \mathcal{I} \times \mz \to \R$ given by $(h, \beta, z) \mapsto f(h, \beta, z) := \mathcal{J}(h,\beta,z) + \rr(h, \beta)$, where $\rr : \tilde \cH \times \mathcal{I} \to \R$ is a regularization function, 
    and $\mathcal{J} : \tilde \cH \times \mathcal{I} \times \mz \to \R$ can be decomposed as $\mathcal{J}(h,\beta,z) := \mathcal{J}_1(h, \beta, z_1) + \mathcal{J}_2(h, \beta, z_2)$ for two functions $\mathcal{J}_1, \mathcal{J}_2 : \tilde \cH \times \mathcal{I} \times \mz \to \R$.
    
    Given distributions $\dd_1, \dd_2$ on $\mz$, let $\mathcal{L} : \tilde \cH \times \mathcal{I} \to \R$ be given for all $h, \beta$ by 
    $$\mathcal{L}(h, \beta) = \E_{Z_1\sim \dd_1, Z_2\sim \dd_2}[ f(h,\beta,Z_1,Z_2) ]$$ 
    denote the population risk, averaging over independent datapoints $Z_1 \sim \dd_1$ and $Z_2 \sim \dd_2$. 
        Suppose 
    that for both $Z \sim \dd_1$ and $Z \sim \dd_2$, 
    $|\mathcal{J}_1(h,\beta,Z)|$ and $|\mathcal{J}_2(h,\beta,Z)|$ are almost surely bounded by a quantity not depending on $h \in \tilde \cH$ and $\beta \in \mathcal{I}$.
    
    Let $\hat{\mathcal{L}} : \tilde \cH \times \mathcal{I} \to \R$ denote the empirical risk computed over 
    $Z_{i, 1} \overset{i.i.d.}{\sim} \dd_1$, $i \in [m_1]$  and 
     $Z_{j, 2} \overset{i.i.d.}{\sim} \dd_2$, $j \in [m_2]$, 
     given by 
    \begin{align*}
        \hat{\mathcal{L}}(h, \beta) := \frac{1}{m_1} \sum_{i=1}^{m_1} \mathcal{J}_1(h,\beta,Z_{i,1}) + \frac{1}{m_2} \sum_{j=1}^{m_2} \mathcal{J}_2(h,\beta,Z_{j,2}) + \rr(h,\beta).
    \end{align*}
    Assume that for each fixed $\beta \in \mathcal{I}$ and $z\in \mathcal{Z}$,
\begin{itemize}
    \item $h\mapsto \mathcal{J}_1(h,\beta,z)$ is convex and $\rho$-Lipschitz with respect to the norm $\|\cdot\|_{\tilde \cH}$,
    \item $h\mapsto \mathcal{J}_2(h,\beta,z)$ is convex and $\rho$-Lipschitz with respect to the norm $\|\cdot\|_{\tilde \cH}$, and
    \item 
    $h\mapsto\hat{\mathcal{L}}(h, \beta)$ is $\mu$-strongly convex with respect to the norm $\|\cdot\|_{\tilde \cH}$ with probability $1 - o(m_1^{-1} + m_2^{-1})$,
\end{itemize}
where the deterministic values $\mu = \mu(\beta)$ and $\rho = \rho(\beta)$ may depend on $\beta$.

    Let $(\hat h, \hat \beta)$ denote an ERM, i.e., a minimizer of $\hat{\mathcal{L}}(h, \beta)$ over $\tilde \cH \times \mathcal{I}$.
    Let $\hat h_{\beta}$ denote a minimizer of the empirical risk in $h$ for fixed $\beta$.
  
    Suppose the stochastic process $\beta \mapsto W_{\beta}$ given by 
    $W_{\beta} = \mathcal{L}(\hat h_{\beta}, \beta) - \hat{\mathcal{L}}(\hat h_{\beta}, \beta)$ for $\beta \in \mathcal{I}$ obeys $| W_{\beta} - W_{\beta'} | \le K |\beta - \beta'|$ for all $\beta, \beta'\in \mathcal{I}$ for some random variable $K$, and suppose that the probability of  $K_{m_1,m_2} \le K_{\tn{max}}$ converges to unity as $m_1,m_2\to\infty$, for some constant $K_{\tn{max}}$.
    Suppose that there exists a constant $C > 0$ such that for all $\beta \in \mathcal{I}$,
\begin{align}\label{rhomu}
    \frac{4\rho(\beta)^2}{\mu(\beta)} \le C.
\end{align}
    Then for sufficiently large $m_1, m_2$, with probability at least $1-\delta$,
    \begin{align*}
        |\mathcal{L}(\hat h, \hat \beta) - \hat{\mathcal{L}}(\hat h, \hat \beta)| \le \sqrt{\frac{16 CK_{\tn{max}}}{\delta} (m_1^{-1} + m_2^{-1})}.
    \end{align*}
\end{lemma}

\begin{remark}\label{rmk:gen-bd-iid}
    A special case is when 
    we do not have any data from $\mathcal{D}_2$, and instead
    all $m_1$ datapoints are sampled 
 i.i.d.~from $\mathcal{D}_1$.
 In this case, defining with a slight abuse of notation $\mathcal{J}:=\mathcal{J}_1$, the statement simplifies to the analysis of the empirical risk
     \begin{align*}
        \hat{\mathcal{L}}(h, \beta) := \frac{1}{m_1} \sum_{i=1}^{m_1} \mathcal{J}(h,\beta,Z_{i,1}) + \rr(h,\beta).
    \end{align*}
 If for each fixed $\beta \in \mathcal{I}$, we have that $h\mapsto \mathcal{J}(h,\beta,z)$ is convex and $\rho$-Lipschitz with respect to the norm $\|\cdot\|_{\tilde \cH}$, and if $|\mathcal{J}(h,\beta,Z)|$ is almost surely bounded by a quantity not depending on $h\in \tilde \cH$ and $\beta \in \mathcal{I}$ for $Z\sim \dd_1 = \dd_2$, then under the remaining assumptions, we obtain the slightly stronger bound 
\begin{align*}
    |\mathcal{L}(\hat h, \hat \beta) - \hat{\mathcal{L}}(\hat h, \hat \beta)| \le \sqrt{\frac{16 CK_{\tn{max}}}{\delta m_1}}.
\end{align*}
We omit the proof, because it is exactly as below.
\end{remark}

\begin{remark}
    We relax the strong convexity assumption on the regularizer $\rr$ from \cite[Corollary 13.6]{shalev2014understanding}, 
    substituting it with the less restrictive condition of strong convexity of the empirical loss $\hat{\mathcal{L}}$. In order to use assumptions that merely hold with high probability, we impose a boundedness condition on $\mathcal{J}$.
\end{remark}

\begin{proof}
Fix $\beta$ and let $E$ denote the event that $h\mapsto\hat{\mathcal{L}}(h,\beta)$ is $\mu$-strongly convex in $h$.
By assumption, $E$ occurs with probability $1- o(m_1^{-1} + m_2^{-1})$.

We modify the proof of \cite[Corollary 13.6]{shalev2014understanding} as follows.
Let $Z_{1}' \sim \dd_1$ and $Z_{2}' \sim \dd_2$
be drawn independently from all other randomness.
For a fixed $i\in[m_1]$,
let $h \mapsto \hat{\mathcal{L}}_{i,1}(h,\beta)$ denote the empirical risk computed from the sample $(Z_{1,1}, \ldots, Z_{i-1,1}, Z_{1}', Z_{i+1,1}, \ldots, Z_{m_1,1}) \cup (Z_{1,2}, \ldots, Z_{m_2,2})$,
and let $\hat h_{\beta}^{(i)}$ denote an ERM for this sample.
Let $I$ be drawn from $[m_1]$ uniformly at random. 
The variables $J, \hat{\mathcal{L}}_{J,2}(h,\beta), \hat h_{\beta}^{(J)}$ are defined similarly but for the sample from $\dd_2$.

Note that for fixed $\beta$,
similarly to the argument in \cite[Theorem 13.2]{shalev2014understanding},
we have
\begin{align*}
     \E[ \mathcal{L}(\hat h_{\beta}, \beta)]
     &= \E_{Z_1'\sim \dd_1, Z_2'\sim \dd_2} 
     [ \mathcal{J}_1(\hat h_{\beta}, \beta, Z_1') + \mathcal{J}_2(\hat h_{\beta}, \beta, Z_2')+ \rr(\hat h_{\beta}, \beta)]\\
        & = \E_{Z_1'\sim \dd_1, Z_2'\sim \dd_2} 
     [ \mathcal{J}_1(\hat h_{\beta}^{(I)}, \beta, Z_{I,1}) + \mathcal{J}_2(\hat h_{\beta}^{(J)}, \beta, Z_{J,2}) + \rr(\hat h_{\beta}, \beta)]
\end{align*}
and
\begin{align*}
    & \E[\hat{\mathcal{L}}(\hat h_{\beta}, \beta) ]  
    =\E[ \mathcal{J}_1(\hat h_{\beta}, \beta, Z_{I,1})
    + \mathcal{J}_2(\hat h_{\beta}, \beta, Z_{J,2}) + \rr(\hat h_{\beta}, \beta)]. 
\end{align*}
Therefore
\begin{align*}
     \E[ \mathcal{L}(\hat h_{\beta}, \beta) - \hat{\mathcal{L}}(\hat h_{\beta}, \beta) ]  
    =&(\E[ \mathcal{J}_1(\hat h_{\beta}^{(I)}, \beta, Z_{I,1}) - \mathcal{J}_1(\hat h_{\beta}, \beta, Z_{I,1}) ]) \\
    + &(\E[ \mathcal{J}_2(\hat h_{\beta}^{(J)}, \beta, Z_{J,2}) - \mathcal{J}_2(\hat h_{\beta}, \beta, Z_{J,2}) ]). 
\end{align*}

Further, splitting the expectations over $E$ and its complement $E^c$, this further equals
\begin{align}\label{erdecomp}
    & (\E[ (\mathcal{J}_1(\hat h_{\beta}^{(I)}, \beta, Z_{I,1}) - \mathcal{J}_1(\hat h_{\beta}, \beta, Z_{I,1})) \mathbf{1}[E] ] + \E[ (\mathcal{J}_1(\hat h_{\beta}^{(I)}, \beta, Z_{I,1}) - \mathcal{J}_1(\hat h_{\beta}, \beta, Z_{I,1})) \mathbf{1}[E^c] ]) \\
    &+ (\E[ (\mathcal{J}_2(\hat h_{\beta}^{(J)}, \beta, Z_{J,2}) - \mathcal{J}_2(\hat h_{\beta}, \beta, Z_{J,2})) \mathbf{1}[E] ] + \E[ (\mathcal{J}_2(\hat h_{\beta}^{(J)}, \beta, Z_{J,2}) - \mathcal{J}_2(\hat h_{\beta}, \beta, Z_{J,2})) \mathbf{1}[E^c] ]).\nonumber
\end{align}

On the event $E$, $h\mapsto \hat{\mathcal{L}}(h,\beta)$ is $\mu$-strongly convex. 
Now, consider the setting of \cite[Corollary 13.6]{shalev2014understanding}. 
We claim that the arguments in their proof hold if we replace
the regularizer $h\mapsto \lambda \|h\|^2$ by $h\mapsto \mathcal{R}(h,\beta)$, as they only leverage the strong convexity of the overall empirical loss $\hat{\mathcal{L}}$.
Indeed, working on the event $E$, since $\hat{\mathcal{L}}$ is $\mu$-strongly convex, we have that $\hat{\mathcal{L}}(h) - \hat{\mathcal{L}}(\hat h_{\beta}) \ge \frac{1}{2} \mu \|h - \hat h_{\beta}\|^2$ for all $h\in \tilde \cH$.
Next, for any $h_1, h_2 \in \tilde \cH$, we have 
\begin{align*}
    \hat{\mathcal{L}}(h_2) - \hat{\mathcal{L}}(h_1) = \hat{\mathcal{L}}_{I,1}(h_2) - \hat{\mathcal{L}}_{I,1}(h_1) &+ \frac{\mathcal{J}_1(h_2,\beta,Z_{I,1}) - \mathcal{J}_1(h_1,\beta,Z_{I,1})}{m_1} \\
    &- \frac{\mathcal{J}_1(h_2,\beta,Z_1') - \mathcal{J}_1(h_1,\beta,Z_1')}{m_1}.
\end{align*}
Setting $h_2 = \hat h_{\beta}^{(I)}$ and $h_1 = \hat h$, since $\hat h_{\beta}^{(I)}$ minimizes $h\mapsto \hat{\mathcal{L}}_{I,1}(h,\beta)$, and using our lower bound on $\hat{\mathcal{L}}(h) - \hat{\mathcal{L}}(\hat h_{\beta})$, we deduce 
\begin{align}\label{str-cvx-step}
    \frac{1}{2} \mu \|\hat h_{\beta}^{(I)} - \hat h_{\beta}\|^2 \le \frac{\mathcal{J}_1(\hat h_{\beta}^{(I)},\beta,Z_{I,1}) - \mathcal{J}_1(\hat h_{\beta},\beta,Z_{I,1})}{m_1} - \frac{\mathcal{J}_1(\hat h_{\beta}^{(I)},\beta,Z_1') - \mathcal{J}_1(\hat h_{\beta},\beta,Z_1')}{m_1}.
\end{align}
Since by assumption, $h\mapsto \mathcal{J}_1(h,\beta,z)$ is $\rho$-Lipschitz, we have the bounds $|\mathcal{J}_1(\hat h_{\beta}^{(I)},\beta,Z_{I,1}) - \mathcal{J}_1(\hat h_{\beta},\beta,Z_{I,1})| \le \rho|\hat h_{\beta}^{(I)} - \hat h_{\beta}|$ and $|\mathcal{J}_1(\hat h_{\beta}^{(I)},\beta,Z_1') - \mathcal{J}_1(\hat h_{\beta},\beta,Z_1')| \le \rho|\hat h_{\beta}^{(I)} - \hat h_{\beta}|$. Plugging these into \Cref{str-cvx-step}, we obtain $\frac{1}{2} \mu \|\hat h_{\beta}^{(I)} - \hat h_{\beta}\|^2 \le \frac{2\rho}{m_1} \|\hat h_{\beta}^{(I)} - \hat h_{\beta}\|$, so that $\| \hat h_{\beta}^{(I)} - \hat h_{\beta} \| \le \frac{4\rho(\beta)}{\mu(\beta) m_1}$.
Using once again that $h\mapsto \mathcal{J}_1(h,\beta,z)$ is $\rho$-Lipschitz, we find
$|\mathcal{J}_1(\hat h_{\beta}^{(I)}, \beta, Z_{I,1}) - \mathcal{J}_1(\hat h_{\beta}, \beta, Z_{I,1})| \le \frac{4\rho(\beta)^2}{\mu(\beta) m_1}$.

Similarly, on the event $E$, we have the bound 
$| \mathcal{J}_2(\hat h_{\beta}^{(J)}, \beta, Z_{J,2}) - \mathcal{J}_2(\hat h_{\beta}, \beta, Z_{J,2}) | \le \frac{4\rho(\beta)^2}{\mu(\beta) m_2}$.
Thus the first and third terms are bounded in magnitude by $\frac{4\rho(\beta)^2}{\mu(\beta) m_1}$ and $\frac{4\rho(\beta)^2}{\mu(\beta) m_2}$, respectively. Due to \eqref{rhomu}, their sum is at most $C(m_1^{-1} + m_2^{-1})$.

By our assumption that $|\mathcal{J}_1(h,\beta,Z)|$ and $|\mathcal{J}_2(h,\beta,Z)|$ are almost surely bounded by a constant for both $Z\sim \dd_1$ and $Z\sim \dd_2$, and our assumption that $\PP{E^c} = o(m_1^{-1} + m_2^{-1})$, 
the second term and fourth terms 
from \eqref{erdecomp}
sum to $o(m_1^{-1} + m_2^{-1})$.
Thus for for each $\beta$, for sufficiently large $m_1, m_2$, we have
$   \E[ |W_{\beta}| ] \le 2C(m_1^{-1}+m_2^{-1})$.
By Markov's inequality,
for any fixed $t>0$, 
$|W_{\beta}| > t$ with probability at most $\frac{2C}{t} (m_1^{-1}+m_2^{-1})$.
We now use chaining.
Let $N$ be an $\ep$-net for $\mathcal{I}$. Then using the fact that 
by assumption, the process $W$ is $K_{m_1,m_2}$-Lipschitz, and by a union bound,
\begin{align*}
    \PP{ \sup_{\beta\in \mathcal{I}} |W_{\beta}| > K_{m_1,m_2}\ep + t } \le \PP{ \sup_{\beta \in N} |W_{\beta}| > t } \le |N| \frac{2C}{t} (m_1^{-1}+m_2^{-1}).
\end{align*}
Pick $N$ with $|N|=1/\ep$, and set
$   t = \frac{4C}{\delta} (m_1^{-1}+m_2^{-1}) \frac{1}{\ep}$.
We deduce that
\begin{align*}
    \sup_{\beta \in \mathcal{I}} |W_{\beta}| > K_{m_1,m_2} \ep + \frac{4C}{\delta} (m_1^{-1}+m_2^{-1}) \frac{1}{\ep}
\end{align*}
with probability at most $\frac{\delta}{2}$.
Set
$    \ep = \sqrt{ \frac{4C}{K_{m_1,m_2} \delta} (m_1^{-1}+m_2^{-1}) }$.
We deduce that
\begin{align*}
    \sup_{\beta \in \mathcal{I}} |W_{\beta}| > \sqrt{\frac{16 CK_{m_1,m_2}}{\delta} (m_1^{-1}+m_2^{-1})}
\end{align*}
with probability at most $\frac{\delta}{2}$.
Since the probability of  $K_{m_1,m_2}\le K_{\tn{max}}$ converges to unity, for sufficiently large $m_1, m_2$,
\begin{align*}
    \sup_{\beta \in \mathcal{I}} |W_{\beta}| > \sqrt{\frac{16 CK_{\tn{max}}}{\delta} (m_1^{-1}+m_2^{-1})}
\end{align*}
holds with probability at most $\delta$.
Since $|W_{\hat \beta}| \le \sup_{\beta\in \mathcal{I}} |W_{\beta}|$, we may conclude.
\end{proof}

\section{Lipschitz process}\label{sec:lip_proc}

\begin{lemma}[Lipschitzness of minimizer of perturbed strongly convex objective]\label{lem:perturb}
    Let $\cc \subseteq \R^d$ be a closed convex set. 
    Suppose $\psi : \cc \to \R$ is $\mu$-strongly convex and 
    $g : \cc \to \R$ is 
    $L$-smooth. 
    Suppose also that $\psi+g$ is convex. 
    Let $x_\psi$ denote the minimizer of $\psi$ in $\cc$, and let $x_{\psi+g}$ denote the minimizer of $\psi+g$ in $\cc$. Then
    for any $x \in \cc$,
\begin{align*}
    \| x_{\psi+g} - x_\psi \|_2 \le \frac{1}{\mu} ( L \|x_{\psi+g} - x \|_2 + \| \nabla g(x) \|_2 ).
\end{align*}
\end{lemma}
\begin{proof}
    Since $\psi$ is $\mu$-strongly convex and since $x_{\psi+g}, x_\psi$ are minimizers of $\psi+g,\psi$ respectively,
    \begin{align*}
        \mu \| x_{\psi+g} - x_\psi \|_2^2 &\le \langle \nabla \psi(x_{\psi+g}) - \nabla \psi(x_\psi), x_{\psi+g} - x_\psi \rangle \\
        &=  \langle \nabla (\psi+g)(x_{\psi+g}), x_{\psi+g} - x_\psi \rangle + \langle \nabla \psi(x_\psi), x_\psi - x_{\psi+g} \rangle \\
        &- \langle \nabla g(x_{\psi+g}), x_{\psi+g} - x_\psi \rangle \\
        &\le - \langle \nabla g(x_{\psi+g}), x_{\psi+g} - x_\psi \rangle \\
        &= - \langle \nabla g(x_{\psi+g}) - \nabla g(x), x_{\psi+g} - x_\psi \rangle - \langle \nabla g(x), x_{\psi+g} - x_\psi \rangle,
    \end{align*}
    so that by $L$-smoothness of $g$,
    \begin{align*}
        \mu \|x_{\psi+g} - x_\psi\|_2^2 \le ( L \|x_{\psi+g}-x\|_2 + \| \nabla g(x) \|_2 ) \|x_{\psi+g}-x_\psi\|_2,
    \end{align*}
    which implies the result.
\end{proof}

\begin{lemma}[Lipschitzness of minimizer of perturbed ERM]\label{lem:hat_h_lip}
    Under \Cref{cond:c-phi}, 
    with 
    $\hat \Sigma$ from Condition \ref{cond:sample-cov-matr},
    and with the notations of Lemma \ref{lem:lip-emp-procs},
    we have
    with respect to the norm $\|\cdot\|$ on $\cH_B$
    that $\beta \mapsto \hat h_{\beta}$ is $C_1$-Lipschitz on $\mathcal{I}$, and $\beta \mapsto \beta \hat h_{\beta}$ is $C_2$-Lipschitz on $\mathcal{I}$, where
    \begin{align}\label{C12}
        C_1 &= (\beta_{\tn{min}}^2 \lambda_{\tn{min}}(\hat \Sigma))^{-1} (  ( 2\beta_{\tn{max}} \lambda_{\tn{max}}(\hat \Sigma) B + C_{\Phi}) + 4 \beta_{\tn{max}} \lambda_{\tn{max}}(\hat \Sigma)  B ),\quad
        C_2 = B + \beta_{\tn{max}} C_1.
    \end{align}
\end{lemma}

\begin{proof}

    First, consider $\hat h_{\beta}$. Fix $\beta > \beta'$ in $\mathcal{I}$. 
    Recalling the definition of $\hat L$ from \eqref{opt:emp-unconstr-CP-obj},
    the difference between the objectives $\hat L(h,\beta)$ and $\hat L(h, \beta')$ is the quadratic
    \begin{align*}
        g(h) := \hat L(h,\beta) - \hat L(h,\beta') = \lambda \hat \E_3[ (\beta^2 - (\beta')^2) h^2 ] + \lambda \hat \E_2[-2(\beta - \beta') h].
    \end{align*}
    We claim that $g$ is $2\lambda (\beta^2 - (\beta')^2) \lambda_{\tn{min}}(\hat \Sigma)$-strongly convex and $2\lambda (\beta^2 - (\beta')^2) \lambda_{\tn{max}}(\hat \Sigma)$-smooth in $h$. To see this, write $h = \langle \gamma, \Phi \rangle$ for $\gamma \in \R^d$, and note that $g$ can be rewritten as
    \begin{align*}
        g(\gamma) = \lambda (\beta^2 - (\beta')^2) \gamma^\top \hat \Sigma \gamma - 2 (\beta - \beta') \lambda \gamma^\top \hat \E_2[\Phi],
    \end{align*}
    a quadratic whose Hessian equals $2\lambda (\beta^2 - (\beta')^2) \hat \Sigma$, which implies the claim.
    
    Similarly, we claim that the function $\psi(h) := \hat L(h, \beta')$ is $2\lambda (\beta')^2 \lambda_{\tn{min}}(\hat \Sigma)$-strongly convex in $h$. To see this, again write $h = \langle \gamma, \Phi \rangle$ for $\gamma \in \R^d$, and note that $\psi$ can be rewritten as
    \begin{align*}
        \psi(\gamma) = \lambda (\beta')^2 \gamma^\top \hat \Sigma \gamma + \hat \E_1[\ell_{\alpha}(\gamma^\top \Phi, S)] + \lambda \hat \E_2[-2\beta \gamma^\top \Phi].
    \end{align*}
    By \Cref{lem:pinball-cvx}, the second term is convex, and since the third term is linear, it too is convex. The Hessian of the quadratic first term is $2\lambda (\beta')^2 \hat \Sigma$, from which it follows that $\psi$ is $2\lambda (\beta')^2 \lambda_{\tn{min}}(\hat \Sigma)$-strongly convex.
    
    Thus $\psi$ and $g$ satisfy the conditions of \Cref{lem:perturb}, which implies the bound 
    \begin{align}\label{hbeta-beta-prime}
        \| \hat h_{\beta} - \hat h_{\beta'} \| \le (2\lambda (\beta')^2 \lambda_{\tn{min}}(\hat \Sigma))^{-1} ( \| \nabla g(\hat h_g) \|_2 + 2\lambda (\beta^2 - (\beta')^2) \lambda_{\tn{max}}(\hat \Sigma)  \cdot \| \hat h_{\beta} - \hat h_g \| ),
    \end{align}
    where $\hat h_g = \hat h_{g, \beta, \beta'}$ denotes the minimizer of $g$ in $\cH_B$.
    Since
    \begin{align*}
        \nabla g(\gamma) = \lambda (\beta - \beta') ( (\beta + \beta') 2 \hat \Sigma \gamma - 2 \hat \E_2[\Phi] ),
    \end{align*}
    and by $|\beta|, |\beta'| \le \beta_{\tn{max}}$, $\|\gamma\| \le B$, and \Cref{cond:c-phi}, we have
    \begin{align*}
        \| \nabla g(\gamma) \|_2 \le \lambda  ( 4\beta_{\tn{max}} \lambda_{\tn{max}}(\hat \Sigma) B + 2 C_{\Phi}) |\beta - \beta'|
    \end{align*}
    for $\beta, \beta' \in \mathcal{I}$ and $h\in \cH_B$.
    Plugging this into the bound 
    \eqref{hbeta-beta-prime}
    on $\| \hat h_{\beta} - \hat h_{\beta'} \|$ and using the fact that $\beta' \ge \beta_{\tn{min}}$ and $\|\hat h_{\beta}\|, \|\hat h_g\| \le B$,
    \begin{align*}
        \| \hat h_{\beta} - \hat h_{\beta'} \| &\le \\
        &(2\lambda \beta_{\tn{min}}^2 \lambda_{\tn{min}}(\hat \Sigma))^{-1} ( \lambda  ( 4\beta_{\tn{max}} \lambda_{\tn{max}}(\hat \Sigma) B + 2 C_{\Phi}) |\beta - \beta'| + 8\lambda \beta_{\tn{max}} \lambda_{\tn{max}}(\hat \Sigma)  B |\beta - \beta'| ).
    \end{align*}
    Thus we may take
\begin{align*}
    C_1 = (\beta_{\tn{min}}^2 \lambda_{\tn{min}}(\hat \Sigma))^{-1} (  ( 2\beta_{\tn{max}} \lambda_{\tn{max}}(\hat \Sigma) B + C_{\Phi}) + 4 \beta_{\tn{max}} \lambda_{\tn{max}}(\hat \Sigma)  B ).
\end{align*}
For the map $\beta \mapsto \beta \hat h_{\beta}$, fix $\beta > \beta'$ in $\mathcal{I}$, and write $\|\beta \hat h_{\beta} - \beta' \hat h_{\beta'}\| \le |\beta - \beta'| \|\hat h_{\beta}\| + |\beta'| \|\hat h_{\beta} - \hat h_{\beta'}\|$.
For the first term, note that since $\hat h_{\beta} \in \cH_B$ implies $\|\hat h_{\beta}\| \le B$, the first term is bounded by $B|\beta - \beta'|$.
For the second term, note that since $|\beta'| \le \beta_{\tn{max}}$ and since $\beta \mapsto \hat h_{\beta}$ is $C_1$-Lipschitz on $\mathcal{I}$, the second term is bounded by $\beta_{\tn{max}} C_1 |\beta - \beta'|$. Summing, we deduce that $\beta \mapsto \beta \hat h_{\beta}$ is $C_2$-Lipschitz on $\mathcal{I}$, where
$C_2 = B + \beta_{\tn{max}} C_1$.
\end{proof}

\begin{lemma}[Lipschitzness of minimizer of perturbed auxiliary ERM]\label{lem:tilde_h_lip}
    Under \Cref{cond:c-phi}, we have that $\beta \mapsto \tilde h_{\beta}$ is $C_1$-Lipschitz on $\mathcal{I}$, and $\beta \mapsto \beta \tilde h_{\beta}$ is $C_2$-Lipschitz on $\mathcal{I}$.
\end{lemma}

\begin{proof}
    The proof is almost identical to \Cref{lem:hat_h_lip}.
\end{proof}

Recalling 
$c_{\tn{min}}$ and  $c_{\tn{max}}$ from 
\Cref{cond:sample-cov-matr},
define
\begin{align}\label{C12max}
    C_{1, \tn{max}} &= (\beta_{\tn{min}}^2 c_{\tn{min}})^{-1} (  ( 2\beta_{\tn{max}} c_{\tn{max}} B + C_{\Phi}) + 4 \beta_{\tn{max}} c_{\tn{max}}  B ),\qquad
    C_{2, \tn{max}} = B + \beta_{\tn{max}} C_{1,\tn{max}},
\end{align}
so that by \Cref{cond:sample-cov-matr}, $C_1 \le C_{1,\tn{max}}$ and $C_2 \le C_{2,\tn{max}}$ with probability tending to unity over the randomness in $\mathcal{S}_3$.

We now compute the Lipschitz constants of the processes used in the proof of \Cref{thm:generalize}.

Recall
$\bar L $ from \eqref{barL},
$\hat L$ from \eqref{opt:emp-unconstr-CP-obj},
$\tilde L$ from \eqref{tildeL}, 
and 
$\cH_B= \{ \langle \gamma, \Phi \rangle : \|\gamma\|_2 \le B < \infty \}$
from \Cref{theory}.
For any fixed $\beta \in \mathcal{I},$
define $\hat h_{\beta}$ as the minimizer of $h \mapsto \hat L(h,\beta)$ over $\mathcal{H}_B$, which exists under the conditions of \Cref{thm:generalize}
due to our argument 
checking the convexity of $h \mapsto \hat L(h,\beta)$
in Term (I) in the
proof of \Cref{thm:generalize}.

\begin{lemma}\label{lem:lip-emp-procs}
Assume the conditions of \Cref{thm:generalize}.
Define the stochastic processes 
$\bar W_{\beta}$ and $\tilde W_{\beta}$ on $\mathcal{I}$
given by $\beta \mapsto (\bar{L} - \hat L)(\hat h_{\beta}, \beta)$ 
and $\beta \mapsto (\tilde L - \hat L)(\hat h_{\beta}, \beta)$, respectively. 
Then $\bar W_{\beta}$ is $K_{1,\lambda}$-Lipschitz on $\mathcal{I}$ 
with probability tending to unity as $n_1,n_2,n_3\to\infty$, 
and $\tilde W_{\beta}$ is $K_{2,\lambda}$-Lipschitz on $\mathcal{I}$ 
with probability tending to unity as $n_1,n_2,n_3\to\infty$,
where 
\begin{align*}
    K_{1,\lambda} &:= 2C_{\Phi} (C_{2,\tn{upper}} + C_{2,\tn{max}}) (1 + \beta_{\tn{max}} BC_{\Phi}) \lambda =: a_1 \lambda, \\ 
    K_{2,\lambda} &:= (1-\alpha) C_{\Phi} (C_{1,\tn{upper}} + C_{1,\tn{max}}) =: a_2,
\end{align*}
with $C_{1, \tn{max}}$ and $C_{2, \tn{max}}$ are defined in \eqref{C12max}
and where $C_{1,\tn{upper}}$ satisfies \Cref{cond:c1-c2-bds} and 
$C_{2,\tn{upper}} := BC_{\Phi} + \beta_{\tn{max}} C_{\Phi} C_{1,\tn{upper}}$.
In fact, $\bar W$ is $K_{1,\lambda}$-Lipschitz on $\mathcal{I}$ with probability tending to unity conditional on $\mathcal{S}_1$,
and  $\tilde W$ is $K_{2,\lambda}$-Lipschitz on $\mathcal{I}$ deterministically, when conditioning on $\mathcal{S}_2, \mathcal{S}_3$, when the event $C_1 \le C_{1, \tn{max}}$ holds.
\end{lemma}

\begin{proof}
We start with the process $\tilde W$. 
Consider $\beta, \beta' \in \mathcal{I}$.
Note that for any $(h,\beta)$, 
using the definition of $\tilde L$ from \eqref{tildeL},
we have the identity 
\begin{align*}
    \tilde L(h,\beta) - \hat L(h,\beta) = \E_1[\ell_{\alpha}(h,S)] - \hat \E_1[\ell_{\alpha}(h,S)].
\end{align*}
Thus we may write 
\begin{align*}
    \tilde {W}_{\beta} - \tilde {W}_{\beta'} &= 
    (\E_1[\ell_{\alpha}(\hat h_{\beta}, S)] - \E_1[\ell_{\alpha}(\hat h_{\beta'}, S)]) 
    - 
    (\hat \E_1[\ell_{\alpha}(\hat h_{\beta}, S)] - \hat \E_1[\ell_{\alpha}(\hat h_{\beta'}, S)]),
\end{align*}
so that 
\begin{align}\label{2terms}
    |\tilde {W}_{\beta} - \tilde {W}_{\beta'}| &\le 
    \E_1[ | \ell_{\alpha}(\hat h_{\beta}, S) - \ell_{\alpha}(\hat h_{\beta'}, S) | ] 
    + 
    \hat \E_1[ | \ell_{\alpha}(\hat h_{\beta}, S) - \ell_{\alpha}(\hat h_{\beta'}, S) | ]
\end{align}
Note that we have the uniform bound
\begin{align*}
     |\ell_{\alpha}(\hat h_{\beta}, S) - \ell_{\alpha}(\hat h_{\beta'}, S)| &\le (1-\alpha)  |\hat h_{\beta} - \hat h_{\beta'}| \\
     &\le (1-\alpha) C_{\Phi} \| \hat h_{\beta} - \hat h_{\beta'} \| \le (1-\alpha) C_{\Phi} C_1 |\beta - \beta'|,
\end{align*}
where in the first step we applied \Cref{lem:pinball-lip}, in the second step we used \Cref{cond:c-phi} to apply \Cref{lem:h-unif-bd}, and in the third step we used \Cref{lem:hat_h_lip}. 
Thus the first term in \Cref{2terms} is bounded by
$(1-\alpha) C_{\Phi} \E_1[C_1] |\beta - \beta'|$,
and the second term in \Cref{2terms} is bounded by 
$(1-\alpha) C_{\Phi} \hat \E_1[C_1] |\beta - \beta'|$.
Summing, we deduce that
\begin{align*}
    |\tilde {W}_{\beta} - \tilde {W}_{\beta'}| \le (1-\alpha) C_{\Phi} (\E_1[C_1] + \hat \E_1[C_1]) |\beta - \beta'|,
\end{align*}
so that the process $\tilde W$ is $K_2$-Lipschitz with $K_2 := (1-\alpha) C_{\Phi} (\E_1[C_1] + \hat \E_1[C_1])$.

We now condition on $\mathcal{S}_2, \mathcal{S}_3$.
Observe that $C_1, C_2$ are $\mathcal{S}_3$-measurable (as     $\hat \Sigma$ from Condition \ref{cond:sample-cov-matr} is $\mathcal{S}_3$-measurable).
Since $\E_1[C_1] \le C_{1,\tn{upper}}$,
on the event that
$C_1 \le C_{1,\tn{max}}$, 
we have $K_2 \le K_{2,\lambda}$,
where
$K_{2,\lambda} = (1-\alpha) C_{\Phi} (C_{1,\tn{upper}} + C_{1,\tn{max}})$, as claimed.

We now continue with the process $\bar W$.
Consider $\beta, \beta' \in \mathcal{I}$.
Note that for any $(h,\beta)$, using the definition of $\bar L$ from \Cref{barL}, we have the identity
\begin{align*}
    \bar L(h,\beta) - \hat L(h,\beta) = ( \lambda \E_3[ \beta^2 h^2 ] + \lambda \E_2[ -2\beta h ] ) - ( \lambda \hat \E_3[ \beta^2 h^2 ] + \lambda \hat \E_2[ -2\beta h ] ).
\end{align*}
Thus we may write 
\begin{align*}
    \bar W_{\beta} - \bar W_{\beta'} &= \lambda (\E_3[ \beta^2 \hat h_{\beta}^2 ] - \E_3[ (\beta')^2 \hat h_{\beta'}^2 ]) + \lambda (\E_2[ -2\beta \hat h_{\beta} ] - \E_2[ -2\beta' \hat h_{\beta'} ]) \\ 
    &- \lambda (\hat \E_3[ \beta^2 \hat h_{\beta}^2 ] - \hat \E_3[ (\beta')^2 \hat h_{\beta'}^2 ])
    - \lambda (\hat \E_2[ -2\beta \hat h_{\beta} ] - \hat \E_2[ -2\beta' \hat h_{\beta'} ]),
\end{align*}
so that
\begin{align}\label{4terms}
    |\bar W_{\beta} - \bar W_{\beta'}| &\le 
    \lambda \E_3[ | \beta^2 \hat h_{\beta}^2 - (\beta')^2 \hat h_{\beta'}^2 | ] 
    + 2\lambda \E_2[ | \beta \hat h_{\beta} - \beta' \hat h_{\beta'} | ] \nonumber \\
    &+ \lambda \hat \E_3[ | \beta^2 \hat h_{\beta}^2 - (\beta')^2 \hat h_{\beta'}^2 | ] 
    + 2\lambda \hat \E_2[ | \beta \hat h_{\beta} - \beta' \hat h_{\beta'} | ]
\end{align}
The integrands of the first and third terms of \Cref{4terms} can be uniformly bounded as
\begin{align*}
    &| \beta^2 \hat h_{\beta}^2 - (\beta')^2 \hat h_{\beta'}^2 | \le | \beta \hat h_{\beta} - \beta' \hat h_{\beta'} | \cdot | \beta \hat h_{\beta} + \beta' \hat h_{\beta'} | \le C_{\Phi} \| \beta \hat h_{\beta} - \beta' \hat h_{\beta'} \| \cdot C_{\Phi} \| \beta \hat h_{\beta} + \beta' \hat h_{\beta'} \| \\
    &\le C_{\Phi} C_2 |\beta - \beta'| \cdot 2C_{\Phi} \beta_{\tn{max}} B = 2 \beta_{\tn{max}} BC_{\Phi}^2 C_2 |\beta - \beta'|.
\end{align*}
where in the first step we used difference of squares, in the second step we used \Cref{cond:c-phi} to apply \Cref{lem:h-unif-bd}, in the third step we applied \Cref{lem:hat_h_lip} to bound the first factor and the triangle inequality and the bounds $\beta \le \beta_{\tn{max}}$ for $\beta\in \mathcal{I}$ and $\|h\|\le B$ for $h\in \cH_B$ to bound the second factor.
The integrand of the second and fourth term in \eqref{4terms} can be bounded as $| \beta \hat h_{\beta} - \beta' \hat h_{\beta'} | \le C_{\Phi} \| \beta \hat h_{\beta} - \beta' \hat h_{\beta'} \| \le C_{\Phi} C_2 |\beta - \beta'|$, where in the first step we used \Cref{cond:c-phi} to apply \Cref{lem:h-unif-bd}, and in the second step we applied \Cref{lem:hat_h_lip}.

Plugging these into our bound in \Cref{4terms}, we deduce
\begin{align*}
    |\bar W_{\beta} - \bar W_{\beta'}| \le (2 C_{\Phi} (\E_2[C_2] + \hat \E_2[C_2]) + 2 \beta_{\tn{max}} C_{\Phi}^2 B (\E_3[C_2] + \hat \E_3[C_2]) ) \lambda |\beta - \beta'|,
\end{align*}
so that the process $\bar W$ is $K_1$-Lipschitz with
\begin{align*}
     K_1 &= (2 C_{\Phi} (\E_2[C_2] + \hat \E_2[C_2]) + 2 \beta_{\tn{max}} C_{\Phi}^2 B (\E_3[C_2] + \hat \E_3[C_2]) ) \lambda.
\end{align*}
We now work conditional on $\mathcal{S}_1$.
On the event that
$C_1 \le C_{1,\tn{max}}$ and 
$C_2 \le C_{2,\tn{max}}$, 
and by \Cref{cond:c1-c2-bds},
we have 
$K_1 \le K_{1,\tn{max}}$, 
where
\begin{align*}
    K_{1,\lambda} &= 
    (2 C_{\Phi} (C_{2,\tn{upper}} + C_{2,\tn{max}}) + 2 \beta_{\tn{max}} C_{\Phi}^2 B (C_{2,\tn{upper}} + C_{2,\tn{max}})) \lambda \\
    &= 2 C_{\Phi} (C_{2,\tn{upper}} + C_{2,\tn{max}}) (1 + \beta_{\tn{max}} B C_{\Phi} )  \lambda.
\end{align*}
Since $C_1 \le C_{1,\tn{max}}$ and 
$C_2 \le C_{2,\tn{max}}$ with probability tending to one due to Condition \ref{cond:sample-cov-matr}, 
$K_1 \le K_{1,\lambda}$
and 
$K_2 \le K_{2,\lambda}$
both hold with probability tending to one if we uncondition on $\mathcal{S}_1$, and we are done.
\end{proof}

\section{Proof of Proposition \ref{prop:infinite-sample}}
\label{pf:infinite-sample}

Fix $\lambda\ge 0$.
Under the assumptions of \Cref{lem:unconstr-existence},
there exists a global minimizer $(h^*,\beta^*)$ of $L(h,\beta)$.
The first order condition with respect to $\beta$ 
reads $2\lambda \E_1[ h^*(X) (\beta^* h^*(X) - r(X)) ] = 0$.
By \Cref{lem:pinball-deriv}, the first order condition with respect to $h$ reads 
\begin{align*}
    &\E_1[ h^*(X) (\mathbb{P}_{S|X}[S(X,Y)\le h^*(X)] - (1-\alpha)) ] 
    + 2\lambda \E_1[ \beta^* h(X) (\beta^* h^*(X) - r(X)) ] = 0
\end{align*}
for all $h\in \cH$.
Setting $h = r_{\cH}$ in the second equation, 
and subtracting $(\beta^*)^2$ times the first equation
from the second, we deduce that 
\begin{align*}
     &\E_1[ h^*(X) (\mathbb{P}_{S|X}[S(X,Y)\le h^*(X)] - (1-\alpha)) ] \\
    &+ 2\lambda \E_1[ \beta^* \cdot r_{\cH}(X) \cdot (\beta^* h^*(X) - r(X)) ] 
    - 2\lambda \E_1[ \beta^* \cdot \beta^*  h^*(X) \cdot (\beta^* h^*(X) - r(X)) ] \\
    &= \E_1[ h^*(X) (\mathbb{P}_{S|X}[S(X,Y)\le h^*(X)] - (1-\alpha)) ] \\
    &+ 2 \lambda \E_1[ \beta^* (r_{\cH}(X) - \beta^* h^*(X)) (\beta^* h^*(X) - r(X)) ] \\
    &= \E_1[ h^*(X) \mathbb{P}_{S|X}[S(X,Y)\le h^*(X)] ] - (1-\alpha) 
    - 2 \lambda \beta^* \E_1[ (r_{\cH}(X) - \beta^* h^*(X))^2 ]
    = 0.
\end{align*}
Therefore,
\begin{align*}
    &\E_1[ r_{\cH}(X) \mathbb{P}_{S|X}[S(X,Y)\le h^*(X)] ] 
    = (1-\alpha) + 2\lambda \beta^* \E_1[ (r_{\cH}(X) - \beta^* h^*(X))^2 ],
\end{align*}
which implies the result.

\section{Proof of Theorem \ref{thm:generalize}}
\label{pf:generalize}

Recall that $\mathcal{S}_1$ are the features of the labeled calibration dataset.
We also recall the notation $\E_j$ and $\hat \E_j$ for $j=1,2,3$ from \Cref{pf}.
Given
  the unlabeled test  data $\mathcal{S}_2$ and the 
   unlabeled calibration data $\mathcal{S}_3$, 
define the auxiliary risks
for $h \in \cH_B, \beta \in \mathcal{I}$,
\begin{align}\label{tildeL}
     \tilde L(h, \beta; \mathcal{S}_2, \mathcal{S}_3) := \E_1[\ell_{\alpha}(h,S)] + \lambda \hat \E_3[ \beta^2 h^2 ] + \lambda \hat \E_2[ -2\beta h ]
\end{align}
and
\begin{align}\label{barL}
\bar{L}(h, \beta; \mathcal{S}_1) := \hat \E_1[\ell_{\alpha}(h,S)] + \lambda \E_3[\beta^2 h^2] + \lambda \E_2[-2\beta h].
\end{align}
Let
\begin{align}\label{thb}
    (\tilde h, \tilde \beta) \in \arg\min_{h \in \cH_B, \beta \in \mathcal{I}} \tilde L(h, \beta; \mathcal{S}_2, \mathcal{S}_3).
\end{align}
For convenience, 
we leave implicit the dependence of $\tilde L$ 
and $(\tilde h, \tilde \beta)$  
 on $\mathcal{S}_2$, $\mathcal{S}_3$
and the dependence of $\bar{L}$ on $\mathcal{S}_1$.

In order to study the generalization error, we write 
\begin{align*}
    L(\hat h, \hat \beta) - L(h^*, \beta^*) &= (L(\hat h, \hat \beta) - \tilde L(\hat h, \hat \beta)) + (\tilde L(\hat h, \hat \beta) - \hat L(\hat h, \hat \beta)) + (\hat L(\hat h, \hat \beta) - \hat L(\tilde h, \tilde \beta)) \\
    &+ (\hat L(\tilde h, \tilde \beta) - \tilde L(\tilde h, \tilde \beta)) + (\tilde L(\tilde h, \tilde \beta) - \tilde L(h^*, \beta^*)) + (\tilde L(h^*, \beta^*) - L(h^*, \beta^*)).
\end{align*}
Since $(\hat h, \hat \beta)$ is a minimizer of the risk $\hat L$, we have $\hat L(\hat h, \hat \beta) - \hat L(\tilde h, \tilde \beta) \le 0$, and since $(\tilde h, \tilde \beta)$ is a minimizer of the risk $\tilde L$, we have $\tilde L(\tilde h, \tilde \beta) - \tilde L(h^*, \beta^*) \le 0$.
Thus our generalization error is bounded by the remaining four terms:
\begin{align}
    L(\hat h, \hat \beta) - L(h^*, \beta^*) &\le (L(\hat h, \hat \beta) - \tilde L(\hat h, \hat \beta)) + (\tilde L(\hat h, \hat \beta) - \hat L(\hat h, \hat \beta)) \nonumber\\
    &+ (\hat L(\tilde h, \tilde \beta) - \tilde L(\tilde h, \tilde \beta)) + (\tilde L(h^*, \beta^*) - L(h^*, \beta^*)) \nonumber\\
    &=: (I) + (II) + (III) + (IV).\label{termsi-iv}
\end{align}
We study the generalization error by conditioning on the unlabeled calibration or test data. Then our regularization becomes data-independent.
Conditional on $\mathcal{S}_1$, Term (I) can be handled with \Cref{lem:generalize} above. Conditional on $\mathcal{S}_2, \mathcal{S}_3$, Term (II) can be handled with \Cref{lem:generalize} above.
Terms (III) and (IV) are empirical processes at fixed functions, conditional on $\mathcal{S}_2, \mathcal{S}_3$.

\textbf{Term (I):} We work conditional on $\mathcal{S}_1$.
First, note that 
due to the definition of $\hat L$ from \eqref{opt:emp-unconstr-CP-obj},
we can write for any $(h,\beta)$,
\begin{align*}
    &L(h,\beta) - \tilde L(h,\beta) = \bar{L}(h,\beta) - \hat L(h,\beta).
\end{align*}
Since $\bar{L}(h,\beta) - \hat L(h,\beta)$  can be viewed as a difference of 
a population risk  
$\lambda \E_3[\beta^2 h^2] + \lambda \E_2[-2\beta h]$
and an empirical risk 
$\lambda \hat \E_3[\beta^2 h^2] + \lambda \hat \E_2[-2\beta h]$
with ``regularizer" 
$\hat \E_1[\ell_{\alpha}(h,S)]$, 
this expression enables us to
 apply \Cref{lem:generalize} to bound $\bar{L}(\hat h,\hat \beta) - \hat L(\hat h,\hat \beta)$.

Explicitly, we can write
\begin{align*}
    \frac{1}{\lambda} \hat L(h,\beta) = \hat \E_3[\beta^2 h^2] + \hat \E_2[-2\beta h] + \frac{1}{\lambda} \hat \E_1[\ell_{\alpha}(h,S)].
\end{align*}

Hence, 
fixing $\beta$, 
we can apply \Cref{lem:generalize}, 
choosing 
$m_1 = n_3$ and $m_2 = n_2$.
Further, 
we choose $\tilde\cH:=\cH_B= \{ \langle \gamma, \Phi \rangle : \|\gamma\|_2 \le B < \infty \}$ with the norm $\langle \gamma, \Phi \rangle =  \|\gamma\|_2$.
Moreover, letting
$z = (x'', x')$ for $x'',x'\in \mathcal{X}$, 
and $\xi = 1/\lambda$,
we use the objective function given by $(h,z)\mapsto f_1(h,z) = \mathcal{J}(h,\beta,z) + \rr(h,\beta)$, where $\mathcal{J}(h,\beta,z) = \mathcal{J}_1(h,\beta,z) + \mathcal{J}_2(h,\beta,z)$, and where
\begin{align*}
    \mathcal{J}_1(h,\beta,z) = \beta^2 h(x'')^2, \qquad \mathcal{J}_2(h,\beta,z) = - 2\beta h(x'), \qquad \rr(h,\beta) = \xi \hat \E_1[\ell_{\alpha}(h,S)].
\end{align*}
We now check the conditions of \Cref{lem:generalize}.

\textit{Boundedness:}
Note that $|\mathcal{J}_1(h,\beta,z)| = |\beta|^2 |h(x'')|^2 \le \beta_{\tn{max}}^2 (BC_{\Phi})^2$, where in the second step we used $|\beta| \le \beta_{\tn{max}}$ for $\beta \in \mathcal{I}$, and we used $h\in \cH_B$ and \Cref{cond:c-phi} to apply \Cref{lem:h-unif-bd}.
Similarly, note that $|\mathcal{J}_2(h,\beta,z)| = 2|\beta| |h(x')| \le 2\beta_{\tn{max}} BC_{\Phi}$, where in the second step we used $|\beta| \le \beta_{\tn{max}}$ for $\beta \in \mathcal{I}$, and we used $h\in \cH_B$ and \Cref{cond:c-phi} to apply \Cref{lem:h-unif-bd}.
Thus $|\mathcal{J}_1(h,\beta,z)|$ and $|\mathcal{J}_2(h,\beta,z)|$ are both bounded by the sum $\beta_{\tn{max}}^2 (BC_{\Phi})^2 + 2\beta_{\tn{max}} BC_{\Phi}$.

\textit{Convexity:} Write $h = \langle \gamma, \Phi\rangle$ for $\gamma \in \R^d$.
The map $h\mapsto \mathcal{J}_1(h,\beta,z)$ can equivalently be written as $\gamma \mapsto \beta^2 \gamma^\top \Phi(x'') \Phi(x'')^\top \gamma$, a quadratic whose Hessian equals the positive semidefinite matrix $2\beta^2 \Phi(x'') \Phi(x'')^\top$. Thus $h\mapsto \mathcal{J}_1(h,\beta,z)$ is convex.
The map $h\mapsto \mathcal{J}_2(h,\beta,z)$ can equivalently be written as $\gamma \mapsto - 2\beta \gamma^\top \Phi(x')$, which is linear, hence convex.

\textit{Lipschitzness:} Write $h = \langle \gamma, \Phi\rangle$ for $\gamma \in \R^d$.
The map $h\mapsto \mathcal{J}_1(h,\beta,z)$ can equivalently be written as $\gamma \mapsto \beta^2 \gamma^\top \Phi(x'') \Phi(x'')^\top \gamma$. The gradient of this quadratic is given by $\gamma \mapsto 2\beta^2 \Phi(x'') \Phi(x'')^\top \gamma$. The norm of this gradient can be bounded by 
\begin{align*}
    \| 2\beta^2 \Phi(x'') \Phi(x'')^\top \gamma \|_2 \le 2 |\beta|^2 \|\Phi(x'')\|_2^2 \|\gamma\|_2 \le 2\beta_{\tn{max}}^2 BC_{\Phi}^2,
\end{align*}
where in the first step we applied the Cauchy-Schwarz inequality, in the second step we used $|\beta| \le \beta_{\tn{max}}$ for $\beta \in \mathcal{I}$, $\|\gamma\|_2 \le B$, and \Cref{cond:c-phi}.
Next, the map $h\mapsto \mathcal{J}_2(h,\beta,z)$ can equivalently be written as $\gamma \mapsto - 2\beta \gamma^\top \Phi(x')$. The gradient of this linear map is given by $\gamma \mapsto - 2\beta \Phi(x')$. The norm of this gradient can be bounded by $2|\beta| \|\Phi(x')\| \le 2\beta_{\tn{max}} C_{\Phi}$, where we used $|\beta| \le \beta_{\tn{max}}$ for $\beta \in \mathcal{I}$ and \Cref{cond:c-phi}.
Thus the norm of each of these gradients is bounded by the sum $\rho_1 := 2 \beta_{\tn{max}}^2 BC_{\Phi}^2 +  2 \beta_{\tn{max}} C_{\Phi}$, and the maps $h\mapsto \mathcal{J}_1(h,\beta,z)$ and $h\mapsto \mathcal{J}_2(h,\beta,z)$ are both $\rho_1$-Lipschitz.

\textit{Strong convexity:}
Since $h\mapsto \ell_{\alpha}(h, s)$ is convex for all $s\in \R$ by \Cref{lem:pinball-cvx} and since $h\mapsto \hat \E_2[\beta h]$ is linear, the map $h\mapsto \xi \hat \E_1[ \ell_{\alpha}(h, S) ] - 2 \hat \E_2[\beta h]$ is convex.
Consider the map $h \mapsto \hat \E_{3}[ \beta^2 h^2 ]$. Writing $h = \langle \gamma, \Phi\rangle$ for $\gamma \in \R^d$, this can be rewritten as $\gamma \mapsto \beta^2 \gamma^\top \hat \Sigma \gamma$, a quadratic whose Hessian equals $2\beta^2 \hat \Sigma$. By $\beta \ge \beta_{\tn{min}}$ for $\beta \in \mathcal{I}$ and \Cref{cond:sample-cov-matr}, 
it follows that with probability $1 - o(n_3^{-1}) = 1 - o(n_2^{-1} + n_3^{-1})$, the map $h\mapsto \hat \E_{2,3}[f_1(h,Z)]$ is $\mu_1$-strongly convex, where $Z = (X'',X')$ with $X'$ is uniform over $\mathcal{X}_2$ and
$X''$ is uniform over $\mathcal{X}_3$, and 
where $\mu_1 := 2\beta_{\tn{min}}^2 c_{\tn{min}} $.
In particular, $h\mapsto \frac{1}{\lambda} \hat L(h,\beta)$ is 
convex.

Let $\widetilde C_1 = \frac{4\rho_1^2}{\mu_1}$.
Let $K_1$ denote the Lipschitz constant of the process 
$\bar W_{\beta}$, 
where $K_1 \le K_{1,\lambda}$ 
with probability tending to unity conditional on $\mathcal{S}_1$ by \Cref{cond:c1-c2-bds} and \Cref{lem:lip-emp-procs}.
From \Cref{lem:generalize} applied with $\xi = 1/\lambda$, 
$\mathcal{L}= \frac{1}{\lambda} \bar L$, and
$\hat{\mathcal{L}}  = \frac{1}{\lambda} \hat L$,
and
$W = (\bar L-\hat L)/\lambda$,
we obtain that 
conditional on $\mathcal{S}_1$, for sufficiently large $n_2, n_3$, with probability at least $1-\frac{\delta}{4}$, we have
for Term (I) from \eqref{termsi-iv},
\begin{align*}
    \frac{1}{\lambda} \text{Term (I)} \le \sqrt{\frac{16 \widetilde C_1 K_{1,\lambda} / \lambda}{\delta / 4} \left( \frac{1}{n_2} + \frac{1}{n_3} \right)}.
\end{align*}
Thus
\begin{align*}
\text{Term (I)} &\le \sqrt{\frac{64 \widetilde C_1 \lambda K_{1,\lambda}}{\delta} \left( \frac{1}{n_2} + \frac{1}{n_3} \right)} = A_1 \lambda \sqrt{\frac{1}{n_2} + \frac{1}{n_3}},
\end{align*}
where we define
$A_1 = \sqrt{\frac{64 \widetilde C_1 a_1}{\delta}}$.
Since the right-hand side does not depend on $\mathcal{S}_1$, the same bound holds when we uncondition on $\mathcal{S}_1$.

\textbf{Term (II):}
We work conditional on $\mathcal{S}_2$, $\mathcal{S}_3$.
The risks $\hat L$ and $\tilde L$ share the same data-independent regularization $\lambda \hat \E_3[ \beta^2 h^2 ] + \lambda \hat \E_2[ -2\beta h ]$.
Write $z = (x,s)$ for $x\in \xx$ and $s\in [0,1]$.
Fixing $\beta$, 
we apply \Cref{lem:generalize}
with the objective function 
$(h,z)\mapsto f(h,z) = \mathcal{J}(h,\beta,z) + \rr(h,\beta)$, where
\begin{align*}
    \mathcal{J}(h,\beta,z) = \ell_{\alpha}(h(x), s), \qquad \rr(h,\beta) = \lambda \hat \E_3[ \beta^2 h^2 ] + \lambda \hat \E_2[ -2\beta h ].
\end{align*}
Since the empirical risk $\hat L$ is computed over 
the i.i.d.~sample $Z_i = (X_i,S_i)$ for $i\in [n_1]$, we use the modified version of \Cref{lem:generalize} given in \Cref{rmk:gen-bd-iid}. In particular, we check boundedness, convexity, and Lipschitzness of $\mathcal{J}$ without writing it as a sum $\mathcal{J}_1 + \mathcal{J}_2$.

\textit{Boundedness:} we have the uniform bound, for all $h,\beta,z$
\begin{align}\label{unif_bdd_pinball}
    |\mathcal{J}(h,\beta,z)| \le (1-\alpha) |h(x)-s| \le (1-\alpha) (|h(x)| + 1) \le (1-\alpha) (BC_{\Phi} + 1),
\end{align}
where in the first step we used \Cref{lem:pinball-bds}, in the second step we used the triangle inequality and $s\in [0,1]$, and in the third step we used $h\in \cH_B$ and \Cref{cond:c-phi} to apply \Cref{lem:h-unif-bd}.

\textit{Convexity:}
By \Cref{lem:pinball-cvx}, $h\mapsto \mathcal{J}(h,\beta,z)$ is convex.

\textit{Lipschitzness:}
Fix $h = \langle \gamma, \Phi\rangle$ and $h' = \langle \gamma', \Phi\rangle$ in $\cH_B$, where $\gamma, \gamma' \in \R^d$.
Note that
\begin{align*}
    &|\mathcal{J}(h,\beta,z) - \mathcal{J}(h,\beta,z)| 
    = | \ell_{\alpha}(h(x), s) - \ell_{\alpha}(h'(x), s) | \\
    &\le (1-\alpha) |h(x) - h'(x)|
    \le (1-\alpha) C_{\Phi} \|h - h'\|,
\end{align*}
where in the second step we used \Cref{lem:pinball-lip}, and in the third step we used \Cref{cond:c-phi} to apply \Cref{lem:h-unif-bd}. 
Thus $h\mapsto \mathcal{J}(h,\beta,z)$ is $\rho_2$-Lipschitz, where $\rho_2 := (1-\alpha) C_{\Phi}$.

\textit{Strong convexity:}
To analyze $\rr$, 
first observe that 
since $h\mapsto \lambda \hat \E_2[-2\beta h]$ is linear, it is convex.
Writing $h = \langle \gamma, \Phi\rangle$ for $\gamma \in \R^d$, 
the term $h \mapsto \lambda \hat \E_3[\beta^2 h^2]$ in $\rr$
can be rewritten as $\gamma \mapsto \lambda \beta^2 \gamma^\top \hat \Sigma \gamma$, a quadratic whose Hessian equals $2\lambda \beta^2 \hat \Sigma$. By $\beta \ge \beta_{\tn{min}}$ for $\beta \in \mathcal{I}$ and \Cref{cond:sample-cov-matr}, it follows that with probability $1-o(n_3^{-1})$ over $\mathcal{S}_2, \mathcal{S}_3$, the map $h\mapsto \rr(h,\beta)$ is $\mu_2$-strongly convex, where $\mu_2(\lambda) := 2\lambda \beta_{\tn{min}}^2 c_{\tn{min}}$. 

Let $\widetilde C_2(\lambda) = \frac{4\rho_2^2}{\mu_2(\lambda)}$.
Let $K_2$ denote the Lipschitz constant of the process $W_{\beta}$; recall that conditional on $\mathcal{S}_2, \mathcal{S}_3$, $K_2 \le K_{2,\lambda}$ deterministically 
on the event $C_1 \le C_{1,\tn{max}}$ by  \Cref{lem:lip-emp-procs}.
By the version of \Cref{lem:generalize} given in \Cref{rmk:gen-bd-iid}, conditional on $\mathcal{S}_2$, $\mathcal{S}_3$, if $h\mapsto\rr(h,\beta)$ is $\mu_2(\lambda)$-strongly convex, and if $C_1 \le C_{1,\tn{max}}$, then for sufficiently large $n_1$, with probability at least $1-\frac{\delta}{8}$, we have 
\begin{align}\label{thesamebound}
    \text{Term (II)} &\le \sqrt{\frac{16 \widetilde C_2(\lambda) K_{2,\lambda}}{(\delta / 8) n_1}}  =  \frac{A_2}{\sqrt{\lambda n_1}},
\end{align}
where we define 
$A_2 = \sqrt{\frac{128 \widehat C_2 a_2}{\delta}}$
and
$\widehat C_2 = \frac{4\rho_2^2}{2 \beta_{\tn{min}}^2 c_{\tn{min}}}$.
Unconditioning on $\mathcal{S}_2, \mathcal{S}_3$, since $\rr(h,\beta)$ is $\mu_2(\lambda)$-strongly convex with probability tending to unity by the above analysis, and since by \Cref{cond:sample-cov-matr} we have $C_1 \le C_{1,\tn{max}}$ with probability tending to unity, 
we deduce that for sufficiently large $n_1, n_2, n_3$, with probability at least $1-\frac{\delta}{4}$, 
\eqref{thesamebound} still holds.

\textbf{Term (III):} 
We work conditional on $\mathcal{S}_2, \mathcal{S}_3$.
Since $\tilde h$ from \eqref{thb} lies in $\cH_B$, we may use the bound in \Cref{unif_bdd_pinball} to obtain $\sup_{x\in \xx} |\ell_{\alpha}(\tilde h, S)| \le (1-\alpha) (BC_{\Phi} + 1)$.
Thus by Hoeffding's inequality \citep{hoeffding1963probability}, with probability at least $1-\frac{\delta}{4}$ we have
\begin{align*}
    (\hat L - \tilde L)(\tilde h, \tilde \beta) = (\hat \E_1 - \E_1)[ \ell_{\alpha}(\tilde h, S) ] \le \frac{ (1-\alpha) (BC_{\Phi} + 1) \sqrt{\frac{1}{2} \log\frac{2}{\delta / 4}}}{\sqrt{n_1}}.
\end{align*}
Thus we have $\text{Term (III)}  \le \frac{A_3}{\sqrt{n_1}}$,
where we define
$A_3 = (1-\alpha) (BC_{\Phi} + 1) \sqrt{\frac{1}{2} \log\frac{8}{\delta}}$.

\textbf{Term (IV):} 
Note that we may write
\begin{align*}
    (\tilde L - L)(h^*, \beta^*) &= (\hat \E_2 - \E_2)[ \lambda (\beta^* h^*)^2 ] + (\hat \E_3 - \E_3)[ -2\lambda \beta^* h^* ].
\end{align*}
Since $\|h^*\| \le B$ by $h^*\in \cH_B$ and since \Cref{cond:c-phi} holds, we may apply \Cref{lem:h-unif-bd} to deduce that $\sup_{x\in\xx} |h^*(x)| \le BC_{\Phi}$.
Consequently, for $\beta \in \mathcal{I}$, we have the uniform bound $\sup_{x\in \xx} |\beta h^*(x)| \le \beta_{\tn{max}} BC_{\Phi}$.
By Hoeffding's inequality \citep{hoeffding1963probability}, with probability at least $1-\frac{\delta}{8}$, we have
\begin{align*}
    |(\hat \E_2 - \E_2)[ \lambda (\beta^* h^*)^2 ]|
    &\le \frac{\lambda (\beta_{\tn{max}} BC_{\Phi})^2 \sqrt{\frac{1}{2} \log\frac{2}{\delta/8}} }{\sqrt{n_2}}.
\end{align*}
By another application of Hoeffding's inequality, with probability at least $1-\frac{\delta}{8}$, we have
\begin{align*}
    |(\hat \E_3 - \E_3)[ -2\lambda \beta^* h^* ]| &\le \frac{4\lambda (\beta_{\tn{max}} BC_{\Phi}) \sqrt{\frac{1}{2} \log\frac{2}{\delta/8}}}{\sqrt{n_3}}.
\end{align*}
Summing, with probability at least $1-\delta$ we have the bound
\begin{align*}
    (\tilde L - L)(h^*, \beta^*) &\le \frac{\lambda (\beta_{\tn{max}} BC_{\Phi})^2 \sqrt{\frac{1}{2} \log\frac{16}{\delta}} }{\sqrt{n_2}} + \frac{4\lambda (\beta_{\tn{max}} BC_{\Phi}) \sqrt{\frac{1}{2} \log\frac{16}{\delta}}}{\sqrt{n_3}}.
\end{align*}
Using the inequality $a + b \le \sqrt{2}\sqrt{a^2+b^2}$ for all $a,b\in \R$, we deduce 
$\text{Term (IV)} \le A_4 \lambda \sqrt{\frac{1}{n_2} + \frac{1}{n_3}}$,
where we define
\begin{align*}
    A_4 = \sqrt{2}
    (\beta_{\tn{max}} BC_{\Phi}) \sqrt{\frac{1}{2} \log\frac{16}{\delta}}
    \max\left\{ \beta_{\tn{max}} BC_{\Phi}, 4  \right\}.
\end{align*}
Returning to the analysis of \eqref{termsi-iv},
and summing all four terms while defining $A_5 = A_1 + A_4$, with probability at least $1-\delta$ we obtain a generalization error bound of 
\begin{align*}
    L(\hat h, \hat \beta) - L(h^*, \beta^*) &\le A_5 \lambda \sqrt{\frac{1}{n_2} + \frac{1}{n_3}} 
    + A_3 \frac{1}{\sqrt{n_1}}
    + A_2 \frac{1}{\sqrt{\lambda}} \frac{1}{\sqrt{n_1}}.
\end{align*}
The result follows by taking $c = A_5$, $c' = A_3$, and $c'' = A_2$.

\section{Proof of Theorem \ref{thm:cov-lower-bd}}
\label{pf:cov-lower-bd}

We use the following result to convert the generalization error bound in \Cref{thm:generalize} to a coverage lower bound.

\begin{lemma}[Bounded suboptimality implies bounded gradient for smooth functions]\label{lem:smooth}
Let $f:\mathbb{R}^{d'} \to \mathbb{R}$, for some positive $d'$. 
Suppose $x^*$ is a global minimizer of $f$. Suppose $x'$ is such that $f(x') \le f(x^*) + \varepsilon$. Suppose $h\in \R^d$ is such that the map $g : \R\to \R$ given by $t\mapsto f(x' + th)$ is $L$-smooth, i.e. $|g''(h)|$ is uniformly bounded by $L$. Then
\[
|f'(x';h)| = |\nabla f(x')^\top h| \le \sqrt{2L\varepsilon}\|h\|_2.
\]
\end{lemma}

\begin{proof}

Assume there exists $h$ and $\delta>0$ with $f'(x';h) > \delta\|h\|$. Setting $y=x' - th$,
\begin{align*}
f(x'-th) \le f(x') - t f'(x';h) + \tfrac{L}{2}t^2\|h\|^2.
\end{align*}
Set $t=\delta/(L\|h\|)$ to obtain
\begin{align*}
f(x'-th) \le f(x') - \tfrac{\delta^2}{L} + \tfrac{\delta^2}{2L} = f(x') - \tfrac{\delta^2}{2L}.
\end{align*}
Since $f(x')\le f(x^*)+\varepsilon$,
we have
$f(x'-th) \le f(x^*) + \varepsilon - \tfrac{\delta^2}{2L}$.
If $\delta>\sqrt{2L\varepsilon}$, then $f(x'-th)<f(x^*)$, a contradiction.

A similar argument with $f'(x';h)<-\delta\|h\|$ and $y=x'+th$ yields the same contradiction. Hence
$
-\sqrt{2L\varepsilon}\|h\|\le f'(x';h)\le \sqrt{2L\varepsilon}\|h\|$.
\end{proof}

By \Cref{cond:c-phi} and \Cref{cond:condl-dens-bdd}, we may apply \Cref{lem:pinball-deriv} to deduce that
the Hessian of our population risk $L$ from \eqref{opt:unconstr-CP-obj}
in the basis $\{\phi_1, \ldots, \phi_d\}$
is the block matrix
\begin{align*}
    \nabla^2 L(h,\beta)
    &= \begin{bmatrix}
        \E_1[\Phi \Phi^\top  (f_{S|X}(h) + 2\lambda \beta^2)] & \E_1[ 2\lambda \Phi^\top  (2\beta h - r) ] \\
        \E_1[ 2\lambda \Phi (2\beta h - r) ] & \E_1[2\lambda h^2]
    \end{bmatrix}.
\end{align*}
Thus by $\beta \le \beta_{\tn{max}}$, $\|h\| \le B$ for $h\in \cH_B$, \Cref{cond:condl-dens-bdd}, and Jensen's inequality, we have the uniform bounds 
\begin{align*}
    \sup_{h\in \cH_B, \beta \in \R} | \partial_{\beta}^2 L(h,\beta) | &\le 2\lambda \E_1[h^2] 
    \le 2\lambda B^2 \lambda_{\tn{max}}(\Sigma) 
    =: \nu_1
\end{align*}
and
\begin{align*}
    \sup_{h \in \cH, \beta \in \mathcal{I}} \| \nabla_{h}^2 L(h,\beta) \|_2 &= \| \E_1[\Phi \Phi^\top  (f_{S|X}(h) + 2\lambda \beta^2)] \|_2 \
    \le (C_f + 2\lambda \beta_{\tn{max}}^2) \lambda_{\tn{max}}(\Sigma) 
    =: \nu_2.\end{align*}

By \Cref{lem:unconstr-existence} and \Cref{lem:apriori-bdd}, a global minimizer of the objective in \Cref{opt:unconstr-CP-obj} exists, and since $\beta_{\tn{min}} \le \beta_{\tn{lower}}$, $\beta_{\tn{max}} \ge \beta_{\tn{upper}}$, and $B \ge B_{\tn{upper}}$, any such minimizer lies in the interior of $\cH_B\times \mathcal{I}$. Thus we may apply \Cref{lem:smooth} to the objective function $L$.
We utilize two directional derivatives in the space $\mathcal{H} \times \R$.
The first is in the direction $0_\mathcal{H} \times 1$, the unit vector in the $\beta$ coordinate. 
Since $(\hat h, \hat \beta) \in \cH_B\times \mathcal{I}$, the 
magnitude of the second derivative
of $L$ along this direction is bounded by $\nu_1$.

The second is 
in the direction of
the vector $r_{B} \times 0$, 
where $r_{B}$
the projection of $r$ onto the closed convex set $\cH_B$ in the Hilbert space induced by the inner product $\langle f,g\rangle = \E_1[fg]$. 
Since $(\hat h, \hat \beta) \in \cH_B\times \mathcal{I}$, the
magnitude of the second derivative
 of $L$ along this direction is bounded by $\nu_2$.
 
Given $\hat h$, let $\coverhat(X) := \PP{ S \le \hat h(X) | X } - (1-\alpha)$. 
Now,
on the event $E$ that
$L(\hat h, \hat \beta) - L(h^*, \beta^*) \le \ee_{\tn{gen}}$, 
we apply 
\Cref{lem:smooth}
with $f$ being $(\gamma,\beta)\mapsto L(h_\gamma,\beta)$,
$x^*$ being $(h^*, \beta^*)$,
$x'$ being $(\hat h,\hat \beta)$, 
$\ep = \ee_{\tn{gen}}$,
and the directions specified above, with their respective smoothness parameters derived above. 
Using 
the formulas for $\nabla L$ from \Cref{lem:pinball-deriv}
and
the bound $\|r_{B}\| \le B$, 
we obtain that
on the event $E$,
\begin{align*}
    &| 2\lambda \E_1[ \hat h (\hat \beta \hat h - r) ] | \le \ee_1, \qquad
    | \E_1[ r_{B} \coverhat ] + \lambda \E_1[ 2\beta r_{B} (\hat \beta \hat h - r) ] | \le \ee_2,
\end{align*}
where
$\ee_1 = \sqrt{2 \nu_1 \ee_{\tn{gen}}}$,
$ \ee_2 = \sqrt{2B^2 \nu_2 \ee_{\tn{gen}}}$.

For any $h$ and $\beta$, we may write
\begin{align*}
    \E_1[ r_{B} \coverhat ] &= (\E_1[ r_{B} \coverhat ] + \lambda \E_1[ 2\beta r_{B} (\beta h - r) ] ) \\
   &- \lambda \E_1[ 2\beta (\beta h) (\beta h - r) ] 
   - \lambda \E_1[ 2\beta (r_{B}-\beta h) (\beta h - r) ].
\end{align*}
Evaluating at $(\hat h, \hat \beta)$, the first term is at most $\ee_2$ in magnitude, the second term is at most $\hat \beta^2 \ee_1$ in magnitude, and the third term equals $2 \hat \beta \lambda \E_1[ (r_{B} - \hat \beta \hat h)^2 ]$. We deduce 
\begin{align*}
    \E_1[ r_{B} \coverhat ] \ge 2 \hat \beta \lambda \E_1[ (r_{B} - \hat \beta \hat h)^2 ] - \hat \beta^2 \ee_1 - \ee_2.
\end{align*}
Since $\coverhat \in [-(1-\alpha), \alpha]$, 
\begin{align*}
    |\E_1[ r \coverhat ] - \E_1[ r_{B} \coverhat ]| \le (1-\alpha) \E_1[ |r - r_{B}| ].
\end{align*}
We deduce that
\begin{align*}
    \E_1[ r \coverhat ] &\ge 2 \hat \beta \lambda \E_1[ (r_{B} - \hat \beta \hat h)^2 ] - \hat \beta^2 \ee_1 - \ee_2 - (1-\alpha) \E_1[ |r - r_{B}| ].
\end{align*}
We now bound the quantity $\hat \beta^2 \ee_1 + \ee_2$.
First, since $\sqrt{a+b} \le \sqrt{a} + \sqrt{b}$ for all $a,b\ge 0$, \Cref{thm:generalize} implies that 
\begin{align*}
    \sqrt{\ee_{\tn{gen}}} &\le 
    A_5^{1/2} \lambda^{1/2} \left(\frac{1}{n_2} + \frac{1}{n_3}\right)^{1/4}
    + \frac{A_3^{1/2}}{n_1^{1/4}}
    + A_2^{1/2} \frac{1}{\lambda^{1/4}} \frac{1}{n_1^{1/4}}.
\end{align*}
We may write 
$\ee_1 = \sqrt{2\nu_1 \ee_{\tn{gen}}} = \sqrt{4B^2 \lambda_{\tn{max}}(\Sigma)} \cdot \lambda^{1/2} \sqrt{\ee_{\tn{gen}}}$,
so that for $\hat \beta \in \mathcal{I}$ we have 
\begin{align*}
    \hat \beta^2 \ee_1 \le \beta_{\tn{max}}^2 \sqrt{4B^2 \lambda_{\tn{max}}(\Sigma)} \cdot \lambda^{1/2} \sqrt{\ee_{\tn{gen}}} =: A_6 \lambda^{1/2} \sqrt{\ee_{\tn{gen}}}.
\end{align*}
Using the inequality $\sqrt{a+b} \le \sqrt{a} + \sqrt{b}$ for all $a,b\ge 0$, we may bound 
\begin{align*}
    \ee_2 &= \sqrt{2B^2 \nu_2 \ee_{\tn{gen}}} \le 
    \sqrt{4B^2 \beta_{\tn{max}}^2 \lambda_{\tn{max}}(\Sigma)} \cdot \lambda^{1/2} \sqrt{\ee_{\tn{gen}}}
    + \sqrt{2B^2 C_f \lambda_{\tn{max}}(\Sigma)} \cdot \sqrt{\ee_{\tn{gen}}} \\
    &=: A_7 \lambda^{1/2} \sqrt{\ee_{\tn{gen}}} + A_8 \sqrt{\ee_{\tn{gen}}},
\end{align*}
Thus 
\begin{align*}
    \hat \beta^2 \ee_1 + \ee_2 &\le A_6 \lambda^{1/2} \sqrt{\ee_{\tn{gen}}} + A_7 \lambda^{1/2} \sqrt{\ee_{\tn{gen}}} + A_8 \sqrt{\ee_{\tn{gen}}} \\
    &=: A_9 \lambda^{1/2} \sqrt{\ee_{\tn{gen}}} + A_8 \sqrt{\ee_{\tn{gen}}}.
\end{align*}
Plugging in our bound on $\sqrt{\ee_{\tn{gen}}}$ and grouping terms according to the power of $\lambda$, we deduce that $\hat \beta^2 \ee_1 + \ee_2 \le \ee_{\tn{cov}}$, where $\ee_{\tn{cov}}$ equals
\begin{align*}
& A_{10} \left(\frac{1}{n_2} + \frac{1}{n_3}\right)^{1/4} \lambda 
+ A_{11} \left(
\frac{1}{n_1^{1/4}} + \left(\frac{1}{n_2} + \frac{1}{n_3}\right)^{1/4} \right) \lambda^{1/2} 
+ A_{12} \frac{\lambda^{1/4}}{n_1^{1/4}} 
+  \frac{A_{13}}{n_1^{1/4}} 
+  A_{14} \frac{ \lambda^{-1/4}}{n_1^{1/4}} 
\end{align*}
and where $A_{10}, \ldots, A_{14}$ are the positive constants given in \Cref{sec:consts}.
It follows that on the event $E$,
\begin{align*}
    \E_1[ r \coverhat ] \ge (1-\alpha) + 2 \hat \beta \lambda \E_1[ (r_{B} - \hat \beta \hat h)^2 ] - \ee_{\tn{cov}} - (1-\alpha) \E_1[ |r - r_{B}| ].
\end{align*}
By \Cref{thm:generalize}, $E$ occurs with probability $1-\delta$
for sufficiently large $n_1,n_2,n_3$,
and we may conclude.

\section{Unconstrained existence and boundedness}\label{sec:unconstr-exist-bdd}


In this section, we prove apriori existence and boundedness of unconstrained global minimizers of the population objective \Cref{opt:unconstr-CP-obj}. We write $(h_{\lambda}^*, \beta_{\lambda}^*)$ for a minimizer of the unconstrained objective in \Cref{opt:unconstr-CP-obj} with regularization strength $\lambda \ge 0$.

In \Cref{lem:equiv-elim-beta}, we show that under \Cref{cond:zero-subopt}, we may eliminate $\beta$ from \Cref{opt:unconstr-CP-obj}, so that \Cref{opt:unconstr-CP-obj} is equivalent to solving the following unconstrained optimization problem over $h$:
\begin{align}\label{opt:elim-beta-unconstr-CP-obj}
    \min_{h\in \cH \setminus \{0\}} \E_1[\ell_{\alpha}(h,S)] - \lambda \frac{\E_1[rh]^2}{\E_1[h^2]}.
\end{align}




\begin{lemma}\label{lem:equiv-elim-beta}
    Under \Cref{cond:zero-subopt}, for $\lambda \ge 0$, given any minimizer $(h_{\lambda}^*, \beta_{\lambda}^*)$ of the objective in \Cref{opt:unconstr-CP-obj} with regularization $\lambda$, $h_{\lambda}^*$ is a minimizer of the objective in \Cref{opt:elim-beta-unconstr-CP-obj} with regularization $\lambda$. Conversely, if $h$ is a minimizer of the objective in \Cref{opt:elim-beta-unconstr-CP-obj} with regularization $\lambda$, then there exists a minimizer $(h_{\lambda}^*, \beta_{\lambda}^*)$ of the objective in \Cref{opt:unconstr-CP-obj} with regularization $\lambda$ such that $h_{\lambda}^* = h$.
\end{lemma}

\begin{proof}
    By \Cref{cond:zero-subopt}, the minimization in \Cref{opt:unconstr-CP-obj} with regularization $\lambda$ can be taken over $\cH\setminus \{0\}$. 
    Further, since the projection of $r$ onto $\tn{span}\{h\}:= \{ch: c\in \R\}$, 
    for $h\neq 0$ is given by $\frac{\E_1[rh]}{\E_1[h^2]} h$, we may explicitly minimize the objective in \Cref{opt:unconstr-CP-obj} over $\beta$ via
    \begin{align*}
        &\ell_{\alpha}(h, S) + \lambda \min_{\beta\in \R} \E_1[(\beta h - r)^2] = \ell_{\alpha}(h, S) + \lambda \E_1\left[ \left( \frac{\E_1[rh]}{\E_1[h^2]} h - r \right)^2 \right] \\
        &= \ell_{\alpha}(h, S) + \lambda \left( \E_1[r^2] - \E_1\left[ \left( \frac{\E_1[rh]}{\E_1[h^2]} h \right)^2 \right] \right) = \ell_{\alpha}(h, S) + \lambda \left( \E_1[r^2] - \frac{\E_1[rh]^2}{\E_1[h^2]} \right),
    \end{align*}
    where in the second step we applied the Pythagorean theorem. Since the term $\lambda \E_1[r^2]$ does not depend on the optimization variable $h$, we may drop it from the objective, which yields the objective in \Cref{opt:elim-beta-unconstr-CP-obj}. It follows that $h$ is a minimizer of the objective in \Cref{opt:elim-beta-unconstr-CP-obj} iff $h = h_{\lambda}^*$ for some minimizer $(h_{\lambda}^*, \beta_{\lambda}^*)$ of the objective of \Cref{opt:unconstr-CP-obj}.
\end{proof}

\begin{lemma}\label{lem:theta-star}
    Let $r_{\cH}$ denote the projection of $r$ onto $\cH$ in the Hilbert space induced by the inner product $\langle f, g\rangle  = \E_1[fg]$. Then under \Cref{cond:condl-dens-bdd} and \Cref{cond:1-in-H}, there exists $\theta^* > 0$ such that $\E_1[S] - \alpha^{-1} \E_1[ \ell_{\alpha}(\theta^* r_{\cH}, S) ] > 0$.
\end{lemma}

\begin{proof}
    Define $g : \R \to \R$ by $g(\theta) = \E_1[S] - \alpha^{-1} \E_1[ \ell_{\alpha}(\theta^* r_{\cH}, S) ]$. Clearly $g(0) = 0$. Note that by \Cref{cond:condl-dens-bdd}, $\mathbb{P}_{S|X}[S=0]=0$, so that
    \begin{align*}
        g'(0) = -\alpha^{-1} \E_1[ r_{\cH} (\mathbb{P}_{S|X}[S\le 0] - (1-\alpha)) ] = \alpha^{-1} (1-\alpha) \E_1[r_{\cH}].
    \end{align*}
    By \Cref{cond:1-in-H},
    $\E_1[r_{\cH}] = \E_1[r_{\cH}\cdot 1] = 
    \E_1[r\cdot 1] =
    \E_1[r] = 1$,
    so $g'(0) > 0$.
    Thus there exists $\theta^* > 0$ such that $g(\theta^*) > g(0) = 0$, as claimed. 
\end{proof}

\begin{lemma}[Existence of unconstrained minimizers]\label{lem:unconstr-existence}
    Under \Cref{cond:pop-cov-matr}, \Cref{cond:condl-dens-bdd}, \Cref{cond:basis-indep}, \Cref{cond:c-align}, \Cref{cond:r-second-mom}, \Cref{cond:1-in-H}, and \Cref{cond:zero-subopt}, for each $\lambda \ge 0$, there exists a global minimizer $(h_{\lambda}^*, \beta_{\lambda}^*)$ of the objective in \Cref{opt:unconstr-CP-obj}.
\end{lemma}

\begin{proof}
Fix $\lambda \ge 0$.
By \Cref{cond:zero-subopt} and \Cref{lem:equiv-elim-beta}, it suffices to show that there exists a global minimizer of the objective in \Cref{opt:elim-beta-unconstr-CP-obj}.
Let $G(h)$ denote the objective of \Cref{opt:elim-beta-unconstr-CP-obj}.
Define the function $\tilde h = \theta^* r_{\cH} \in \cH \setminus \{0\}$, 
where $\theta^*$ is chosen to 
satisfy \Cref{lem:theta-star}.
With $c_{\tn{indep}}$ from \Cref{cond:basis-indep},
define 
$\tilde B(\lambda) := 2c_{\tn{indep}}^{-1} ( 1 + \alpha^{-1}  \E_1[\ell_{\alpha}(\tilde h,S)] ) > 0$
and
\begin{align*}
    \tilde b(\lambda) := \frac{1}{2} \lambda_{\tn{max}}(\Sigma)^{-1/2} ( \E_1[S] - \alpha^{-1} \E_1[\ell_{\alpha}(\tilde h,S)] ) > 0.
\end{align*}
We show that if $\|h\| \ge \tilde B(\lambda)$ or $\|h\| \le \tilde b(\lambda)$, then $G(h) > G(\tilde h)$.
Consequently, the minimization in \Cref{opt:elim-beta-unconstr-CP-obj} can be taken over the compact set
$\{ \langle \gamma, \Phi \rangle : \tilde b(\lambda) \le \| \gamma \|_2 \le \tilde B(\lambda) \} \subseteq \cH$, 
so that by continuity of $G$ on $\cH\setminus \{0\}$, a global minimizer $h^*_{\lambda}$ exists.

To see this, first suppose $\|h\| \ge \tilde B(\lambda)$. Then writing $h = \langle \gamma, \Phi\rangle$ for $\gamma \in \R^d$ and applying \Cref{lem:pinball-bds}, the triangle inequality, and $S\in [0,1]$,
\begin{align}\label{ell-lb-2}
    \E_1[\ell_{\alpha}(h,S)] \ge \alpha \E_1[|h-S|] \ge \alpha(\E_1[|h|] - \E_1[|S|]) \ge \alpha(\E_1[|\langle \gamma, \Phi\rangle|] - 1).
\end{align}
By \Cref{cond:basis-indep} and our assumption that $\|h\| \ge \tilde B(\lambda)$, this implies that $\E_1[\ell_{\alpha}(h,S)] \ge \alpha( \tilde B(\lambda) c_{\tn{indep}} - 1 )$.
Further, by the Cauchy-Schwarz inequality, 
\begin{align*}
    \frac{\E_1[r h]^2}{\E_1[h^2]} \le \sup_{\tilde h'\in \cH \setminus \{0\}} \frac{\E_1[r \tilde h']^2}{\E_1[(\tilde h')^2]} \le \E_1[r_{\cH}^2].
\end{align*}
Thus by \Cref{lem:pinball-bds} and \Cref{cond:r-second-mom},
$  G(h) \ge \alpha( \tilde B(\lambda) c_{\tn{indep}} - 1 ) - \lambda \E_1[r_{\cH}^2]$.
To prove the inequality $G(h) > G(\tilde h)$, it suffices to show that 
\begin{align*}
    \alpha( \tilde B(\lambda) c_{\tn{indep}} - 1) - \lambda \E_1[r_{\cH}^2] > \E_1[\ell_{\alpha}(\tilde h,S)] - \lambda \frac{\E_1[r \tilde h]^2}{\E_1[\tilde h^2]}.
\end{align*}
Indeed, since $\tilde h$ is a scalar multiple of $r_{\cH}$, we have $\E_1[r_{\cH}^2] = \frac{\E_1[r \tilde h]^2}{\E_1[\tilde h^2]}$, so the inequality reduces to
$   \alpha( \tilde B(\lambda) c_{\tn{indep}} - 1) >  \E_1[\ell_{\alpha}(\tilde h,S)]$.
This holds by our choice of $\tilde B(\lambda)$, which finishes the argument in this case.

Next, suppose $\|h\| \le \tilde b(\lambda)$.
By \Cref{lem:pinball-bds}, the triangle inequality, and $S\in [0,1]$,
\begin{align}\label{ell-lower-bound}
    \E_1[\ell_{\alpha}(h,S)] \ge \alpha \E_1[|h-S|] \ge \alpha(\E_1[S] - \E_1[|h|]).
\end{align}
As above, the Cauchy-Schwarz inequality  implies the bound $\frac{\E_1[r h]^2}{\E_1[h^2]} \le \E_1[r_{\cH}^2]$.
We deduce that 
\begin{align*}
    G(h) &\ge \alpha (\E_1[S] - \E_1[|h|]) - \lambda \E_1[r_{\cH}^2].
\end{align*}
Writing $h = \langle \gamma, \Phi\rangle$ for $\gamma \in \R^d$, our assumption that $\|h\| \le \tilde b(\lambda)$ implies that  
\begin{align*}
    \E_1[|h|] \le \E_1[|h|^2]^{1/2} = \E_1[ \gamma^\top  \Phi \Phi^\top  \gamma ]^{1/2} \le  \tilde b(\lambda) \lambda_{\tn{max}}(\Sigma)^{1/2},
\end{align*}
which when plugged into our lower bound on $G(h)$ yields  
\begin{align*}
    G(h) \ge \alpha (\E_1[S] -  \tilde b(\lambda) \lambda_{\tn{max}}(\Sigma)^{1/2}) - \lambda \E_1[r_{\cH}^2].
\end{align*}
To prove the inequality $G(h) > G(\tilde h)$, it suffices to show that 
\begin{align*}
     \alpha (\E_1[S] - \tilde b(\lambda) \lambda_{\tn{max}}(\Sigma)^{1/2}) - \lambda \E_1[r_{\cH}^2] > \E_1[\ell_{\alpha}(\tilde h,S)] - \lambda \frac{\E_1[r \tilde h]^2}{\E_1[\tilde h^2]}.
\end{align*}
As above, since $\tilde h$ is a scalar multiple of $r_{\cH}$, we have $\E_1[r_{\cH}^2] = \frac{\E_1[r \tilde h]^2}{\E_1[\tilde h^2]}$, so the inequality reduces to
\begin{align*}
    \alpha (\E_1[S] - \tilde b(\lambda) \lambda_{\tn{max}}(\Sigma)^{1/2}) > \E_1[\ell_{\alpha}(\tilde h,S)].
\end{align*}
This holds for our choice of $\tilde b(\lambda)$, finishing the proof.
\end{proof}

\begin{lemma}[Bounds on unconstrained minimizers]\label{lem:apriori-bdd}
    Under the conditions used in \Cref{lem:unconstr-existence}, for all $\lambda > 0$, for any minimizer $(h_{\lambda}^*, \beta_{\lambda}^*)$ of the objective in \Cref{opt:unconstr-CP-obj}, we have that $\|h_{\lambda}^*\| \in (B_{\tn{lower}}, B_{\tn{upper}})$ and $\beta_{\lambda}^* \in (\beta_{\tn{lower}}, \beta_{\tn{upper}})$, where 
    \begin{align}\label{bdef}
        B_{\tn{lower}} &= \frac{1}{2} \lambda_{\tn{max}}(\Sigma)^{-1/2} ( \E_1[S] - \alpha^{-1} \E_1[ \ell_{\alpha}(\theta^* r_{\cH}, S) ] ) > 0, \\ 
        B_{\tn{upper}} &= 2c_{\tn{indep}}^{-1} (\alpha^{-1} \E_1[ \ell_{\alpha}(\theta^* r_{\cH}, S) ]  + 1 ), \qquad
        \beta_{\tn{lower}} = \frac{c_{\tn{align}}}{B_{\tn{upper}} \lambda_{\tn{max}}(\Sigma)^{1/2}} > 0, \nonumber \\
        \beta_{\tn{upper}} &= \frac{\E_1[r^2]^{1/2}}{B_{\tn{lower}} \lambda_{\tn{min}}(\Sigma)^{1/2}},\nonumber
    \end{align}
    and where $\theta^* > 0$ is as in \Cref{lem:theta-star} and $r_{\cH}$ denotes the projection of $r$ onto $\cH$ in the Hilbert space induced by the inner product $\langle f, g\rangle  = \E_1[fg]$.
\end{lemma}

\begin{proof}

In order to derive our bounds, we consider the reparametrized optimization problem
\begin{align}\label{opt:elim-beta-unconstr-CP-obj-reparam}
    \min_{h\in \cH \setminus \{0\}} \xi \E_1[\ell_{\alpha}(h,S)] - \frac{\E_1[rh]^2}{\E_1[h^2]}
\end{align}
for $\xi \ge 0$. We claim that for $\xi > 0$, any minimizer of the objective in \Cref{opt:elim-beta-unconstr-CP-obj-reparam} is of the form $h_{1/\xi}^*$.
To see this, note that for $\xi > 0$, the objective of \Cref{opt:elim-beta-unconstr-CP-obj} with regularization $\lambda = 1/\xi$ can be obtained by scaling the objective of \Cref{opt:elim-beta-unconstr-CP-obj-reparam} by the positive factor $1/\xi$. Next, by \Cref{cond:zero-subopt}, we may apply \Cref{lem:equiv-elim-beta} to deduce that $h \in \cH\setminus \{0\}$ is a minimizer of the objective in \Cref{opt:elim-beta-unconstr-CP-obj} with regularization $\lambda = 1/\xi$ iff $h = h_{1/\xi}^*$. 

In particular, by \Cref{lem:unconstr-existence}, for all $\xi > 0$, there exists a global minimizer of \Cref{opt:elim-beta-unconstr-CP-obj-reparam} with regularization $\xi$. 
In the case that $\xi=0$, 
it is clear that any minimizer 
$h^*_{\infty}$
of the objective in \Cref{opt:elim-beta-unconstr-CP-obj-reparam} with regularization $\xi = 0$
has the form 
$h^*_{\infty} = \theta r_{\cH}$ for some scalar $\theta > 0$.

Since there exists a minimizer of the objective in \Cref{opt:elim-beta-unconstr-CP-obj-reparam} for all regularizations $\xi$ in the interval $[0,\infty)$, we may apply \Cref{lem:mono} to deduce that for all $\xi > 0$ we have 
$    \E_1[ \ell_{\alpha}(h_{1/\xi}^*, S) ] \le \E_1[ \ell_{\alpha}(h_{\infty}^*, S) ]$.

We prove lower and upper bounds on $\|h_{1/\xi}^*\|$ for all $\xi > 0$. We begin with the lower bound.

\textit{Lower bound:} 
By \eqref{ell-lower-bound}, we have 
$    \E_1[\ell_{\alpha}(h_{1/\xi}^*, S)] 
    \ge \alpha (\E_1[S]
    - \E_1[ |h_{1/\xi}^*| ] ).$
Rearranging, we obtain the lower bound
\begin{align*}
    \E_1[ |h_{1/\xi}^*| ] \ge \E_1[S] - \alpha^{-1} \E_1[ \ell_{\alpha}(h_{\infty}^*, S) ].
\end{align*}
By \Cref{lem:theta-star}, there exists $\theta^* > 0$ such that 
$    \E_1[S] - \alpha^{-1} \E_1[ \ell_{\alpha}(\theta^* r_{\cH}, S) ] > 0$.
Setting $h_{\infty}^* = \theta^* r_{\cH}$ and plugging in the expression for $B_{\tn{lower}}$ given in \eqref{bdef}, our lower bound becomes $\E_1[|h_{1/\xi}^*|] > \lambda_{\tn{max}}(\Sigma)^{1/2} B_{\tn{lower}}$. 
We now convert this $L^1$ norm bound to an $L^2$ norm bound as follows. 
Write $h_{1/\xi}^* = \langle \gamma_{1/\xi}^*, \Phi\rangle$ for $\gamma_{1/\xi}^* \in \R^d$. 
By 
the 
Cauchy-Schwarz inequality, we obtain the upper bound 
\begin{align*}
    \E_1[|h_{1/\xi}^*|] \le \E_1[|h_{1/\xi}^*|^2]^{1/2} = \E_1[ (\gamma_{1/\xi}^*)^\top  \Phi \Phi^\top  \gamma_{1/\xi}^* ]^{1/2} \le \lambda_{\tn{max}}(\Sigma)^{1/2} \| \gamma_{1/\xi}^* \|_2.
\end{align*}
Combining this with the lower bound $\E_1[|h_{1/\xi}^*|] > \lambda_{\tn{max}}(\Sigma)^{1/2} B_{\tn{lower}}$, we deduce that $\|h_{1/\xi}^*\| = \|\gamma_{1/\xi}^*\|_2 > B_{\tn{lower}}$, as claimed.

\textit{Upper bound:} We prove the upper bound in a similar manner.
By the first two steps in \eqref{ell-lb-2},
and using $S\in [0,1]$, we have
\begin{align*}
    \E_1[\ell_{\alpha}(h_{1/\xi}^*, S)] 
    \ge \alpha (\E_1[ |h_{1/\xi}^*| ] - \E_1[ |S| ]) 
    \ge \alpha (\E_1[ |h_{1/\xi}^*| ] - 1).
\end{align*}
Rearranging, we obtain the upper bound 
$    \E_1[ |h_{1/\xi}^*| ] \le \alpha^{-1} \E_1[ \ell_{\alpha}(h_{\infty}^*, S) ]  + 1$.
Write $h_{1/\xi}^* = \langle \gamma_{1/\xi}^*, \Phi\rangle$ for $\gamma_{1/\xi}^* \in \R^d$. Since we have already established that $\|h_{1/\xi}^*\| > B_{\tn{lower}} > 0$, we know that $\gamma_{1/\xi}^* \neq 0$. Thus we may write 
\begin{align*}
    \E_1[ |h_{1/\xi}^*| ] &= \E_1[ | \langle \gamma_{1/\xi}^*, \Phi \rangle | ]
    = \| \gamma_{1/\xi}^* \|_2 \E_1\left[ \left| \left\langle \frac{\gamma_{1/\xi}^*}{\| \gamma_{1/\xi}^* \|_2}, \Phi \right\rangle \right| \right].
\end{align*}
By \Cref{cond:basis-indep}, this is at least $\|\gamma_{1/\xi}^*\|_2 c_{\tn{indep}}$. Combining these upper and lower bounds on $\E_1[|h_{1/\xi}^*|]$, we obtain $\|\gamma_{1/\xi}^*\|_2 c_{\tn{indep}} \le \alpha^{-1} \E_1[ \ell_{\alpha}(h_{\infty}^*, S) ]  + 1$. 
Isolating $\| \gamma_{1/\xi}^* \|_2$, we have 
\begin{align*}
    \|h^*_{1/\xi}\| = \| \gamma_{1/\xi}^* \|_2 \le c_{\tn{indep}}^{-1} (\alpha^{-1} \E_1[ \ell_{\alpha}(h_{\infty}^*, S) ]  + 1 ) < B_{\tn{upper}},
\end{align*}
as claimed.

Having established
$  0 < B_{\tn{lower}} < \inf_{\lambda > 0} \| h_{\lambda}^* \| \le \sup_{\lambda > 0} \| h_{\lambda}^* \| < B_{\tn{upper}} < \infty$,
we turn to upper and lower bounds on $\beta^*_{\lambda}$.
As shown in the proof of \Cref{lem:equiv-elim-beta}, if $(h_{\lambda}^*, \beta_{\lambda}^*)$ is a minimizer of the objective in \Cref{opt:unconstr-CP-obj} with regularization $\lambda$, then 
$\beta_{\lambda}^* = \frac{\E_1[ r h_{\lambda}^* ]}{\E_1[ |h_{\lambda}^*|^2 ]}$.
By \Cref{cond:c-align},
$    \frac{\E_1[ rh_0^* ]}{\E_1[ |h_0^*|^2 ]^{1/2}} \ge c_{\tn{align}} > 0$
for some minimizer $(h_0^*, \beta_0^*)$ of the objective in \Cref{opt:unconstr-CP-obj} with regularization $0$.
By \Cref{cond:zero-subopt} and \Cref{lem:equiv-elim-beta}, $h$ is a minimizer of the objective in \Cref{opt:elim-beta-unconstr-CP-obj} with regularization $\lambda \ge 0$ iff $h = h_{\lambda}^*$ for some minimizer $(h_{\lambda}^*, \beta_{\lambda}^*)$ of the objective in \Cref{opt:unconstr-CP-obj}.
Thus by \Cref{lem:unconstr-existence}, for all $\lambda \ge 0$, there exists a global minimizer of \Cref{opt:elim-beta-unconstr-CP-obj}, and we may apply \Cref{lem:mono} to \Cref{opt:elim-beta-unconstr-CP-obj} to deduce that for any $\lambda \ge 0$ we have
$    \frac{\E_1[ rh_{\lambda}^* ]}{\E_1[ |h_{\lambda}^*|^2 ]^{1/2}} \ge c_{\tn{align}} > 0$.
Consequently, by our bounds on $h_{\lambda}^*$, \Cref{cond:r-second-mom}, and the Cauchy-Schwarz inequality, if we write $h_{\lambda}^* = \langle \gamma_{\lambda}^*, \Phi\rangle$ for $\gamma_{\lambda}^* \in \R^d$, then we have
\begin{align*}
    \beta_{\lambda}^* \ge \frac{c_{\tn{align}}}{\E_1[|h_{\lambda}^*|^2]^{1/2}} = \frac{c_{\tn{align}}}{\E_1[(\gamma_{\lambda}^*)^\top  \Phi \Phi^\top  \gamma_{\lambda}^*]^{1/2}} > \frac{c_{\tn{align}}}{B_{\tn{upper}} \lambda_{\tn{max}}(\Sigma)^{1/2}} =: \beta_{\tn{lower}}
\end{align*}
and
\begin{align*}
     \beta_{\lambda}^* \le \frac{\E_1[r^2]^{1/2}}{\E_1[|h_{\lambda}^*|^2]^{1/2}} < \frac{\E_1[r^2]^{1/2}}{B_{\tn{lower}} \lambda_{\tn{min}}(\Sigma)^{1/2}} =: \beta_{\tn{upper}},
\end{align*}
completing the proof.
\end{proof}

\section{Monotonicity}

\begin{lemma}\label{lem:mono}
    For some set $\mathcal{X}$ and $f,g:\mathcal{X}\to \R$, let $x(c) = \arg\min_{x\in \mathcal{X}} (f(x) + c g(x))$, where $f,g$ are such that
     for some interval $\mathcal{I}\subset \R$,
    the minimum is attained for all $c\in \mathcal{I}$.
    Then $G: \mathcal{I} \to \R$, $G: c\mapsto g(x(c))$ is non-increasing in $c$. 
\end{lemma}
\begin{proof}
Let $c_1,c_2 \in \mathcal{I}$, $c_1 < c_2$.
At \( c = c_1 \), the minimizer \( x(c_1) \) satisfies:
    \[
    f(x(c_1)) + c_1 g(x(c_1)) \leq f(x(c_2)) + c_1 g(x(c_2)).
    \]
At \( c = c_2 \), the minimizer \( x(c_2) \) satisfies:
    \[
    f(x(c_2)) + c_2 g(x(c_2)) \leq f(x(c_1)) + c_2 g(x(c_1)).
    \]

Adding the two inequalities, we find
\begin{align*}
    &\big[ f(x(c_1)) + c_1 g(x(c_1)) \big] + \big[ f(x(c_2)) + c_2 g(x(c_2)) \big] \\
&\leq 
\big[ f(x(c_1)) + c_2 g(x(c_1)) \big] + \big[ f(x(c_2)) + c_1 g(x(c_2)) \big].
\end{align*}
Subtracting the common terms \( f(x(c_1)) + f(x(c_2)) \) leads to
\[
c_1 g(x(c_1)) + c_2 g(x(c_2)) \leq c_2 g(x(c_1)) + c_1 g(x(c_2)).
\]
Rearranging, 
and factoring out \( c_1 \) and \( c_2 \), we find
\[
c_1 \big[ g(x(c_1)) - g(x(c_2)) \big] - c_2 \big[ g(x(c_1)) - g(x(c_2)) \big] \leq 0.
\]
Thus, 
\(
(c_1 - c_2) \big[ g(x(c_1)) - g(x(c_2)) \big] \leq 0.
\)
Since \( c_2 - c_1 > 0 \), the inequality implies
$g(x(c_1)) \geq g(x(c_2))$, as desired.
\end{proof}

\section{Helper lemmas}

\begin{lemma}\label{lem:pinball-bds}
    If $\alpha \le 0.5$, then $\alpha |c-s| \le \ell_{\alpha}(c,s) \le (1-\alpha)|c-s|$ for all $c,s\in \R$.
\end{lemma}

\begin{proof}
    If $s\ge c$, then $\ell_{\alpha}(c,s) = (1-\alpha)(s-c)$. Since $s-c \ge 0$ and $\alpha \le 1-\alpha$, we have $\alpha (s-c) \le \ell_{\alpha}(c,s) \le (1-\alpha) (s-c)$, which implies $\alpha |c-s| \le \ell_{\alpha}(c,s) \le (1-\alpha) |c-s|$.
    If $s < c$, then $\ell_{\alpha}(c,s) = \alpha(c-s)$. Since $c-s > 0$ and $\alpha \le 1-\alpha$, we have $\alpha (c-s) \le \ell_{\alpha}(c,s) \le (1-\alpha) (c-s)$, which implies $\alpha |c-s| \le \ell_{\alpha}(c,s) \le (1-\alpha) |c-s|$.
\end{proof}

\begin{lemma}\label{lem:pinball-lip}
    If $\alpha \le 0.5$, then the map $\R\to \R$ given by $c \mapsto \ell_{\alpha}(c, s)$ is $(1-\alpha)$-Lipschitz.
\end{lemma}

\begin{proof}
    If $s \le c_1 \le c_2$, we have $0 \le \ell_{\alpha}(c_2,s) - \ell_{\alpha}(c_1,s) = \alpha(c_2-c_1)$, which by $\alpha \le 0.5$ is at most $(1-\alpha)(c_2-c_1)$. Hence $|\ell_{\alpha}(c_2,s) - \ell_{\alpha}(c_1,s)| \le (1-\alpha) |c_2-c_1|$.
    If $c_1 \le s \le c_2$ and $\ell_{\alpha}(c_2,s) \ge \ell_{\alpha}(c_1,s)$, then we have
    \begin{align*}
        0 \le \ell_{\alpha}(c_2,s) - \ell_{\alpha}(c_1,s) = \alpha (c_2-s) - (1-\alpha) (s-c_1) \le \alpha (c_2-s) + \alpha (s-c_1) = \alpha (c_2-c_1),
    \end{align*}
    which by $\alpha \le 0.5$ implies $|\ell_{\alpha}(c_2,s) - \ell_{\alpha}(c_1,s)| \le (1-\alpha) |c_2-c_1|$.
    If $c_1 \le s \le c_2$ and $\ell_{\alpha}(c_2,s) \le \ell_{\alpha}(c_1,s)$, then
    \begin{align*}
        &0 \le \ell_{\alpha}(c_1,s) - \ell_{\alpha}(c_2,s) = (1-\alpha) (s-c_1) - \alpha (c_2-s) \\
        &\le (1-\alpha) (s-c_1) + (1-\alpha) (c_2-s) = (1-\alpha) (c_2-c_1),
    \end{align*}
    hence $|\ell_{\alpha}(c_2,s) - \ell_{\alpha}(c_1,s)| \le (1-\alpha) |c_2-c_1|$.
    Finally, if $c_1 \le c_2 \le s$, we have $0\le \ell_{\alpha}(c_1,s) - \ell_{\alpha}(c_2,s) = (1-\alpha) (c_2-c_1)$, hence $|\ell_{\alpha}(c_2,s) - \ell_{\alpha}(c_1,s)| \le (1-\alpha) |c_2-c_1|$.
\end{proof}

\begin{lemma}\label{lem:pinball-cvx}
    The map $\cH \to \R$ given by $h \mapsto \ell_{\alpha}(h(x), s)$ is convex for all $x\in \xx$ and $s\in \R$.
\end{lemma}

\begin{proof}
    Write $h(x) = \langle \gamma, \Phi \rangle$ for $\gamma \in \R^d$. It suffices to show that the mapping $\R^d \to \R$ given by $\gamma \mapsto \ell_{\alpha}(\gamma^\top \Phi(x), s)$ is convex. But this map is the composition of the linear function $\R^d \to \R$ given by $\gamma \mapsto \gamma^\top \Phi(x)$ and the convex function $\R \to \R$ given by $c \mapsto \ell_{\alpha}(c,s)$, hence it is convex.
\end{proof}

\begin{lemma}\label{lem:h-unif-bd}
    Under \Cref{cond:c-phi}, if $h\in \cH$, then $\sup_{x\in \xx} |h(x)| \le C_{\Phi} \|h\|$, where we use the norm given by $\|h\| = \|\gamma\|_2$ for $h = \langle \gamma, \Phi \rangle$.
    In particular, if $h\in \cH_B$, then $\sup_{x\in \xx} |h(x)| \le BC_{\Phi}$.
\end{lemma}

\begin{proof}
    Writing $h = \langle \gamma, \Phi\rangle$ for $\gamma \in \R^d$, we have $\sup_{x\in \xx} |h(x)| = \sup_{x\in \xx} |\langle \gamma, \Phi(x) \rangle| \le \sup_{x\in \xx} \|\gamma\|_2 \|\Phi(x)\|_2 \le C_{\Phi} \|h\|$, where in the second step we applied the Cauchy-Schwarz inequality.
\end{proof}

\begin{lemma}\label{lem:pinball-deriv}
Consider the function $\varphi : \R^d \to \R$ given by
$\varphi(\gamma) = \E_1[ \ell_{\alpha}(h_\gamma(X), S)  ]$,
where $h := h_\gamma : \xx \to \R$ is given by $h(x)  = \langle \gamma, \Phi(x) \rangle$ for all $x\in \xx$.
Then under \Cref{cond:c-phi} and \Cref{cond:condl-dens-bdd},
$\varphi$ is twice-differentiable, with gradient and Hessian given by
\begin{align*}
    \nabla_{\gamma} \varphi(\gamma) = \E_1[ (\mathbb{P}_{S|X}[h(X) > S] - (1-\alpha)) \Phi(X) ], \qquad \nabla_{\gamma}^2 \varphi(\gamma) = \E_1[f_{S|X}(h(X)) \Phi(X) \Phi(X)^\top ].
\end{align*}
Consequently, given $\tilde \gamma \in \R^d$,
defining $g : \xx \to \R$ as $g(x) = \langle \tilde \gamma, \Phi(x) \rangle$ for all $x\in \xx$,
the directional derivative of $\varphi : \cH \to \R$ in the direction $g$ is given by
$\langle \tilde \gamma, \nabla_{\gamma} \varphi(\gamma)\rangle = \E_1[ (\mathbb{P}_{S|X}[h(X) > S] - (1-\alpha)) g(X) ]$.
\end{lemma}

\begin{proof}
For each $x\in \xx$, define the function $\eta(\cdot; x) : \R \to \R$ given,  for all $u$, by 
$\eta(u; s) = \E_{S|X=x}[ \ell_{\alpha}(u, S) ]$.
For each $s\in \R$, define the function
$\chi(\cdot; s) : \R \to \R$,
where for all $u$,
$\chi(u; s) = \alpha \mathbf{1}[u > s] - (1-\alpha) \mathbf{1}[u \le s]$.

By the definition of the pinball loss $\ell_{\alpha}(\cdot, \cdot)$,
and since by \Cref{cond:condl-dens-bdd}
the conditional density $f_{S|X=x}(\cdot)$ of $S|X=x$ exists for all $x\in \xx$,
the derivative of $\ell_{\alpha}(u,S)$
with respect to $u$ 
agrees with the random variable $\chi(u; S)$
almost surely with respect to the distribution $S|X=x$.
Also, note that for fixed $u\in \R$, $|\chi(u;S)|$ is bounded by the constant $(1-\alpha)$.
By the dominated convergence theorem,
it follows that 
$u\mapsto \eta(u;x)$ is differentiable,
and that its derivative equals 
$\frac{\partial}{\partial u} \eta(u; x) = \E_{S|X=x}[ \chi(u;S) ]$, which, by the formula for $\chi(u;S)$, can be written as
$\alpha \mathbb{P}_{S|X=x}[u > S] - (1-\alpha) \mathbb{P}_{S|X=x}[u\le S]$.
Thus for all $u\in \R$ and $x\in \mathcal{X},$ we may write
$\frac{\partial}{\partial u} \eta(u;x) = \mathbb{P}_{S|X=x}[u > S] - (1-\alpha)$.
Since by \Cref{cond:condl-dens-bdd}
the conditional density $f_{S|X=x}$ of the distribution $S|X=x$ exists for all $x\in \xx$,
it follows that the cdf $u \mapsto \mathbb{P}_{S|X=x}[u > S]$ is differentiable
for all $u\in \R$ and
 all $x\in \xx$ with derivative given by $u \mapsto f_{S|X=x}(u)$.
Thus the map $u \mapsto \frac{\partial}{\partial u} \eta(u;x)$ is differentiable for all $x\in \xx$
with derivative given by 
$u \mapsto f_{S|X=x}(u)$.
In particular, $\eta(\cdot;x)$ is twice-differentiable
with second derivative given by
$f_{S|X=x}(\cdot)$.

Next, for each $x\in \xx$,
define the function
$\psi(\cdot; x) : \R^d \to \R$
given by
$\psi(\gamma; x) = \E_{S|X=x}[\ell_{\alpha}(h_\gamma(x), S)]$, 
where $h = h_\gamma =  \langle \gamma, \Phi\rangle$.
For each $x\in \xx$, 
let $\tn{ev}(\cdot; x) : \R^d \mapsto \R$ be given by
$\tn{ev}(\gamma; x) = h_\gamma(x)$,
where $h = h_\gamma= \langle \gamma, \Phi\rangle$.
Then $\psi(\cdot;x)$ is given by the composition $\eta(\cdot; x) \circ \tn{ev}(\cdot; x)$.
Since $\tn{ev}(\gamma;x) = \langle \gamma, \Phi(x)\rangle$, $\tn{ev}(\cdot;x)$ is linear, it is smooth.
Its gradient is given by
$\nabla_{\gamma} \tn{ev}(\gamma; x) = \Phi(x)$
for all $\gamma \in \R^d$,
and its Hessian is zero.
It follows that $\psi(\cdot;x)$ is twice-differentiable.
By the chain rule, the gradient of $\psi(\cdot;x)$ is given by
\begin{align*}
\nabla_{\gamma} \psi(\gamma; x)
= 
\frac{\partial}{\partial u} \eta(u;x)\bigg|_{u=\tn{ev}(\gamma; x)} \cdot \nabla_{\gamma} \tn{ev}(\gamma; x)
= (\mathbb{P}_{S|X=x}[h(x) > S] - (1-\alpha)) \Phi(x).
\end{align*}
Since the map
$\gamma \mapsto \mathbb{P}_{S|X=x}[h(x) > S] - (1-\alpha)$
is given by the composition
$\frac{\partial}{\partial u} \eta(\cdot; x) \circ \tn{ev}(\cdot; x)$,
we may again apply the chain rule to deduce that the Hessian of $\psi(\cdot;x)$ is given by
\begin{align*}
    \nabla_{\gamma}^2 \psi(\gamma; x) = \frac{\partial^2}{\partial u^2} \eta(u;x) \bigg|_{u=\tn{ev}(\gamma;x)} \cdot \nabla_{\gamma}\tn{ev}(\gamma;x) \cdot \Phi(x)^\top = f_{S|X=x}(h(x)) \Phi(x) \Phi(x)^\top.
\end{align*}

Returning to our original function $\varphi$, note that by the tower property, $\varphi(\gamma) = \E_1[\psi(\gamma; X)]$.
Note that
$\| \nabla_{\gamma} \psi(\gamma;x) \|_2$ is at most
\begin{align*}
    | \mathbb{P}_{S|X=x}[h(x) > S] - (1-\alpha) | \| \Phi(x) \|_2 \le (|\mathbb{P}_{S|X=x}[h(x) > S]| + (1-\alpha)) \| \Phi(x) \|_2 \le (2-\alpha) C_{\Phi},
\end{align*}
where in the first step we used the triangle inequality, and in the second step we used the fact that $\mathbb{P}_{S|X=x}[h(x) > S] \le 1$ and \Cref{cond:c-phi}.
Similarly,
we may bound the Frobenius norm $\|\cdot\|_F$ of
$\nabla_{\gamma}^2 \psi(\gamma; x)$ by 
\begin{align*}
    |f_{S|X=x}(h(x))| \| \Phi(x) \Phi(x)^\top \|_F \le C_f \|\Phi(x)\|_2^2 \le C_f C_{\Phi}^2,
\end{align*}
where in the first step we used \Cref{cond:c-phi}, the identity $\|vv^\top\|_F = \|v\|_2^2$, and in the second step we used \Cref{cond:condl-dens-bdd}.
Since the entries of
$\nabla_{\gamma} \psi(\cdot; x)$
and
$\nabla_{\gamma}^2 \psi(\cdot; x)$
are bounded by constants,
we may apply the dominated convergence theorem to deduce that
$\varphi$ is twice-differentiable,
with gradient given by
$\nabla_{\gamma} \varphi(\gamma) = \E_1[ \nabla_{\gamma} \psi(\gamma; X) ]$
and Hessian given by
$\nabla_{\gamma}^2 \varphi(\gamma) = \E_1[ \nabla_{\gamma}^2 \psi(\gamma; X) ]$.

Finally, since the directional derivative of $\varphi$ in the direction $g$
is defined as $\langle \tilde \gamma, \nabla_{\gamma} \varphi(\gamma) \rangle$, we may plug in our expression for the gradient to deduce
\begin{align*}
    \langle \tilde \gamma, \nabla_{\gamma} \varphi(\gamma) \rangle
    &= \langle \tilde \gamma, \E_1[ (\mathbb{P}_{S|X}[h(X) > S] - (1-\alpha)) \Phi(X) ] \rangle \\
    &= \E_1[ (\mathbb{P}_{S|X}[h(X) > S] - (1-\alpha)) \langle \tilde\gamma, \Phi(X)\rangle ] \\
    &= \E_1[ (\mathbb{P}_{S|X}[h(X) > S] - (1-\alpha)) g(X) ].
\end{align*}
The result follows.
\end{proof}

\end{document}